\documentclass[11pt, twoside]{article}

\usepackage{fullpage}

\input{commands-arXiv-version}

\begin{document}

\begin{center}

{\bf{\LARGE{Estimating stationary mass, frequency by frequency}}}

\vspace*{.2in}

{\large{
\begin{tabular}{ccc}
Milind Nakul$^{\star}$, Vidya Muthukumar$^{\dagger, \star}$, Ashwin Pananjady$^{\star, \dagger}$
\end{tabular}
}}
\vspace*{.2in}

\begin{tabular}{c}
H. Milton Stewart School of Industrial and Systems Engineering$^\star$\\
School of Electrical and Computer Engineering$^\dagger$ \\
Georgia Institute of Technology
\end{tabular}

\vspace*{.2in}

\today

\vspace*{.2in}
\end{center}

\begin{abstract}
Suppose we observe a trajectory of length $n$ from an exponentially $\alpha$-mixing stochastic process over a finite but potentially large state space. We consider the problem of estimating the probability mass placed by the stationary distribution of any such process on elements that occur with a certain frequency in the observed sequence. We estimate this vector of probabilities in total variation distance, showing universal consistency in $n$ and recovering known results for i.i.d. sequences as special cases.
Our proposed methodology---implementable in linear time---carefully combines the plug-in (or empirical) estimator with a recently-proposed modification of the Good--Turing estimator called \textsc{WingIt}, which was originally developed for Markovian sequences. En route to controlling the error of our estimator, we develop new performance bounds on \textsc{WingIt} and the plug-in estimator for 
 exponentially $\alpha$-mixing stochastic processes. Importantly, the extensively used method of Poissonization can no longer be applied in our non i.i.d. setting, and so we develop complementary tools---including concentration inequalities for a natural self-normalized statistic of mixing sequences---that may prove independently useful in the design and analysis of estimators for related problems. Simulation studies corroborate our theoretical findings.
\end{abstract}

\section{Introduction}

Estimating the stationary distribution of a stochastic process from an $n$-length trajectory is a foundational problem in statistics and machine learning, with broad implications for fields such as ecology \citep{fisher1943relation}, genomics \citep{favaro2012new,lijoi2007bayesian} and natural language processing~\citep{church1991probability,chen1999empirical,ney1994structuring}. Historically, most research has focused on settings where the data is either i.i.d. or exchangeable, where classical techniques were based on so-called ``add-constant" estimators~\citep{laplace1814essai,krichevsky1981performance}. 
Other estimators that see appealing empirical performance are absolute discounting~\citep{ney1994structuring}, Jelinek--Mercer smoothing \citep{jelinek1985probability}, and the Good--Turing estimator~\citep{good}. These estimators proceed by first estimating the probability mass\footnote{Note that $M^{\pi}_{\zeta}$ is a random functional of the stochastic process, since the set of symbols occurring $\zeta$ times in the observed sequence is random.} $M^{\pi}_{\zeta}$ placed by the distribution $\pi$ on symbols occurring with each frequency $\zeta = 0, 1, \ldots, n$ in the sequence, and this estimation problem will form the focus of the current paper. One can then transform the estimate of the vector $(M^{\pi}_{\zeta})_{\zeta = 0}^{n}$ to an estimator of $\pi$ itself; see~\cite{orlitsky-suresh}. Among the aforementioned estimators, the Good--Turing estimator has been thoroughly analyzed in the i.i.d. setting, starting from the pioneering work of~\citet{mcallester2000convergence}. In particular, these authors showed a \emph{universal} result of the following form: For any i.i.d. sequence from a distribution $\pi$ defined on a finite alphabet, the Good--Turing estimator consistently estimates the probability mass $M^{\pi}_{\zeta}$ placed by the distribution on symbols occurring with each fixed frequency $\zeta$ in the sequence. Subsequent work~\citep{drukh2005concentration, acharya2013optimal, orlitsky-suresh} refined these bounds and proposed a \emph{hybrid} estimator, i.e.~the Good--Turing estimator to approximate $M^{\pi}_{\zeta}$ for small $\zeta$ (i.e., to estimate the probability mass placed on small-frequency symbols) and the plug-in or empirical estimator to approximate $M^{\pi}_{\zeta}$ for large $\zeta$ (i.e., to estimate the probability mass placed on large-frequency symbols). In addition, the problem of estimating the missing mass, $M^{\pi}_0$, has garnered significant interest in the i.i.d. setting and has been analyzed thoroughly~\citep{mcallester2003concentration,berend2012missing,berend2013concentration}.
In particular, the problem of estimating the probability vector $( M^{\pi}_{\zeta} )^n_{\zeta = 0}$ has been studied in various divergences including total variation, Kullback--Leibler, and chi-square. 
In addition to upper bounds, we also have an intricate understanding of the sharpness properties of various estimators as well as minimax lower bounds~\citep{acharya2013optimal, orlitsky-suresh}. There has also been a significant body of follow-up work in the i.i.d. setting~\citep[see, e.g.,][]{ohannessian2012rare,painsky2023generalized,hao2019doubly}.

While the aforementioned papers have provided a detailed understanding of the i.i.d. setting, many real-world stochastic processes exhibit temporal dependencies. Examples abound in natural language processing, where sequences of words are sometimes better captured through hidden Markov modeling, and genomics, where DNA sequences are formed from interdependent base pairs. These dependencies in the stochastic process pose significant challenges for estimation of the stationary distribution. Indeed, the naive application of estimators such as Good--Turing is no longer valid even when the stochastic process is Markovian~\citep{chandra2021good} and different algorithmic tools are required for addressing temporal dependence~\citep{hao2018learning, skorski2020missing, chandra2021good}. Motivated by this issue, recent work by \citet{pananjady2024just} studied the Markovian setting and introduced a ``leave-a-window-out" variant of the Good--Turing estimator for estimating the stationary \emph{missing mass}, i.e., the probability mass placed by the stationary distribution on elements that do not occur in the observed sequence. They showed that this variant of Good--Turing --- which they termed \emph{windowed Good--Turing} or \Wingit --- estimates the stationary missing mass with minimax optimal mean-squared error. 
Some extensions were also provided for estimation of the ``small-count" probability, i.e., the mass placed by the distribution on elements occurring \emph{at most} a certain number of times in the observed sequence. However, these results do not imply a satisfactory bound on distribution estimation. In particular, we still do not have an estimator for the stationary distribution in a frequency-by-frequency sense that estimates the vector of probabilities $(M^{\pi}_{\zeta})_{\zeta = 0}^n$ consistently in any natural divergence measure.
Finally, we mention that besides frequency-by-frequency stationary distribution estimation, researchers have studied various learning problems involving other parameters of Markov chains \citep[e.g.,][]{hao2018learning,han2021optimal,wolfer2019minimax}. 
%However, the size of the state space appears explicitly as a parameter in these results.

In this paper, we propose and theoretically analyze an estimator for the stationary distribution in a frequency-by-frequency sense. Specifically, we consider a stochastic process $X^n = (X_1,X_2,\ldots,X_n)$ defined over a finite but unknown state space $\Xspace$ in which we accommodate the regime $|\Xspace| \gg n$. We assume that the process is exponentially $\alpha$-mixing ---which subsumes mixing Markov chains and some hidden Markov models; see Section~\ref{sec:mixing_SP} to follow --- and aim to estimate its stationary distribution $\pi$ in the aforementioned sense. Specifically, 
letting $M^{\pi}_{\zeta}$ denote the probability mass placed by the stationary distribution on elements occurring $\zeta$ times in $X^n$, 
our goal is to estimate the vector $(M^{\pi}_{\zeta})_{\zeta = 0}^n$ in total variation distance. Importantly, we would like to develop a universal result like in the i.i.d. case, which shows that estimation can always be performed consistently in $n$. We show that such an estimator then naturally yields an estimator for the stationary distribution $\pi$ defined over the sample space $\Xspace$ (see Lemma~\ref{lemma:oracle_inequality}).

\subsection*{Contributions and organization}
Our main contributions are summarized below:
\begin{itemize}
    \item We propose a simple and efficient estimator for the stationary distribution of any exponentially $\alpha$-mixing stochastic process. Our estimator combines the $\Wingit$ estimator (\citet{pananjady2024just}) for low frequency and the plug-in estimator for high frequency symbols.
    \item In Theorem \ref{theorem:competitive_regret_theorem}, we provide a risk bound on the proposed hybrid estimator, showing that with the appropriate parameter settings it attains total variation risk  
    $\widetilde{\mathcal{O}}\left(n^{-1/6} \sqrt{\mathsf{t_{mix}}} \right)$ for any exponentially $\alpha$-mixing process with mixing time $\mathsf{t_{mix}}$ (see Eqs.~\eqref{eq:alpha_mixing} and~\eqref{eq:mixing_time}). 
    This rate recovers known guarantees in the i.i.d. setting as a special case.  
\end{itemize}
Our main result, Theorem~\ref{theorem:competitive_regret_theorem}, is proved through two key technical pieces of possible independent interest.
First, Proposition~\ref{lemma:empirical_TV_bound}, presented in Section \ref{sec:theoretical_results}, provides a high-probability bound on the $\ell_1$ error of the plug-in estimator. Notably, these results recover the known guarantees for the i.i.d. setting as special cases, but our analysis does not rely on Poissonization, a common technique tailored to the i.i.d. setting used to handle dependent multiplicities of symbols. Instead, we develop an alternative approach based on blocking arguments, which decomposes the sequence into approximately independent blocks. In the process, we prove concentration inequalities for a natural self-normalized statistic for mixing stochastic processes.

Second, we establish a bound on the $\ell_1$-risk of the $\Wingit$ estimator for a general exponentially $\alpha$-mixing process in Proposition~\ref{lemma:wingit_TV_bound}, presented in Section~\ref{sec:theoretical_results}.
Importantly, simply applying the worst-case bounds provided in~\citet{pananjady2024just} would not yield Theorem~\ref{theorem:competitive_regret_theorem}, and so
we provide a sharper bound on the $\ell_1$-risk that \emph{adapts} to the behavior of the stationary distribution $\pi$. Specifically, if $\pi$ assigns low probability to low-frequency symbols, the bound naturally adjusts to reflect this, leading to a smaller error for the $\Wingit$ estimator.

The rest of this paper is organized as follows. Section~\ref{sec:setup} describes basic background and formalizes our problem statement. Section~\ref{sec:estimator} describes the methodology for our estimator, and our main results for the analysis of this estimator are described in Section~\ref{sec:theoretical_results}. In Section~\ref{sec:expts}, we simulate our estimator and compare its performance to alternatives that are available in the literature.  In Section~\ref{sec:proofs-main}, we provide proofs of our main results, deferring statements and proofs of auxiliary lemmas to the appendix.

\subsection*{Notation}
Let $[n]$ denote the set of natural numbers less than or equal to $n$. For an index set $P \subseteq [n]$, let $X_P = (X_i)_{i \in P}$ denote the sequence of random variables with indices in $P$, ordered canonically. For a sequence $Z \in \Xspace^{\mathbb{N}}$ and $x \in \Xspace$, we let $N_x(Z) = \#\{i: Z_i = x \}$ denote the number of occurrences of $x$ in $Z$. We frequently use the shorthand $N_x = N_x(X^n)$. We also let $\varphi_\zeta(Z) = \#\{x: N_x(Z) = \zeta\}$ denote the number of symbols occurring with frequency $\zeta$ in $Z$, once again adopting the shorthand $\varphi_{\zeta} = \varphi_{\zeta}(X^n)$. We use $\Delta(S)$ to denote the set of all probability mass functions on a finite set $S$. We use $\oplus$ to denote the concatenation of sequences, e.g., $(X_1, X_2) \oplus (X_4, X_5) = (X_1, X_2, X_4, X_5)$. 
We use the notation $f(u) \lesssim g(u)$ to mean that there exists some absolute positive constant $C$ that is independent of all problem parameters, such that $f(u) \le C \cdot g(u)$ for all $u$ in the domain of $f$ and $g$. We use the notation $f(u) \gtrsim g(u)$ when $g(u) \lesssim f(u)$. We write $f(u) \asymp g(u)$ if both relations $f(u) \gtrsim g(u)$ and $g(u) \lesssim f(u)$ hold. Logarithms are taken to the base $e$. We use $(c, C)$ to denote universal positive constants that could be different in each instantiation. We use $\var(X)$ to denote the variance of a random variable $X$ and $\cov(X, Y)$ to define the covariance between two random variables $X$ and $Y$. For a sequence of (deterministic or random) scalars $a^n = (a_1, \ldots, a_n)$, let $\Var(a^n) := \frac{1}{n(n-1)} \sum_{i = 1}^n \sum_{j = i+1}^n (a_i - a_j)^2$ denote its spread.

\section{Background and problem formulation}\label{sec:setup}

As mentioned before, we are interested in estimating the stationary distribution of a mixing stochastic process $X^n$. In this section, we formally set up the assumptions on our stochastic process as well as the metric in which we wish to perform estimation.

\subsection{Model for the stochastic process}\label{sec:mixing_SP}

We assume that the sequence $X^n := (X_1, \ldots, X_n)$ is an ergodic stochastic process defined over a finite (and possibly unknown) state space $\Xspace$. The unique stationary distribution of this ergodic process is denoted by $\pi$ and we assume that the process is initialized at stationarity\footnote{This assumption can easily be relaxed by allowing for a burn-in period.} with $X_1 \sim \pi$. 

We also assume that the stochastic process is exponentially $\alpha$-mixing, which ensures that the dependencies between lagged past and future observations in the process decay exponentially with the time lag. 
Formally, for an ergodic stochastic process $\{X_t\}_{t \geq 1}$, define its $\alpha$-mixing coefficient as
\begin{subequations} \label{eq:exp-alpha-mixing}
\begin{align}\label{eq:alpha_mixing}
    \alpha(\tau) \defn \sup_{t \in \mathbb{N}} \; \sup \big\{ |\mathbb{P}(A \cap B) - \mathbb{P}(A)\mathbb{P}(B)| : A \in \sigma(X_t^-), B \in \sigma(X_{t+\tau}^+) \big\},
\end{align}
where at time $t$, we use $\sigma(X_t^-)$ to denote the $\sigma$-algebra generated by the past RVs $(X_1, X_2, \ldots, X_t)$, and $\sigma(X_{t+\tau}^+)$ to denote the $\sigma$-algebra generated by the future RVs $(X_{t+\tau}, X_{t+\tau+1}, \ldots)$.
Then, the stochastic process is said to satisfy \emph{exponential $\alpha$-mixing} if there exist constants $\cm > 0$ and \(\rho \in (0,1)\) such that 
\begin{align} \label{eq:mixing_condition}
    \alpha(\tau) \leq \cm \rho^\tau \quad \text{for all } \tau \geq 1.
\end{align}
\end{subequations}
We also define the mixing time of such a process to level $\epsilon \in (0, 1)$ as
\begin{align}
    \label{eq:mixing_time}
    \tmix(\epsilon) \defn \min \{ \tau \in \mathbb{N}: \alpha(\tau) \leq \epsilon \}.
\end{align}
Note that for an exponentially mixing process satisfying Eq.~\eqref{eq:mixing_condition}, we have
\begin{align} \label{eq:mixing-time-bound}
\tmix(\epsilon) \leq \frac{\log(\cm/\epsilon)}{\log(1/\rho)}.
\end{align}
The assumption of exponential \(\alpha\)-mixing is satisfied by a wide range of stochastic processes, making it a versatile and practical modeling choice. Below, we discuss some examples:

\begin{enumerate}
    \item \textbf{I.I.D. sequences:} IID sequences are special cases of $\alpha$-mixing sequences. Since all samples are i.i.d., we have from Eq.~\eqref{eq:alpha_mixing} that $\alpha(\T) = 0$ for all $\T \in \mathbb{N}$. Then, from the definition of mixing time in Eq.~\eqref{eq:mixing_time} we obtain $\tmix(\epsilon) = 1$ for all $\epsilon >0$.
    
    \item \textbf{Finite-State Ergodic Markov Chains:}  
    Any ergodic Markov chain on a finite state space satisfies the exponential $\alpha$-mixing condition~\eqref{eq:mixing_condition}. In particular, the parameters $(\cm, \rho)$ depend on the size of the state space and the spectral gap of the Markov transition matrix~\citep[Theorem 4.9]{levin2017markov}.

    \item \textbf{Hidden Markov Models (HMMs):}  
    In HMMs, the observations depend on a latent Markov chain. When the latent Markov chain $(H_t)_{t \geq 1}$ is exponentially $\alpha$-mixing and the observed process $(X_t)_{t \geq 1}$ is some measurable function $f_t$ of $H_t$ and another exponentially mixing Markov process $W_t$, then $X_t = f_t(H_t, W_t)$ is also exponentially $\alpha$-mixing \citep{doukhan2012mixing}. 

    \item \textbf{Random duplication model:}  
    Suppose i.i.d. samples from $\pi$ are input to a random duplication model~\citep{chandra2024missing}, 
    which for each input $x$, outputs the duplicated sequence $(\underbrace{x,\ldots, x}_{k \text{ times}})$ with probability $\alpha$, and $x$ otherwise.
    The output process is ergodic and satisfies exponentially $\alpha$-mixing, and the mixing time satisfies $\mathsf{t_{mix}}(\epsilon) = k$ for all $\epsilon > 0$. 
\end{enumerate}

\subsection{Frequency-by-frequency estimation of the stationary distribution}

The total variation (TV) distance between two probability mass functions $p$ and $q$ defined on a common probability space $(\Xspace, \Fspace)$ is given by 
\[
\TV(p, q) = \sup_{A \in \Fspace} |p(A) - q(A)| = \frac{1}{2}\sum_{x \in \Xspace} | p_x - q_x|,
\]
where we have used the shorthand $p_x = p(\{x\})$ and $q_x = q(\{x\})$. 

Recall that our goal is to estimate the vector of probabilities $(M^{\pi}_{\zeta})_{\zeta = 0}^n$ consistently in total variation, and thereby the stationary distribution $\pi$.
Let us set up this problem formally. For each $\zeta = 0, 1, \ldots, n$, define
\begin{align}
    \label{eq:count_probability_mass_defn}
    M^{\pi}_{\zeta} (X^n) := \sum_{x \in \Xspace} \pi_x \cdot \ind{N_x = \zeta}
\end{align}
as the mass placed by $\pi$ on elements occurring exactly $\zeta$ times in the observed sequence $X^n$.
Note that each scalar $M^{\pi}_{\zeta} (X^n)$ is a random functional, since it depends both on the underlying distribution and on the realized sequence $X^n$. Denote the vector of count probabilities by
\begin{align}
    \label{eq:vector_count_probabilities}
    M^{\pi} (X^n) = (M^{\pi}_{0} (X^n), M^{\pi}_{1} (X^n), \ldots,M^{\pi}_{n} (X^n)),
\end{align}
which is a random vector taking values in the probability simplex $\Delta(\{ 0, 1, \ldots, n\})$.
 
 Our goal is to develop an estimator of the random vector $M^{\pi}(X^n)$ as a functional of only the observed sequence $X^n$. In particular, we seek to design an estimator $\Mhat: \Xspace^n \to \Delta(\{ 0, 1, \ldots, n\})$ and measure its error using total variation distance. Specifically, define the risk
 \begin{align}
     \mathcal{R}_n (M^{\pi}, \Mhat) := \EE \left[ \TV ( M^{\pi}(X^n), \Mhat(X^n)) \right],
 \end{align}
 where the expectation is taken over $X^n$ and any additional randomness in the estimator. 

While the problem of estimating the vector $M^{\pi}$ is interesting in itself, it is worth pausing to ask in what sense this yields an estimator for the stationary distribution $\pi$. To describe this, we recall the notion of a \emph{natural} estimator due to~\citet{orlitsky-suresh}, which is an estimator of $\pi$ that assigns the same probability to all symbols appearing with the same frequency in $X^n$. Formally, define the set of all natural estimators
\[
\mathcal{Q}^{\nat} := \{q: \Xspace^n \to \Delta(\Xspace) \mid q_x = q_y \text{ if } N_x(X^n) = N_y(X^n)\}.
\]
Note that natural estimators are indeed intuitive, since if we know only that a sequence is exponentially $\alpha$-mixing and therefore reaches stationarity, then it is \emph{natural} to assign the same probability to symbols that occur an equal number of times.
In particular, any estimator $\Mhat: \Xspace^n \to \Delta(\{0, 1, \ldots, n \})$ of the count probability vector can be used to generate an estimator $\qhat$ for the stationary distribution by dividing the mass $\Mhat_{\zeta}$ equally among all elements of $\Xspace$ that occur $\zeta$ times in $X^n$.
The estimator $\qhat$ is natural by definition. The following lemma shows that this estimator is competitive with respect to the class of natural estimators $\mathcal{Q}^{\nat}$.
 \begin{lemma}
    \label{lemma:oracle_inequality}
    Suppose $\qhat: \Xspace^n \to \Delta(\Xspace)$ is generated from an estimator $\Mhat: \Xspace^n \to \Delta(\{0,1,\ldots, n\})$ by splitting the mass $\Mhat_{\zeta}$ equally among all elements that occur exactly $\zeta$ times in the sequence $X^n$. Then, we have
    \begin{align}
        \TV(\pi,\qhat(X^n)) \leq 2 \cdot \inf_{q \in \mathcal{Q}^{\nat}} \TV(\pi, q) + \TV (\Mhat(X^n),M^{\pi}(X^n)).
    \end{align}
\end{lemma}
Note that the infimum on the RHS is taken over all natural estimators, including those that have perfect knowledge of $\pi$ but are constrained to assign the same probability to symbols appearing the same number of times in $X^n$. We thus have an \emph{oracle inequality} with approximation factor $2$, and estimating the vector of count probabilities automatically yields a good estimator of the stationary distribution. As an aside, we note that such a competitive relation was proved for the KL-divergence by~\citet{orlitsky-suresh}, who showed an approximation constant of $1$ instead of as above. A corresponding result for the TV distance has not appeared before (to our knowledge).

\section{Methodology} \label{sec:estimator}

We now turn our attention to the estimation problem set up above. Our distribution estimation algorithm is based on a careful combination of the \Wingit estimator \citep{pananjady2024just} for low-frequency symbols and the plug-in estimator for high-frequency symbols. Let us begin by describing these two estimators.

\subsection{\Wingit Estimator}

The \Wingit estimator generalizes the Good-Turing estimator's leave-one-out technique to provide a ``leave-a-window-out" framework for estimation with dependent data.
To describe this framework, we define some additional notation. For each index $i \in [n]$ in the sequence, define the following index sets:
\begin{align} \label{eq:index-sets}
\Dset_i = \{k \in [n]: |k - i| < \tau \} \quad \text{ and } \quad \Iset_i = [n] \setminus \Dset_i.
\end{align}
For a suitable choice of $\tau$, the set of indices $\Dset_i$ represents the ``dependent set", in that the random variables $\{ X_j \}_{j \in \Dset_i}$ may depend strongly on the random variable $X_i$. Conversely, $\Iset_i$ represents the ``independent set": If $\tau$ is chosen large enough and the stochastic process is mixing, then we expect the random variables indexed by the set $\Iset_i$ to be approximately independent of $X_i$. 
With this notation in hand, we define the \emph{exact-count} estimator
\begin{align}\label{eq:small_count_wingit}
    \Mhat_{\tau, \zeta}^{(i)} := \ind{N_{X_i}(X_{\Iset_i}) = \zeta}, \;\; \text{ and the averaged estimator } \;\;    \Mhat_{\some,\zeta}(\tau) = \frac{1}{n} \sum_{i=1}^{n} \Mhat_{\tau, \zeta}^{(i)},
\end{align}
where $\Mhat_{\some,\zeta}(\tau)$ is the total probability mass assigned to symbols appearing $\zeta$ times in the sample by the $\Wingit$ estimator. Note that substituting $\tau=1$ yields the Good--Turing estimator. We note that the definitions of the exact-count $\Wingit$ estimators in our work differ slightly from those in \citet{pananjady2024just}. Specifically, while \citet{pananjady2024just} were concerned with estimating the \emph{small-count} probability given by $M^{\pi}_{\leq \zeta} = \sum_{s = 0}^{\zeta} M^{\pi}_s$, which corresponds to the total mass of elements appearing up to $\zeta$ times, our focus is on estimating the exact count probabilities $M^{\pi}_{\zeta}$, which are objects that are more directly useful for distribution estimation.

We now present an explicit implementation of the $\Wingit$ estimator defined in Eq.~\eqref{eq:small_count_wingit}.
We implement the $\Wingit$ estimator in parallel for all values of $\zeta \in \{0,\ldots,n\}$ and return the $n+1$-dimensional vector $\Mhat_{\Wingit}$. This implementation runs in $\order(n\tau)$ time, where $\tau$ is the window size used. 
%Thus in the i.i.d. case when we choose $\tau=1$, this reduces to the linear time implementation of the Good--Turing estimator.
Given a sequence $X^n$ and a natural number $\tau$, our implementation proceeds via two passes through the data as shown in Algorithm~\ref{alg:wingit}.
\begin{algorithm}[H]
\caption{Linear time implementation of the \textsc{WingIt} estimator.}\label{alg:wingit}
\begin{algorithmic}[1]
\REQUIRE Sequence $X^n=(X_1,\ldots,X_n)$, Natural number $\T$ (window size)
\STATE Initialize \texttt{locations} as a dictionary
\FOR{$i=1,\ldots,n$}
    \IF{$X_i\notin$ \texttt{locations}}
        \STATE Initialize \texttt{locations}$[X_i]$ as a list
    \ENDIF
    \STATE Append $i$ to \texttt{locations}$[X_i]$
\ENDFOR
\STATE $\Mhat_{\Wingit} \gets [\underbrace{0,\ldots,0}_{n+1}]$
\FOR{$i=1,\ldots,n$}
        \STATE count $\gets 0$
        \FOR{$j=\max(0,i-\tau),\ldots,\min(i+\tau,n)$}
            \IF{$X_j == X_i$}
               \STATE count $\gets$ count $+ 1$
            \ENDIF
        \ENDFOR
        \STATE $\zeta \gets \texttt{len}(\texttt{locations}[X_i]) -$ count
        \STATE $\Mhat_{\Wingit}[\zeta] \gets \Mhat_{\Wingit}[\zeta] + 1/n$
\ENDFOR
\RETURN $\Mhat_{\Wingit}$
\end{algorithmic}
\end{algorithm}

The first loop in Algorithm~\ref{alg:wingit} (Steps 2-7) requires a single pass through the data and thus runs in $\order(n)$ time. This loop stores the locations of each symbol in the dictionary $\texttt{locations}$. In Step 8 we initialize the vector of count probabilities $\Mhat_{\Wingit}$ as a vector of zeros of size $n+1$ corresponding to each frequency. The $\zeta$-th element of this vector, denoted by $\Mhat_{\Wingit}[\zeta]$, corresponds to the quantity $\Mhat_{\Wingit,\zeta}(\tau)$ defined in Eq.~\eqref{eq:small_count_wingit}. The second loop in this algorithm (Steps 9-18) runs in $\order(n\tau)$ time. The outer loop performs a single pass through the data, which takes $\order(n)$ time, and the inner loop (Steps 11-15) counts the number of occurrences of a symbol in a window around its current occurrence, which takes $\order(\tau)$ time. The variable $\zeta$ counts the number of occurrences of the symbol outside the window in Step 16, which can be implemented in $\order(1)$ time on Python. 
Finally, the algorithm returns the vector of count probabilities of the $\Wingit$ estimator and runs in $\order(n\tau)$ time.
% \vmcomment{One small thing that confused me after reading the pseudocode is the notation. We should consider making $\Mhat_{\Wingit}$ a vector and saying that $\Mhat_{\Wingit}[\zeta]$ exactly corresponds to the definition $\Mhat_{\Wingit,\zeta}(\tau)$ as defined in (9).}

\subsection{Plug-in estimator}
Recall that we denote by $\varphi_{\zeta}$ the number of symbols appearing $\zeta$ times in $X^n$. Then, the plug-in (or empirical) estimator assigns the following total mass to all symbols with multiplicity $\zeta$:
\begin{align}\label{eq:empirical_estiamtor}
    \Mhat_{\emp,\zeta} = \varphi_{\zeta} \cdot \frac{\zeta}{n}.
\end{align}
Clearly, the plug-in estimator (simultaneously for all values of $\zeta$) is linear-time implementable.

When a particular symbol occurs sufficiently many times, the law of large numbers~\citet{} guarantees that the relative frequency of the symbol in the sample $X^n$ provides a reliable estimate of its stationary probability.
Therefore, we expect this estimator to be accurate for sufficiently large frequencies $\zeta$. Conversely, when $\zeta$ is small, the plug-in estimator $ \Mhat_{\emp,\zeta}$ will typically underestimate $M^{\pi}_{\zeta}$.

\subsection{Combined estimator}
To leverage the complementary strengths of the \Wingit estimator (which we expect to be inaccurate for large $\zeta$) and the plug-in estimator (which we expect to be inaccurate when $\zeta$ is small), we define a hybrid estimator that switches between the $\Wingit$ estimator and the plug-in estimator based on the symbol multiplicity $\zeta$. Specifically, our estimator takes the following form:
\begin{align}
    \label{eq:final_estimator}
    \Mhat_{\zeta}(\tau; \overline{\zeta}) & =
    \begin{cases}
    \nu^{-1} \cdot \Mhat_{\some,\zeta}(\tau) \quad
 &\text{ if }\zeta \leq \overline{\zeta}\\
 \nu^{-1} \cdot \Mhat_{\emp,\zeta}  &\text{ if } \zeta > \overline{\zeta},
\end{cases}
\end{align}
where $\overline{\zeta} \in \mathbb{N}$ is some parameter to be chosen, and $\nu$ is chosen to be a normalizing constant to ensure that $\sum_{\zeta=0}^n \Mhat_{\zeta}(\tau) =1$. 
Note that $\overline{\zeta}$ determines the \emph{transition point}, or the threshold for transitioning from the \Wingit to the plug-in estimator, and we make a specific choice of $\overline{\zeta}$ in stating our theorem to follow. When $\overline{\zeta}$ and $\tau$ are clear from context, we often use the shorthand $\Mhat_{\zeta} \defn \Mhat_{\zeta}(\T;\overline{\zeta})$ to denote the scalar estimate and $\Mhat \defn  (\Mhat_{\zeta}(\T; \overline{\zeta}))_{\zeta = 0}^n$ to denote its vector counterpart.

The combined estimator can be implemented in linear time using the linear time implementation of the plug-in estimator and that of the $\Wingit$ estimator in Algorithm~\ref{alg:wingit}.

\section{Main results}\label{sec:theoretical_results}
Our main result is a bound on the TV error attained by our hybrid estimator $\Mhat(\tau; \overline{\zeta})$.

\begin{theorem}
    \label{theorem:competitive_regret_theorem}
    There exists a universal positive constant $C$ such that if we choose the window size $\tau \geq \mathsf{t_{mix}}(n^{-5})$ and transition point $\overline{\zeta} = \lfloor n^{1/3} \rfloor - 1$, then
    \begin{align*}
        \mathcal{R}_n(M^{\pi}, \Mhat(\tau; \overline{\zeta})) \leq C \cdot\left( \frac{\sqrt{\T\log\left(Cn\right)}}{n^{1/6}}
        \right).
    \end{align*}
\end{theorem}
A few remarks are in order. First, in the special case of an i.i.d. sequence we have $\mathsf{t_{mix}}(n^{-5}) = 1$,
so we can set $\tau = 1$ and our proposed estimator reduces to the combination of the Good--Turing and plug-in estimators studied in \citet{acharya2013optimal}. In this case, Theorem~\ref{theorem:competitive_regret_theorem} recovers the TV guarantee of \citet{acharya2013optimal}, but without requiring Poissonization (see the discussion below). 

Second, note that by the discussion in Section~\ref{sec:mixing_SP}, we can always choose the window size $\tau = \log(n^5\cm)/\log(1/\rho)$ in a general mixing stochastic process satisfying Eq.~\eqref{eq:mixing_condition}. This parameter depends logarithmically on $n$ and also depends on the mixing parameters $\cm$ and $\rho$, and yields the bound (ignoring $\log \log$ factors)
\[
\EE_{X^n} \left[\TV(M^{\pi}(X^n), \Mhat(\tau; \overline{\zeta})) \right] \lesssim \left(\frac{\log^3(n^5  \cm)/\log^3(1/\rho)}{n}\right)^{1/6}
\]
uniformly over all exponentially $\alpha$-mixing stochastic processes satisfying Eq.~\eqref{eq:mixing_condition}.

Third, we highlight that our proof technique necessarily departs from the ones in~\citet{acharya2013optimal} and \citet{orlitsky-suresh}, which rely on Poissonization. Poissonization essentially translates a sequence of fixed length into a sequence where the total number of samples itself is an independent Poisson random variable, which in turn renders the multiplicity of each individual symbol independent. 
While powerful, Poissonization requires independent random variables in the original sequence, and can no longer be used in our setting.
Instead, our techniques to bound the error of the \Wingit estimator are based on the leave-a-window-out interpretation of the estimator, and careful control of its error in terms of appropriately defined indicator random variables. 
To handle the error of the plug-in estimator, we use a careful blocking technique 
and prove a self-normalized concentration inequality for mixing sequences --- see
Lemma \ref{lemma:alpha_mixing_EB}. Our control on the error of both estimators is in terms of fine-grained quantities --- see Propositions~\ref{lemma:empirical_TV_bound} and~\ref{lemma:wingit_TV_bound}.

As a side remark, we note that~\citet{acharya2013optimal} and~\citet{orlitsky-suresh} also analyze the error of their estimator in KL-divergence, and show that it achieves a rate of $\mathcal{O}(n^{-1/3})$ for IID sequences.
By applying the reverse Pinsker inequality~\citep[Theorem 1]{sason2015reverse} to Theorem~\ref{theorem:competitive_regret_theorem} above, one can show that a minor modification of our estimator\footnote{To apply the reverse Pinsker inequality, we need to ensure that our estimator does not assign $0$ probability  to any $\zeta$. This can be done by adding $1$ to the total count of each $\zeta$, ensuring that $\Mhat_{\zeta}(\T;\overline{\zeta}) > 1/n$ for all $\zeta$.} attains KL risk $\widetilde{\mathcal{O}}((\tau^3/n)^{1/6})$ for general mixing sequences. While consistent, this rate is suboptimal, and proving a faster rate in KL error for our estimator remains an open problem. 

As mentioned above, the proof of our main result relies on a concentration inequality for a self-normalized statistic of exponentially $\alpha$-mixing sequences, which we state below along with a more standard Bernstein-type inequality. We prove the lemma in Section~\ref{sec:lemma_alpha_mixing_EB}.

\begin{lemma}
    \label{lemma:alpha_mixing_EB}
    Fix $\delta,\epsilon>0$ and let $U^n \defn (U_1,U_2,\ldots,U_n)$ denote a sequence of random variables from an ergodic, stochastic process that is exponentially $\alpha$-mixing with parameters $\cm$ and $\rho$, with each $U_j \in [0,B]$ almost surely. 
    Recall the definition of mixing time from Eq.~\eqref{eq:mixing_time}.
    
    \noindent (a) For each fixed $\T \geq \tmix(\epsilon/n)$, the following Bernstein-type inequality holds:
    \begin{align*}
        \left|\EE[U_1] - \frac{1}{n}\sum_{j=1}^n U_j\right| \leq \sqrt{\frac{4\T v^2 \log(1/\delta)}{n}} + \frac{4B \T \log(1/\delta)}{3n}
    \end{align*}
    with probability at least $1-4\delta-\epsilon$. Here $v^2 \defn \frac{1}{\T^2} \var\left(\sum_{j=1}^{\T}U_j\right)$ is the normalized block variance.
    
    \noindent (b) In addition, if $\sum_{j=1}^n U_j \geq 36  B \log(2/\delta) \cdot \tmix(\epsilon/n)$ and $n \geq 24\tmix(\epsilon/n)$, then the following holds:
    \begin{align*}
    \left|\EE[U_1] - \frac{1}{n}\sum_{j=1}^n U_j\right| \leq 82 \cdot \frac{\sqrt{B  \log(2/\delta) \cdot (\sum_{j=1}^n U_j) \cdot \tmix(\epsilon/n)}}{n}
\end{align*}
with probability at least $1-10\delta - 2 \epsilon$.

\end{lemma}

On the one hand, it is worth comparing the Bernstein bound in Lemma \ref{lemma:alpha_mixing_EB}(a) to the Bernstein bound for exponentially $\alpha$-mixing sequences given in \citet[Theorem 2]{merlevede2009bernstein}. The authors in that paper show that if $U^n$ is $\alpha$-mixing and satisfies $\alpha(\tau) \leq \exp(-2c\tau)$, then defining 
\[
r^2 := \sup_{i} \left( \var(U_i) + 2\sum_{j>i} |\cov(U_i,U_j)| \right),
\]
we have the tail bound
\begin{align} \label{eq:melvedere-bound}
    \mathbb{P}\left( \left|\EE[U_1] - \frac{1}{n}\sum_{j=1}^n U_j\right| > t \right) \leq \exp\left( -\frac{Cnt^2}{r^2 + \frac{B^2}{n}+ B\log(n)^2 t}  \right),
\end{align}
where $C$ is an unspecified constant that depends on the parameter $c$ defined in assuming a bound on $\alpha(\tau)$.
By comparison, Lemma \ref{lemma:alpha_mixing_EB}(a) makes the weaker $\alpha$-mixing assumption~\eqref{eq:alpha_mixing} and implies the following tail bound in terms of the mixing time:
\begin{align} \label{eq:lemma2a-conclusion}
    \mathbb{P}\left( \left|\EE[U_1] - \frac{1}{n}\sum_{j=1}^n U_j\right| > t \right) \leq 4 \exp\left( -\frac{nt^2}{4r^2 + \frac{4 \T B }{3}t}  \right)+\epsilon.
\end{align}
In stating Eq.~\eqref{eq:lemma2a-conclusion}, we have used the fact that $\tau v^2 \leq r^2$, which is evident from the following sequence of calculations:
\begin{align*}
    v^2 = \frac{1}{\T^2} \var\left(\sum_{i=1}^{\T} U_i\right) 
    = \frac{1}{\T^2} \left(\sum_{i=1}^{\T} \var( U_i) + 2 \sum_{i=1}^{\T}\sum_{j>i}^{\T}\cov(U_i,U_j)\right) 
    &\leq \frac{\T}{\T^2} \sup_{i} \left(\var( U_i) + 2\sum_{j>i}|\cov(U_i,U_j)|\right).
\end{align*}
Comparing Eqs.~\eqref{eq:lemma2a-conclusion} and~\eqref{eq:melvedere-bound}, we see that our bound is stated explicitly in terms of mixing times and may be more user-friendly under the mixing assumption~\eqref{eq:alpha_mixing}.

On the other hand, Lemma \ref{lemma:alpha_mixing_EB}(b) should be viewed as a concentration inequality on a self-normalized statistic constructed from a mixing sequence --- note that the constants $(24, 36, 82)$ appearing in the theorem are not sharp and can likely be improved. A consequence of this result --- stated in a form that is typical in this literature~\citep{de_la_Pe_a_2004} --- is that for an exponentially $\alpha$-mixing sequence $U^n$ satisfying Eq.~\eqref{eq:alpha_mixing} with $U_j \in [0, B]$ a.s. for all $j \in [n]$, we have
 \begin{align*}
     &\mathbb{P}\left( \frac{\left|\EE[U_1] - \frac{1}{n}\sum_{j=1}^n U_j\right|}{\sqrt{\frac{B\sum_{j=1}^n U_j}{n}}} >  82 \sqrt{\frac{\log(2/\delta) \cdot \tmix(\epsilon/n) }{n}} \text{ and } \frac{B\sum_{j=1}^n U_j}{n} > \frac{36 B^2  \log(2/\delta) \cdot \tmix(\epsilon/n)}{n} \right) \\ &\qquad \qquad \qquad \qquad \qquad \leq 10 \delta + 2\epsilon.
 \end{align*}
Such a self-normalized concentration inequality for exponentially $\alpha$-mixing sequences may be of independent interest.

We conclude this section with the two key propositions that we prove, which provide bounds---for each $\zeta$---on the error made by both the plug-in and \Wingit estimators in approximating $M^{\pi}_{\zeta}$.
\begin{proposition}
    \label{lemma:empirical_TV_bound}
    Consider any fixed $\delta, \epsilon > 0$ and let $\T_0 := \tmix(\epsilon/n^2)$ for convenience. There is a universal positive constant $C$ such that if $n \geq 24 \T_0$, then for each 
    \[
    \zeta \geq \max \Big\{36 \T_0 \log(22n/\delta), 1+\sqrt{(4 + 8\T_0 \log(11n/\delta)} + \frac{4\T_0 \log(11n/\delta)}{3} \Big\},
    \]
    we have
\begin{align*}
     \left| M^{\pi}_{\zeta} -  \Mhat_{\emp,\zeta}\right|\leq C \left( \frac{\sqrt{ \zeta \tmix(\epsilon/n^2) }\cdot\varphi_{\zeta}(X^n)}{n}\sqrt{\log\left(\frac{Cn}{\delta}\right)} \right)
\end{align*}
with probability at least $1-\delta-3\epsilon$.
\end{proposition}

\begin{proposition}
    \label{lemma:wingit_TV_bound}
    There exists a universal positive constant $C$ such that the following holds for all $\zeta$. 
    If the window size is chosen such that $\tau \geq \tmix(n^{-2})$, then we have
    \begin{align*}
        &\EE\left[\left|M_{\zeta}^{\pi}(X^n)-\Mhat_{\some,\zeta}\right|\right] \\
     &\quad \leq C \sqrt{\frac{\T}{n}}\Bigg( \sqrt{\EE[M^{\pi}_{\zeta}]}+  \sqrt{\zeta \log(2\T) \EE[M^{\pi}_{\zeta}]}
     + \sqrt{ \sum_{u=1}^{4\T-2}\frac{ (\zeta + u) }{u}\EE\left[ M^{\pi}_{\zeta+u} \right]} \Bigg) + C\frac{(\zeta+1) \T}{n}.
    \end{align*}
\end{proposition}

It is worth comparing Proposition~\ref{lemma:wingit_TV_bound} to the analogous bound for missing mass with $\zeta = 0$~\citep{pananjady2024just}.
Substituting $\zeta=0$ in the above bound and ignoring the final (lower order) term, we have
\begin{align*}
    \EE\left[\left|M_{0}^{\pi}(X^n)-\Mhat_{\some,0}\right|\right] &\lesssim  
    \sqrt{\frac{\T}{n}}\left(\sqrt{\EE[M^{\pi}_{0}]}+ \sqrt{ \sum_{u=1}^{4\T-2}\EE\left[M^{\pi}_{u}\right] }\right).
\end{align*}
Because $\sum_{u=1}^{4\T-2}\EE\left[M^{\pi}_{u}\right] \leq 1$, this recovers the minimax rate of $\mathcal{O}(\sqrt{\T/n})$ given by~\cite{chandra2022missing}.
Moreover, the bound above is strictly stronger, since the RHS can be significantly smaller when the quantities $\{ \mathbb{E}[M^{\pi}_{u}] \}_{u = 0}^{4\T - 2}$ are small. Note that these are the expected count probability masses for the stationary distribution $\pi$ --- if the small count mass of the stationary distribution is small (i.e. most symbols appear many times), then our bound shows that \Wingit incurs a smaller error for missing mass estimation. This adaptivity is a distinct advantage of our approach.  
In contrast, the $\MSE$ bound provided in \citet{pananjady2024just} captures the correct scaling in terms of $\T$ and $n$, but is not adaptive to this specific behavior of the stationary distribution $\pi$.
This adaptivity property plays a central role in our final bound in Theorem~\ref{theorem:competitive_regret_theorem}, since the errors incurred across all the $\zeta$ values are aggregated to produce the claimed bound on TV error. 

\section{Numerical experiments} \label{sec:expts}
In this section, we provide a set of simulations on synthetically constructed Markov chains in order to demonstrate the effectiveness of the proposed estimator. We first describe the experimental set-up and discuss some of the estimators used in practice. We then compare the performance of these estimators against the proposed hybrid estimator.
\subsection{Experimental Setup}
We perform experiments on samples that we generate from different types of $\alpha$-mixing processes. Specifically, we instantiate the stationary distribution of the mixing sequence from one of five different distributions. These distributions include: the uniform distribution, a step distribution with half of the symbols occurring with probability $1/(2|\Xspace|)$ and the other half of symbols occurring with probability $3/(2|\Xspace|)$, a Dirichlet distribution with the concentration parameter $1.5$, and Zipf distributions with parameters $1.1$ and $1.5$. The Zipf distribution with parameter $a$ assigns the probability $p(i) \propto i^{-a}$ for all $i \in \mathbb{N}$. We then layer a \emph{random duplication model} on the i.i.d. samples generated from the stationary distribution to obtain the final $\alpha$-mixing sequence. This process is described below. For notational convenience, we define $\Tmix \defn \tmix(1/4)$, where $\tmix(\epsilon)$ is defined in Eq.~\eqref{eq:mixing_time}. 
\begin{enumerate}
    \item First, we generate $n_0$ i.i.d. samples from the underlying stationary distribution. The number of symbols is set as $|\Xspace|=n_0$. Thus, as we increase the number of samples, the number of symbols increases proportionally.
    \item We then use a geometric duplication model in order to generate a mixing sequence from the i.i.d. samples. Specifically, for each sample in the i.i.d. sequence, we draw a geometric random variable with mean $n_0^c$ (for some $c > 0$), which determines how many times that symbol is duplicated in the final sequence. This processcan be viewed as a draw from a sticky Markov chain \citep{chandra2022missing} in which the probability of observing an i.i.d. sample from the stationary distribution is given by $1/n_0^c$. The mixing time, $\Tmix$, for this sequence is given by $\Tmix = \mathcal{O}(n_0^c)$ \citep{pananjady2024just}. Since the duplication factors are random, the total length of the resulting sequence, denoted as $n$, is also a random variable with expected value $\mathbb{E}[n] = n_0^{c+1}$. 
    % \apcomment{Need to be more precise here about what $\tmix$ is, connecting back to your notation and definitions from before. Also point out that this is basically a sticky Markov chain.}
\end{enumerate}
We compare the performance of our proposed estimator with several existing estimators that were originally designed for i.i.d.\ sequences. Specifically, we include the Good--Turing + Plug-in estimator analyzed by \citet{acharya2013optimal}, as well as a family of popular add-$\beta$ estimators defined as:
\begin{align*}
    \Mhat_{\zeta}(\beta) = \frac{\varphi_{\zeta}(\zeta + \beta(\zeta))}{\nu},
\end{align*}
where $\nu$ is a normalizing constant chosen to ensure that the estimated probabilities sum to $1$. Different choices of the function $\beta(\zeta)$ yield well-known estimators:
\begin{itemize}
    \item \textbf{Laplace estimator} \citep{laplace1814essai}: $\beta(\zeta) = 1$ for all $\zeta$,
    \item \textbf{Krichevsky--Trofimov estimator} \citep{krichevsky1981performance}: $\beta(\zeta) = 0.5$ for all $\zeta$,
    \item \textbf{Braess--Sauer estimator} \citep{braess2004bernstein}: $\beta(0) = 0.5$, $\beta(1) = 1$, and $\beta(\zeta) = 0.75$ for all $\zeta > 1$.
\end{itemize}

An important detail is that the normalization constant $\nu$ for the add-$\beta$ estimators depends on the alphabet size $|\mathcal{X}|$. For the uniform, step, and Dirichlet distributions, the alphabet size is finite and set to $|\mathcal{X}| = n_0$. However, for the Zipf distribution, the alphabet is the set of all natural numbers, so that $|\Xspace| = \infty$. To address this in our simulations, we approximate the alphabet size for the Zipf distribution using the number of unique symbols observed in the samples. This approximation is reasonable because the missing mass---corresponding to mass placed on unobserved symbols---is small with high probability for Zipf distribution and has a negligible effect on the TV distance.

We compare estimator performance by plotting the TV distance between the true and estimated vectors of count probabilities for each method. We present simulation results for three different values of $c$, the parameter controlling the mixing time of the sequences. All plots are averaged over $50$ independent instances, and the shaded regions around each curve represent the standard deviation across these runs.

% We draw special attention to an intriguing difference in performance of estimation of the Zipf distribution with parameters $1.1$ and $2$ respectively. 
% Specifically, observe that in the case of the Zipf distribution with parameter $2$, the performance of the Good-Turing+Plug-in estimator becomes competitive with our proposed estimator for the cases $c=0.2$ and $c = 0.33$, despite the well-documented bias of the Good-Turing estimator when applied to non-iid data.
% (This good and non-standard performance of the Good-Turing+Plug-in estimator does not manifest for the case of the Zipf distribution with parameter $1.1$.)
% Recall that the Zipf distribution with parameter $a$ has a probability mass function given by $p(i) \propto i^{-a}$ for all $i \in \mathbb{N}$. Thus, for larger values of the parameter $a$ (e.g. $a = 2$), the distribution becomes more concentrated around the lower symbol values, i.e. lower values of $i$. This results in a larger number of samples being drawn from these concentrated regions and most of the sequence being made up of ``large-count" symbols. 
% This implies that the ``small-count" error, even when itself large (say, when the Good-Turing estimator is used), contributes minimally to the overall estimation error, which is instead dominated by the ``large-count" error (for which both hybrid estimators use the plug-in estimator).

% Consequently, for larger values of $a$, such as 2 compared to 1.1, the performance of the the $\Wingit$+Plugin and Good--Turing+Plugin are closer together.
 \begin{figure}
		\centering
            \subfigure[Uniform]{\label{fig:1}\includegraphics[width=7.85cm]{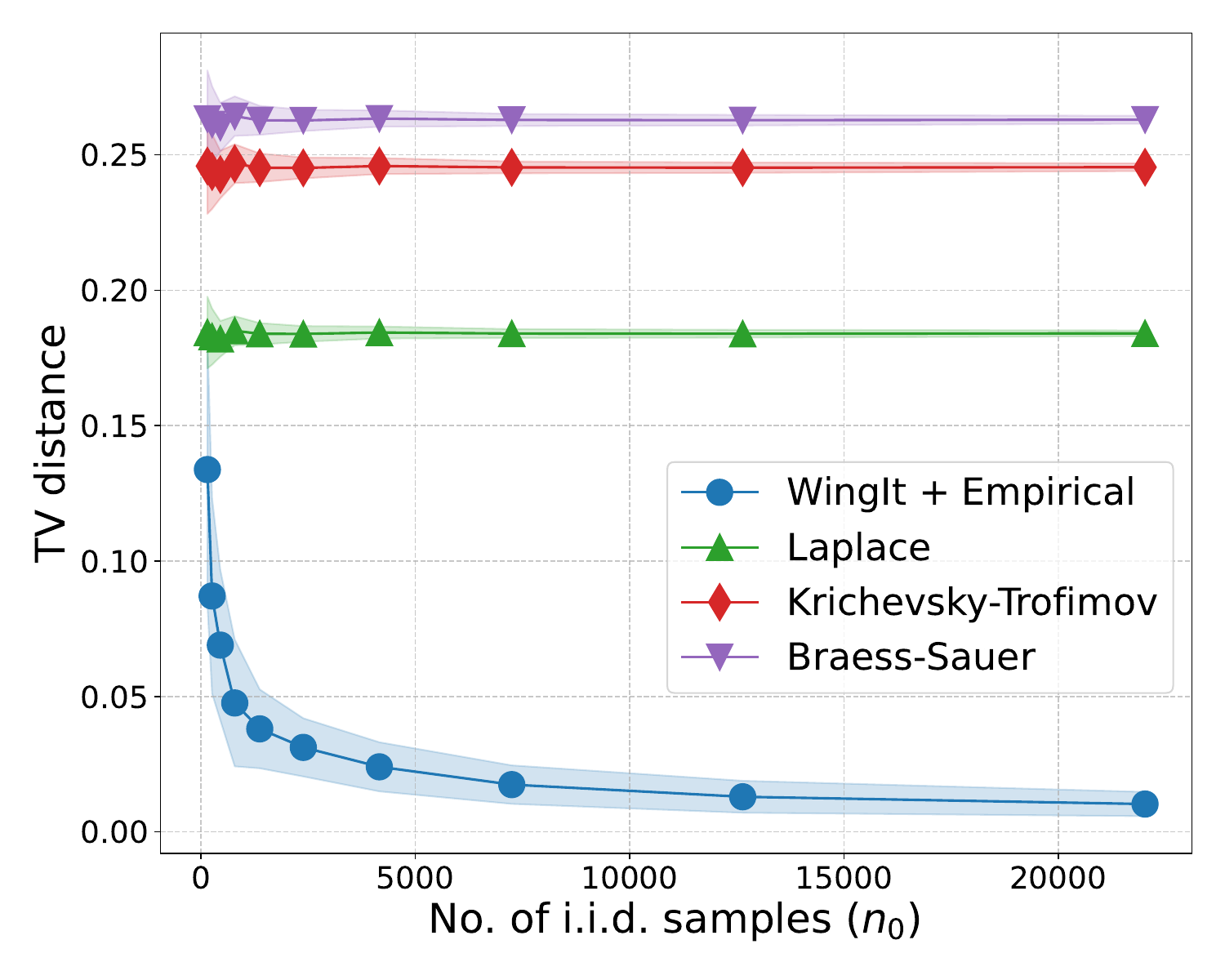}}
            \subfigure[Dirichlet]{\label{fig:2}\includegraphics[width=7.85cm]{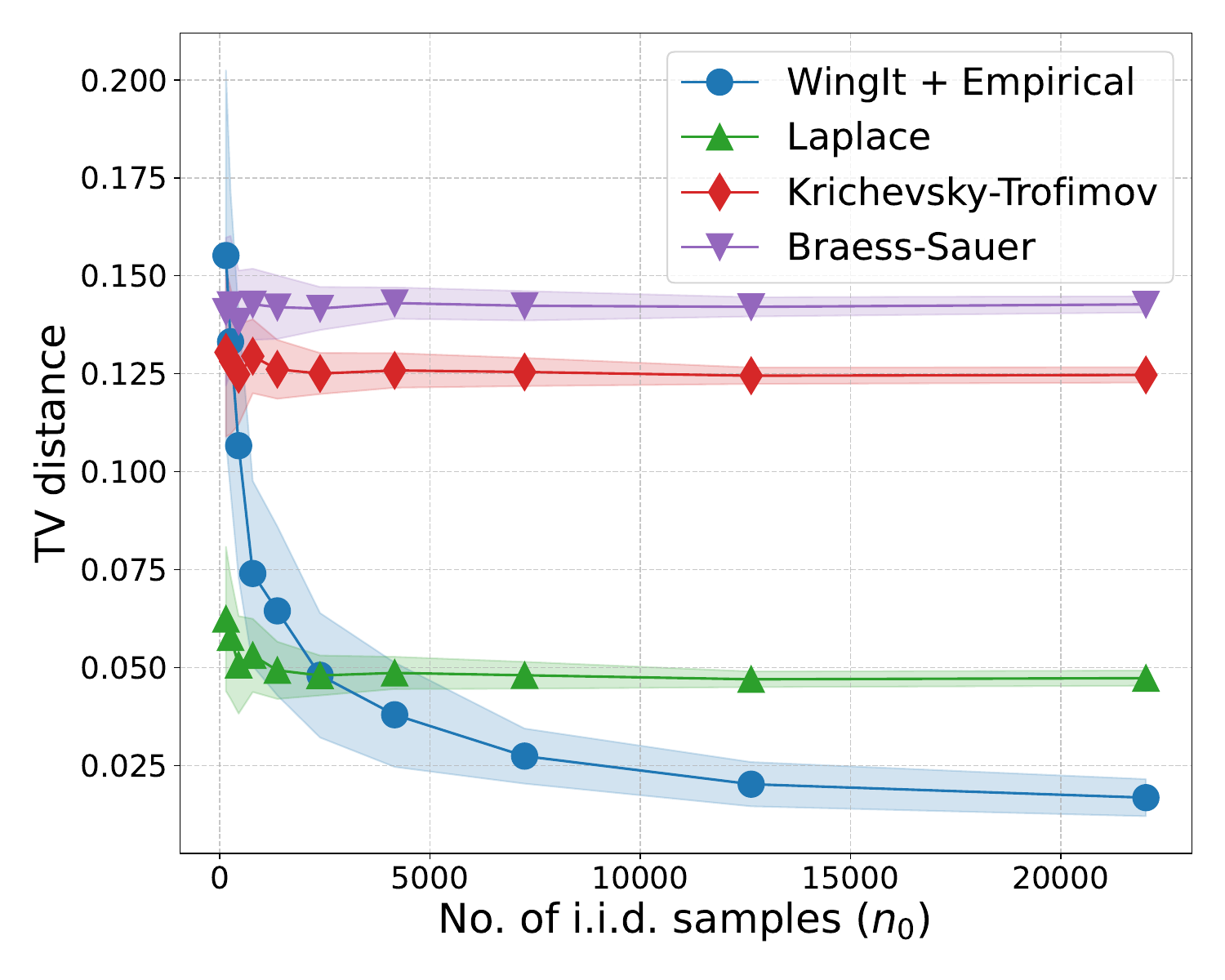}}
            \subfigure[Zipf with parameter $1.1$]{\label{fig:3}\includegraphics[width=7.85cm]{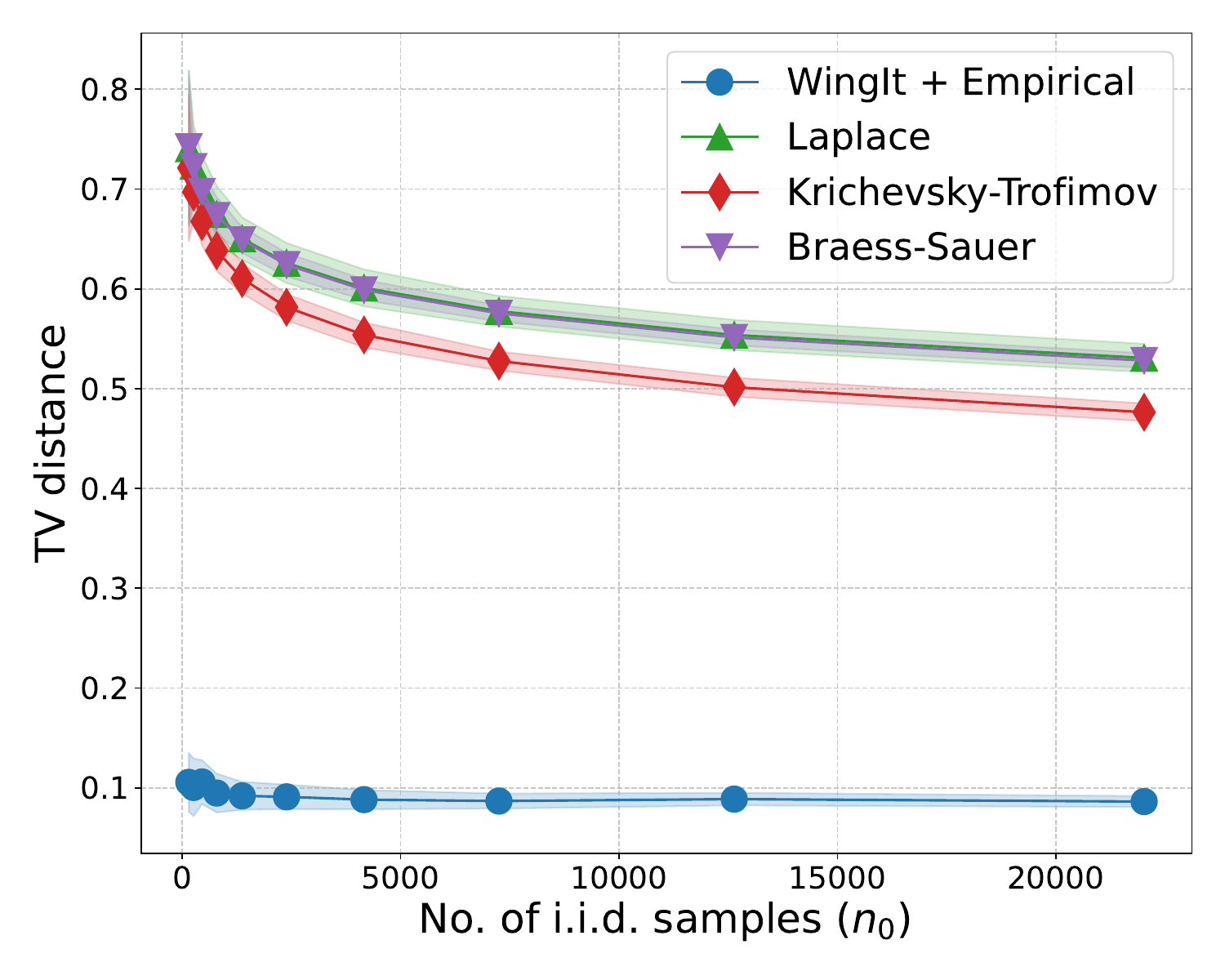}}
            \subfigure[Zipf with parameter $1.5$]{\label{fig:4}\includegraphics[width=7.85cm]{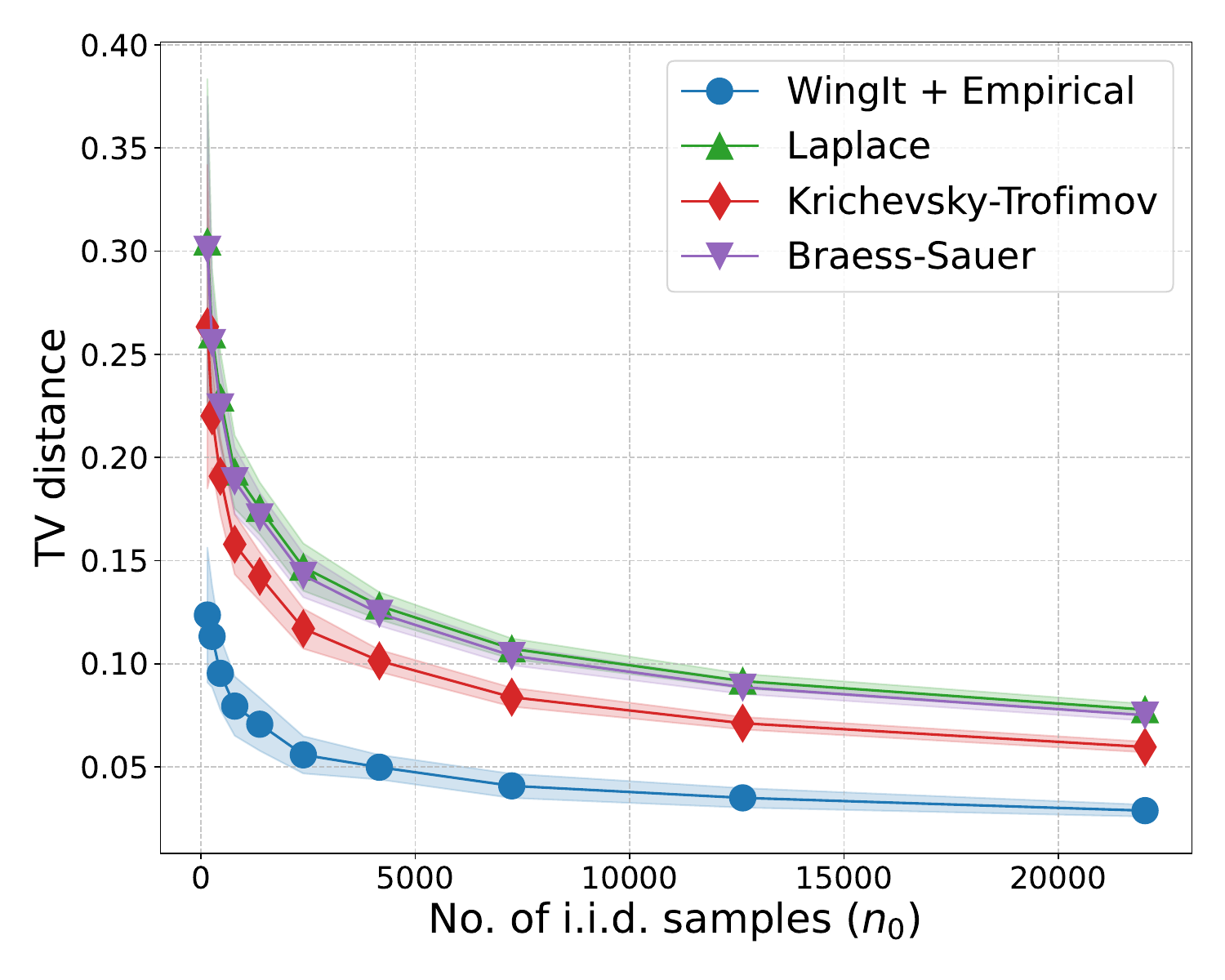}}
        \subfigure[Step]{\label{fig:5}\includegraphics[width=7.85cm]{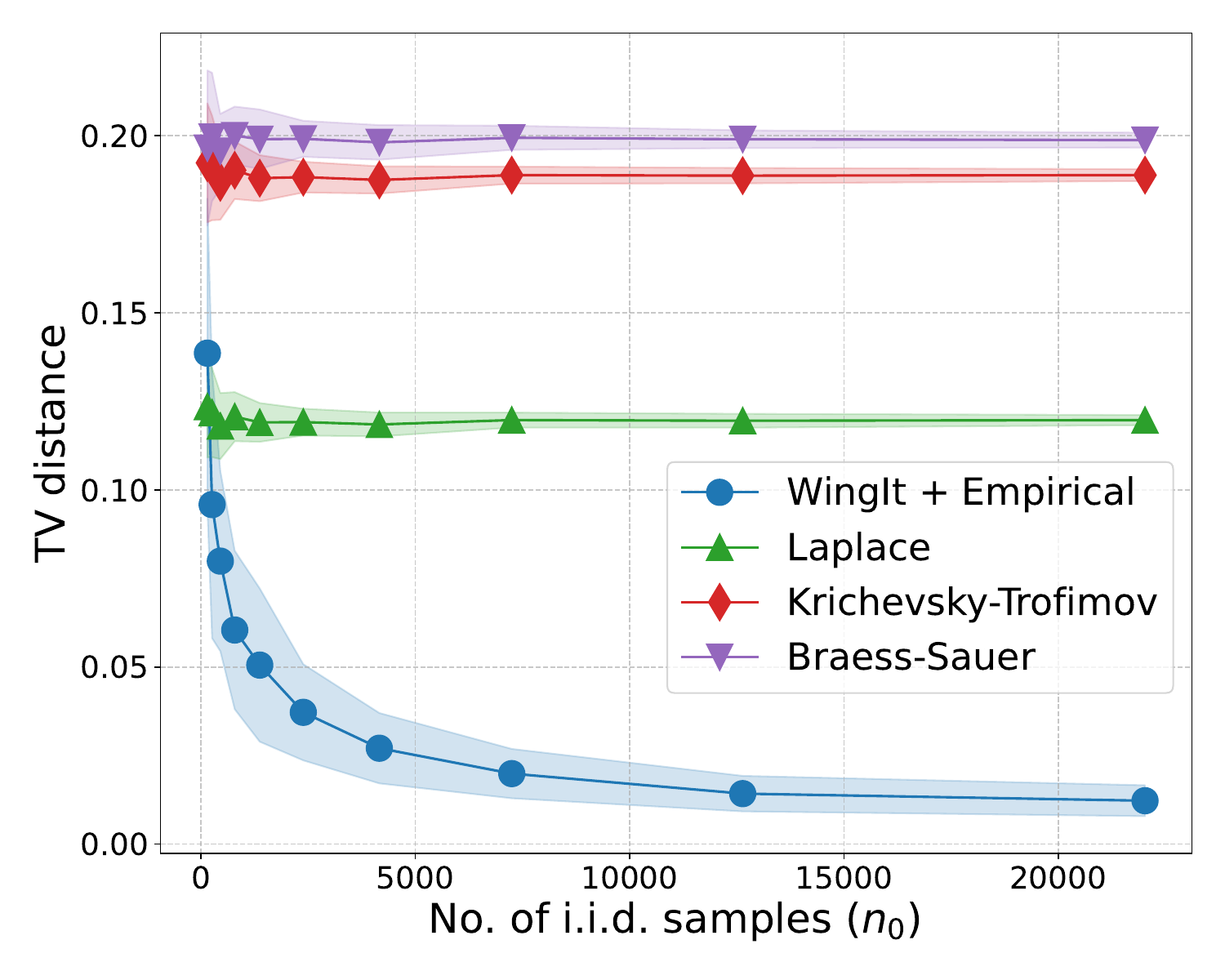}}
		\label{fig:estimation_plots1]}
        \caption{$c = 0.0$ corresponding to the IID sequence for $n_0$ samples, averaged over $50$ instances.}
        \label{fig:c=0.0}
	\end{figure}

    \begin{figure}
		\centering
            \subfigure[Uniform]{\label{fig:6}\includegraphics[width=7.85cm]{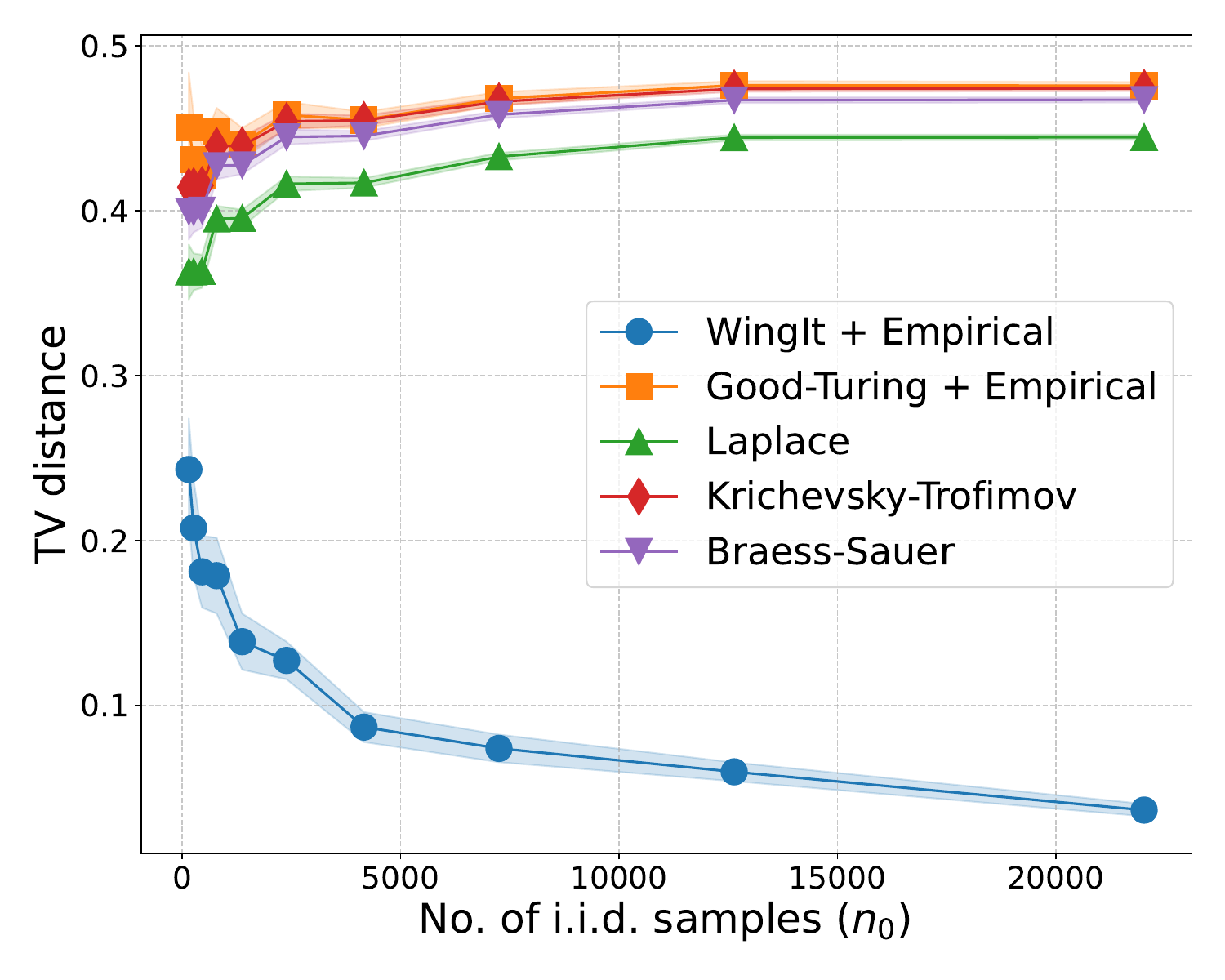}}
            \subfigure[Dirichlet]{\label{fig:7}\includegraphics[width=7.85cm]{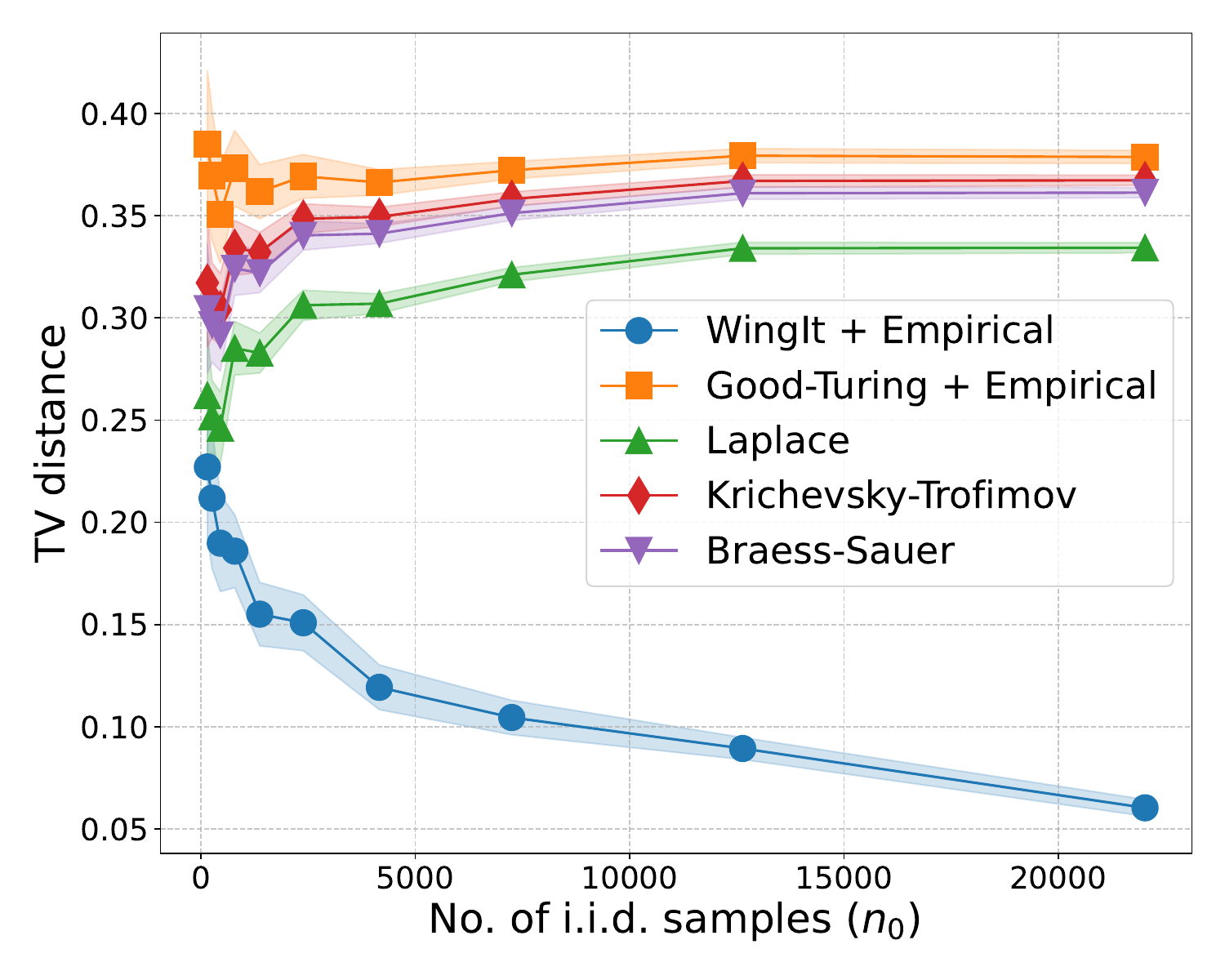}}
            \subfigure[Zipf with parameter $1.1$]{\label{fig:8}\includegraphics[width=7.85cm]{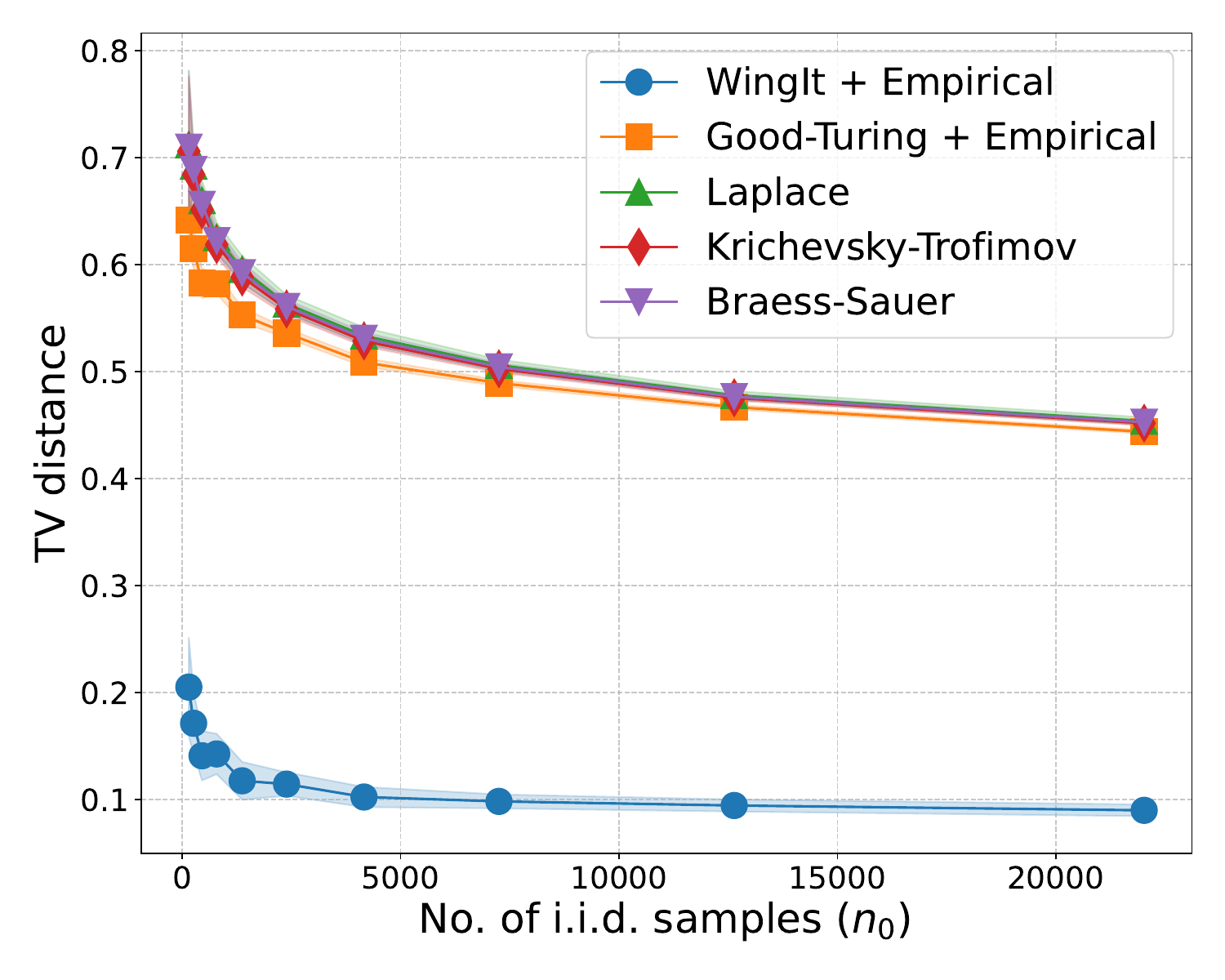}}
            \subfigure[Zipf with parameter $1.5$]{\label{fig:9}\includegraphics[width=7.85cm]{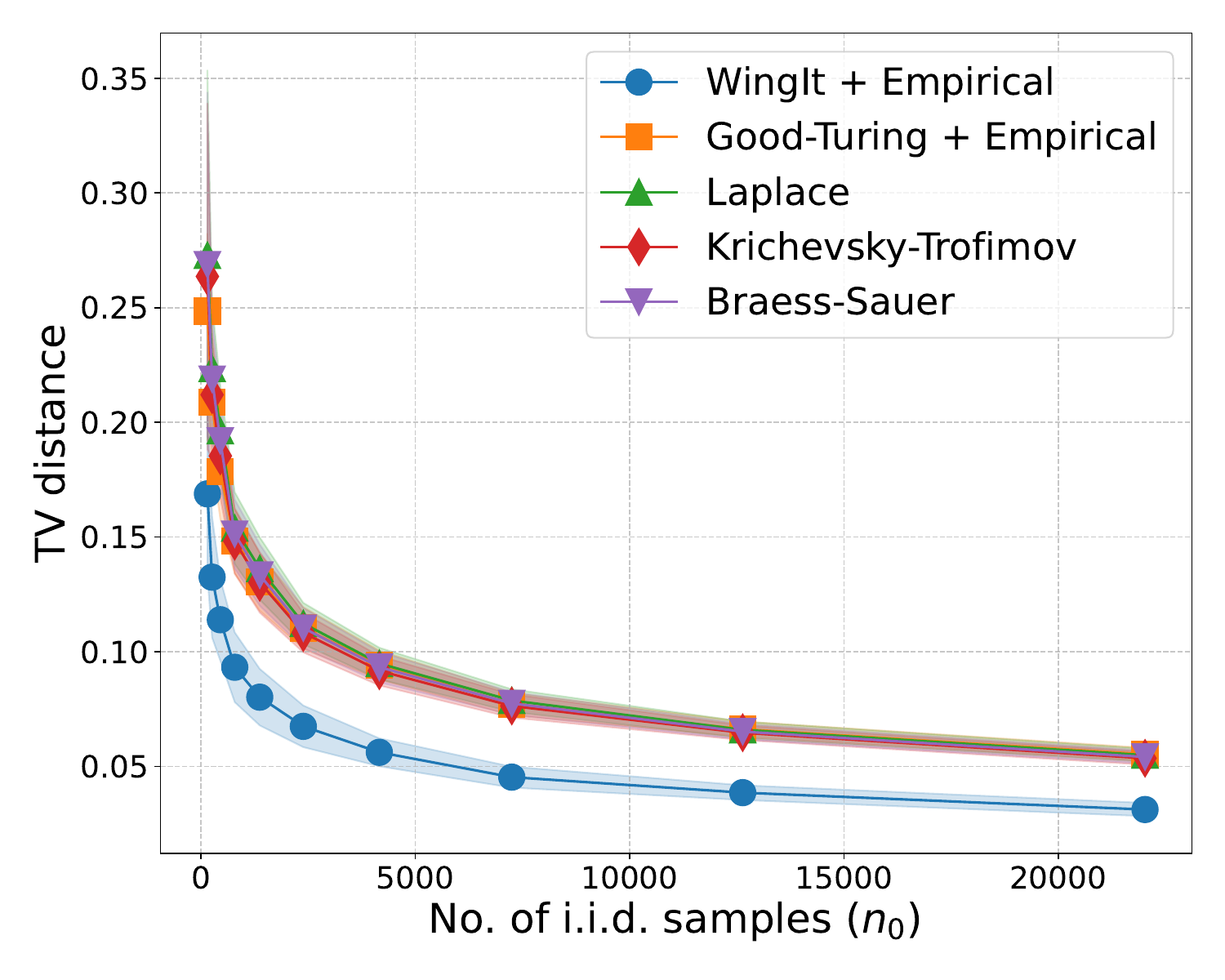}}
        \subfigure[Step]{\label{fig:10}\includegraphics[width=7.85cm]{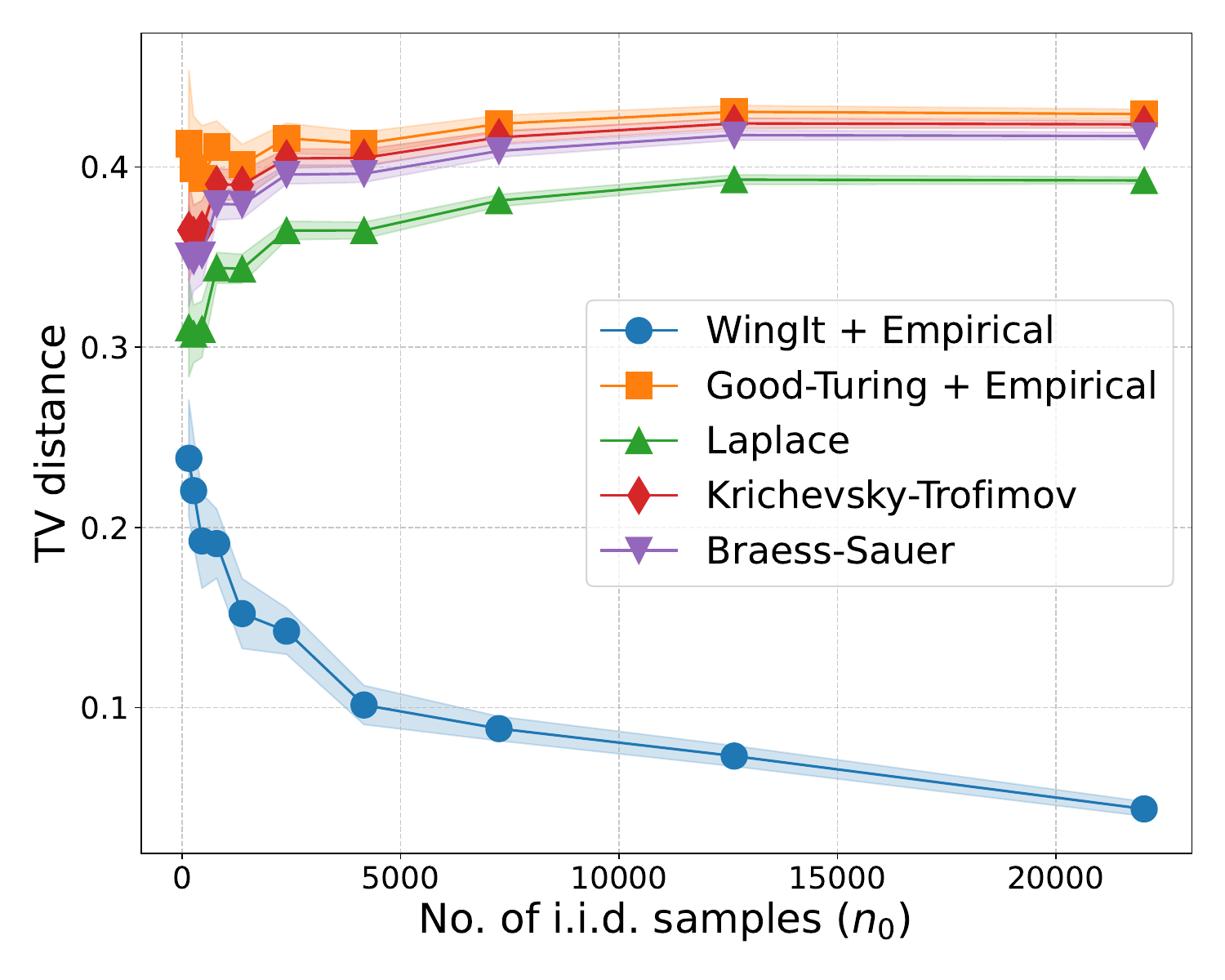}}
		\label{fig:estimation_plots2]}
        \caption{$c=0.2$ corresponding to a mixing sequence with mixing time $\Tmix = \mathcal{O}(n_0^{0.2})$ and expected number of samples $\EE[n] = n_0^{1.2}$, averaged over $50$ instances.}
        \label{fig:c=0.2}
	\end{figure}

    \begin{figure}
		\centering
            \subfigure[Uniform]{\label{fig:11}\includegraphics[width=7.85cm]{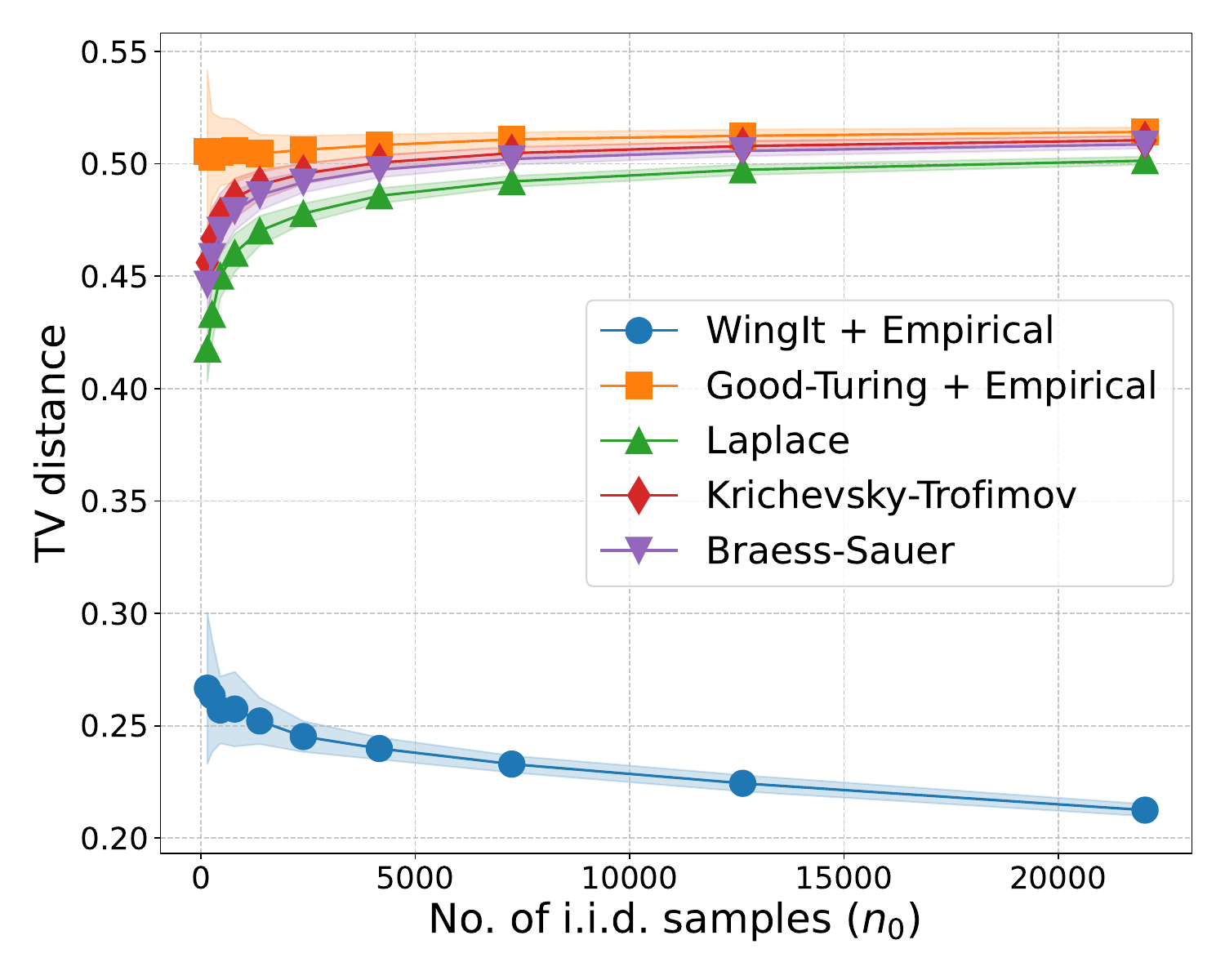}}
            \subfigure[Dirichlet]{\label{fig:12}\includegraphics[width=7.85cm]{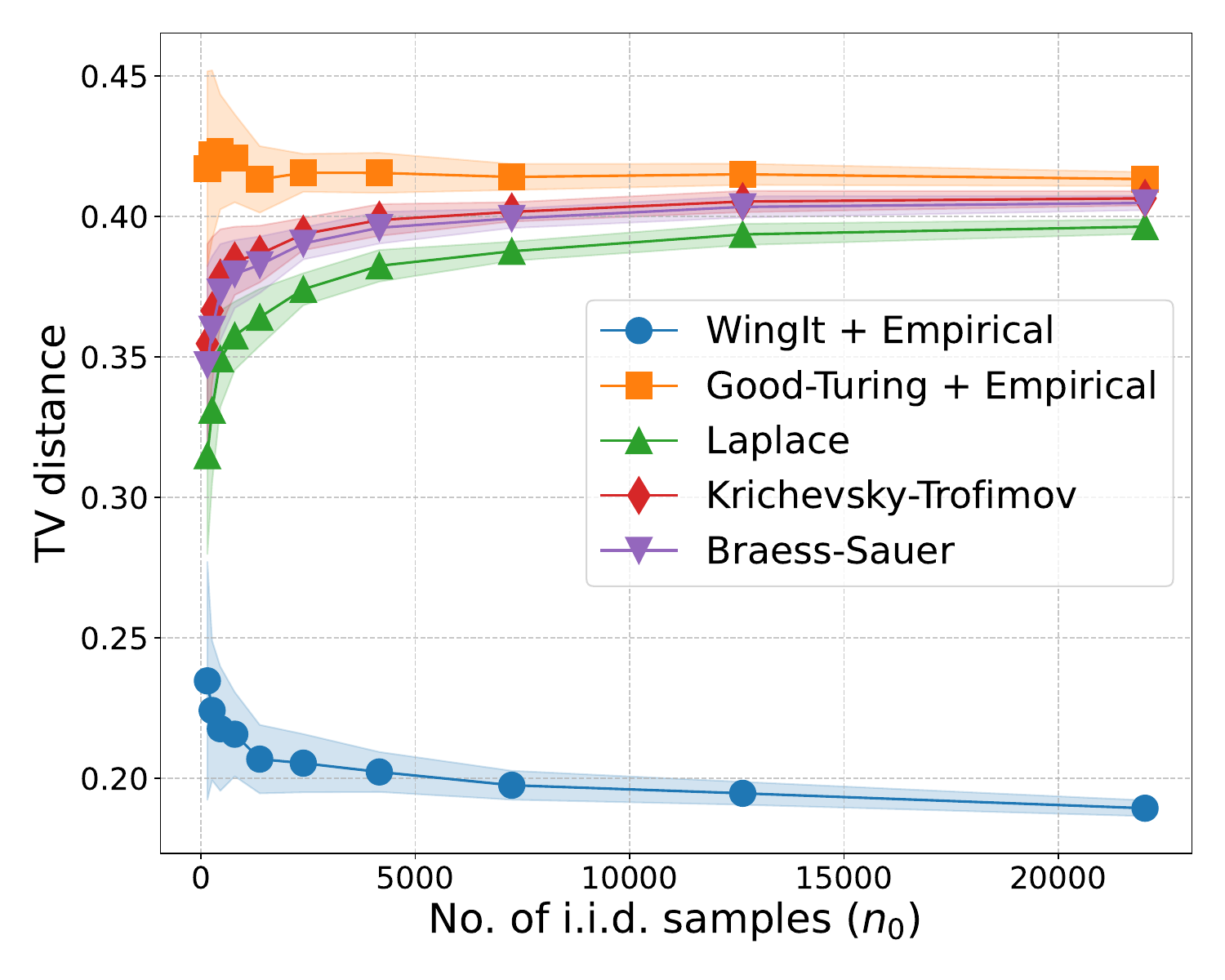}}
            \subfigure[Zipf with parameter $1.1$]{\label{fig:13}\includegraphics[width=7.85cm]{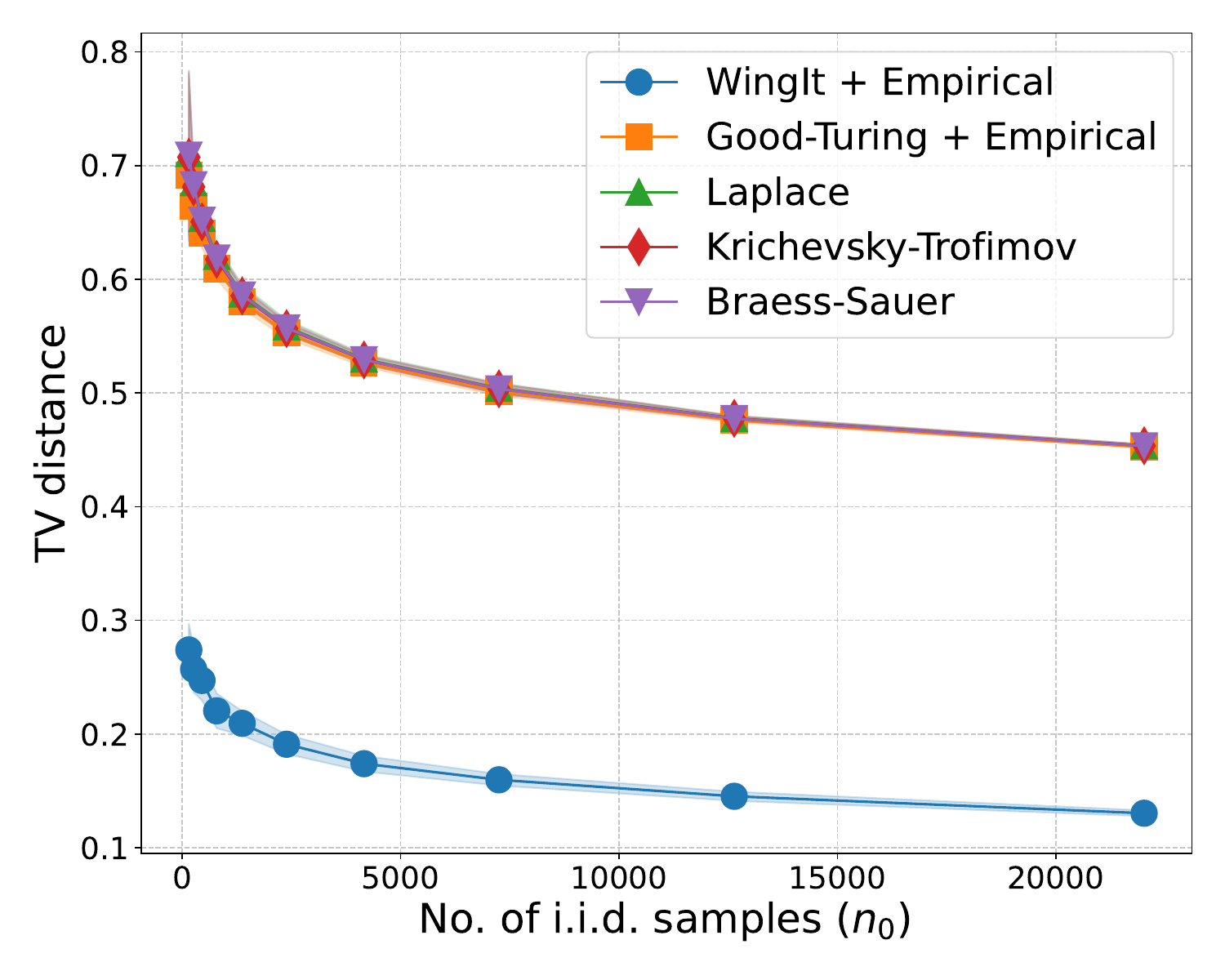}}
            \subfigure[Zipf with parameter $1.5$]{\label{fig:14}\includegraphics[width=7.85cm]{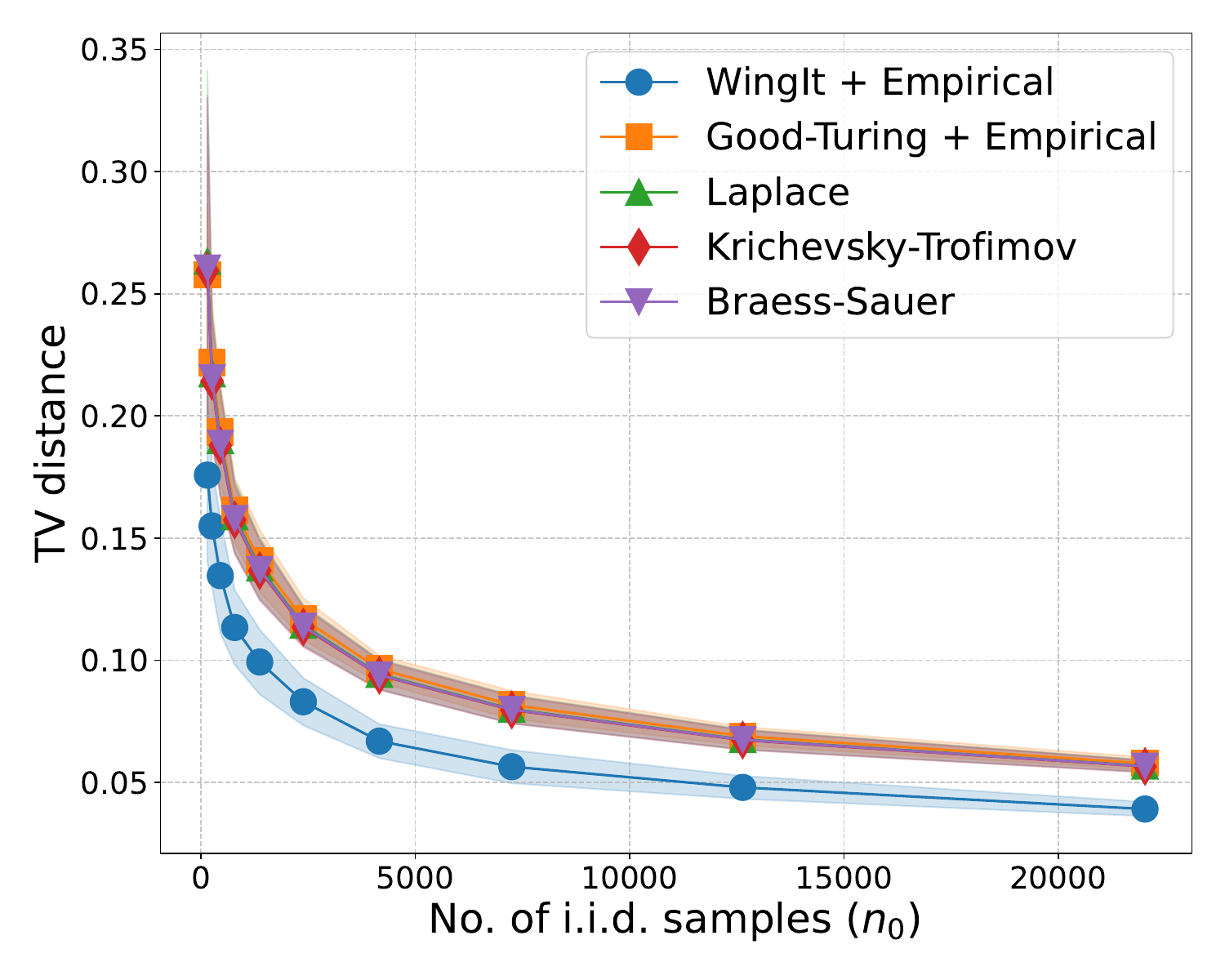}}
        \subfigure[Step]{\label{fig:15}\includegraphics[width=7.85cm]{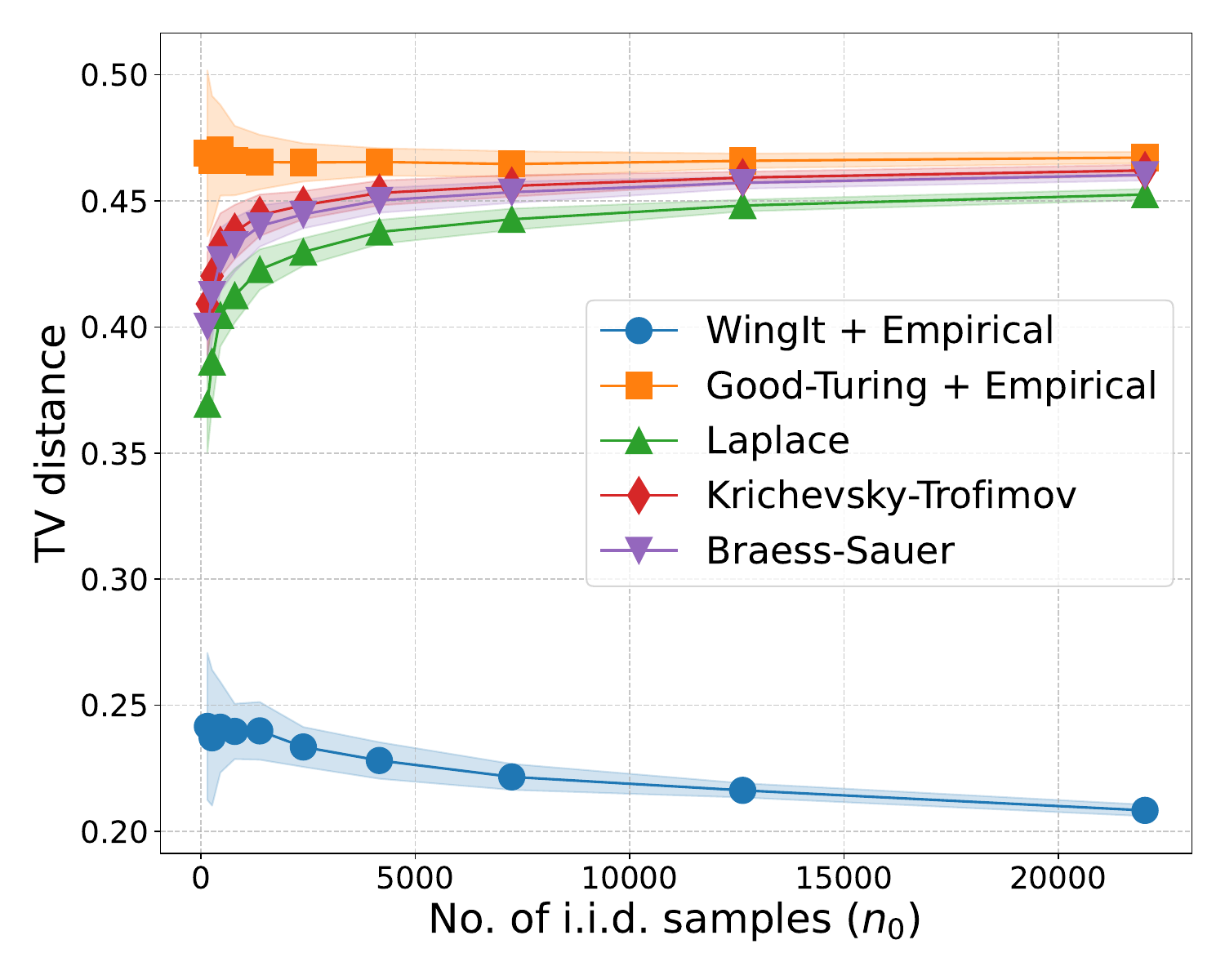}}
		\label{fig:estimation_plots3]}
        \caption{$c=0.2$ corresponding to a mixing sequence with mixing time $\Tmix = \mathcal{O}(n_0^{0.33})$ expected number of samples $\EE[n] = n_0^{1.33}$, averaged over $50$ instances.}
        \label{fig:c=0.33}
	\end{figure}

\subsection{Results}
First, in Figure~\ref{fig:c=0.0} we plot the results for the case $c=0.0$. Since the duplication factor for each sample in this case is simply $1$, this represents the case of executing the estimators on an i.i.d. sequence and is a convenient reference point for evaluation. In this case, the $\Wingit$+Plug-in estimator is exactly equal to the Good--Turing+Plug-in estimator and thus the performance of these estimators overlap; we only include the evaluation of the $\Wingit$+Plug-in estimator to avoid redundancy. This estimator performs well on all distributions and is significantly better than the add-$\beta$ estimators on all base distributions.

Figure~\ref{fig:c=0.2} shows the results for \( c = 0.2 \), where the data is no longer i.i.d.\ The window size for the $\Wingit$+Plug-in estimator is set to \( \tau = 10 \times n_0^{c} \), which is on the order of the mixing time of the sequence. In this setting, the $\Wingit$+Plug-in estimator continues to outperform all competing methods across all stationary distributions considered. It exhibits consistent estimation behavior, with the TV distance from the true count probabilities decreasing steadily as the number of samples increase. In contrast, the add-$\beta$ estimators fail to achieve consistent estimation across the different base distributions.

Finally, Figure~\ref{fig:c=0.33} presents the results for \( c = 0.33 \), where the data is even more dependent. Once again, we observe that the $\Wingit$+Plug-in estimator significantly outperforms the other estimators, highlighting its effectiveness in dealing with temporal dependencies in the data.

\section{Proofs} \label{sec:proofs-main}

We now provide proofs of all our results. We begin by recalling and defining notation. Let $n_0 = n/\tau$, and assume for readability that $\tau$ divides $n$. Recall the index sets $\{ \Dset_i \}_{i = 1}^n$ and $\{ \Iset_i \}_{i = 1}^n$ from Eq.~\eqref{eq:index-sets}. For $j \in [n_0]$ and $\ell = 0, 1, \ldots, 2\T - 1$, define the sets $\Dsetone_{j, \ell} := \Dset_{2\T j - \ell}$ and $\Isetone_{j, \ell} := \Iset_{2\T j - \ell}$.
Mnemonically, the set $\Dsetone_{j, \ell}$ represents a ``block" around the $2 \tau j - \ell$-th index of the sequence, and $\Isetone_{j, \ell}$ represents the corresponding ``hole", which contains all indices but with the block $\Dsetone_{j, \ell}$ removed. We drop floor and ceiling notation for readability, and assume divisibility of particular integers when convenient.

\subsection{Proof of Lemma \ref{lemma:oracle_inequality}}\label{sec:pf_oracle_inequality}

Recall that the estimator $\qhat$ is defined as $\qhat_x = \Mhat_{N_x} / \varphi_{N_x}$ for all $x \in \Xspace$. Also define the following \emph{oracle-aided} natural estimator, which has knowledge of the stationary distribution $\pi$:
    \begin{align}
    \label{eq:best_estimator}
    q^{\pi}_x = \frac{M^{\pi}_{N_x}}{\varphi_{N_x}} \quad \text{ for all } x \in \Xspace.
\end{align}
The estimator $q^{\pi}$ calculates each count probability mass of the stationary distribution $\pi$ and distributes it equally among all symbols having the same count in $X^n$.
Denote the optimal natural estimator with respect to the TV error as 
\[
q^{*} = \arg \inf_{q \in \mathcal{Q}^{\nat}} \TV(\pi,q),
\]
and for each $\zeta = 0, 1, \ldots, n$, define $M^*_{\zeta} = \sum_{x: N_x = \zeta} q^*_x$.
Let $M^{*} \in \Delta(\{0, 1, \ldots, n\})$ denote the corresponding vector of count masses. Using the triangle inequality for TV distance, we have
\begin{align}
\label{eq:intermediate_TV}
    \TV\left(\pi,\qhat\right) \leq \TV(\pi,q^{\pi}) + \TV(\qhat,q^{\pi}).
\end{align}
We also have the following:
\begin{align*}
    \TV(\pi,q^{\pi}) &\leq \TV(\pi,q^{*}) + \TV(q^{*},q^{\pi}) \overset{\1}{=} \TV(\pi,q^{*}) + \TV (M^{*}(X^n),M^\pi(X^n)),
\end{align*}
where step $\1$ follows from Lemma \ref{lemma:TV_count_masses} in the appendix, since both $q^*$ and $q^{\pi}$ are natural estimators.
Finally, we have the sequence of inequalities
\begin{align*}
\TV\left(M^\pi(X^n),M^{*}(X^n)\right) &= \frac{1}{2} \sum_{u=0}^n |M^\pi_u(X^n) - M^*_u(X^n)| \\ 
&= \frac{1}{2}\sum_{u=0}^n\left|\sum_{x \in \Xspace} \ind{N_x=u}(\pi_x - q^*_x)\right| \\ 
&\leq \frac{1}{2}\sum_{u=0}^n\sum_{x \in \Xspace} \ind{N_x=u}\left|\pi_x - q^*_x\right| \\ 
&= \frac{1}{2}\sum_{x \in \Xspace}\left(\left|\pi_x - q^*_x\right| \sum_{u=0}^n\ind{N_x=u}\right) \\ 
&\overset{\1}{=} \frac{1}{2}\sum_{x \in \Xspace}\left|\pi_x - q^*_x\right| = \TV(\pi,q^*),
\end{align*}
where step $\1$ follows because $\sum_{t=0}^n\ind{N_x=t} =1$. Putting together the pieces, we obtain
 \begin{align*}
    \TV(\pi,\qhat) \leq 2\cdot \TV(\pi,q^{*}) + \TV(\qhat,q^{\pi}) \overset{\1}{=} 2\cdot \TV(\pi,q^{*}) +\TV(\Mhat,M^{\pi}),  
\end{align*}
where step $\1$ follows since both $\qhat$ and $q^{\pi}$ are natural estimators and so we can again appeal to Lemma \ref{lemma:TV_count_masses}.
This completes the proof of Lemma \ref{lemma:oracle_inequality}.
\qed

\subsection{Proof of Theorem \ref{theorem:competitive_regret_theorem}}
We prove Theorem \ref{theorem:competitive_regret_theorem} from Propositions~\ref{lemma:empirical_TV_bound} and~\ref{lemma:wingit_TV_bound}, which in turn are proved in Sections \ref{sec:pf_empiirical_TV_bound} and~\ref{sec:pf_wingit_TV_bound}. 
Denote $\Mhat^{\unnorm}$ as the \emph{unnormalized} estimator, i.e. $\Mhat^{\unnorm} := \Mhat(\tau;\overline{\zeta}) \cdot \nu$, with 
\begin{align*}
    \Mhat^{\unnorm}_{\zeta} & =
    \begin{cases}
    \Mhat_{\some,\zeta}(\tau) \quad
 &\text{ if }\zeta \leq \overline{\zeta}\\
\Mhat_{\emp,\zeta}  &\text{ if } \zeta > \overline{\zeta},
\end{cases}
\end{align*}

We have
\begin{align*}
    \TV(M^{\pi}(X^n), \Mhat(\tau;\overline{\zeta})) &\overset{\1}{\leq} \| M^{\pi} - \Mhat^{\unnorm} \|_1 \\
    &= \sum_{\zeta = 0}^{\overline{\zeta}} \left| M^{\pi}_{\zeta} - \Mhat_{\some,\zeta}(\tau) \right| + \sum_{\zeta = \overline{\zeta}+1}^n \left| M^{\pi}_{\zeta} - \Mhat_{\emp,\zeta} \right| ,
\end{align*} 
where step $\1$ follows from Lemma \ref{lemma:normalized} in the appendix. 
The following algebraic inequalities will be useful in our bounds going forward:
\begin{align} \label{eq:cauchy_application}
\sum_{\zeta=\overline{\zeta}+1}^n \varphi_{\zeta} \overset{\text{(a)}}{\leq} \frac{n}{\overline{\zeta}+1} \quad \text{ and } \quad 
    \sum_{\zeta = \overline{\zeta}+1}^{n} \varphi_{\zeta} \sqrt{\zeta} \overset{\text{(b)}} \leq \frac{n}{\sqrt{\overline{\zeta}+1}}.
\end{align}
We defer a proof of Eq.~\eqref{eq:cauchy_application} to the end of the section. We also assume that $n \gtrsim \tau^3 \log^3 (Cn)$ for the rest of this section; the theorem is trivially true otherwise since the TV is always bounded above by $1$.

\paragraph{Bounding $\Wingit$ error:}
Applying Proposition \ref{lemma:wingit_TV_bound}, we have
\begin{align*}
    &\sum_{\zeta = 0}^{\overline{\zeta}} \EE\left[\left|M_{\zeta}^{\pi}(X^n)-\Mhat_{\some,\zeta}\right|\right] \\  
    & \lesssim \sum_{\zeta = 0}^{\overline{\zeta}} \sqrt{\frac{\T}{n}}\left(\sqrt{\EE[M^{\pi}_{\zeta}]}+\sqrt{\zeta \log(2\T) \EE\left[ \ind{N_Y=\zeta} \right]}
     + \sqrt{\sum_{u=1}^{4\T-2}\frac{ (\zeta + u) }{u}\EE\left[M^{\pi}_{\zeta+u}\right] }\right)+ \sum_{\zeta = 0}^{\overline{\zeta}} \frac{(\zeta+1) \T}{n} \\
     & \leq
    \sqrt{\frac{\T}{n}(\overline{\zeta}+1) }+\sqrt{\frac{\T}{n}}\underbrace{\sum_{\zeta = 0}^{\overline{\zeta}} \sqrt{\zeta \log(2\T) \EE\left[ \ind{N_Y=\zeta} \right]}}_{E_{1}}
     + \sqrt{\frac{\T}{n}}\underbrace{\sum_{\zeta = 0}^{\overline{\zeta}} \sqrt{ \sum_{u=1}^{4\T-2}\frac{ (\zeta + u) }{u} \EE\left[M^{\pi}_{\zeta+u}\right]}}_{E_{2}}+ \frac{(\overline{\zeta}+1)^2 \T}{n}.
\end{align*}
The first term in the above inequality is lower order compared to $E_1$ and $E_2$. Thus we now focus on providing the bounds for $E_{1}$ and $E_{2}$ separately. We have
\begin{align*}
    E_{1} = \sum_{\zeta = 0}^{\overline{\zeta}} \sqrt{\zeta \log(2\T) \EE\left[ \ind{N_Y=\zeta} \right]} \overset{\1}{\leq} (\overline{\zeta}+1)\sqrt{ \log(2\T)},
\end{align*}
where step $\1$ follows because the Cauchy-Schwarz inequality and the fact $\sum_{\zeta = 0}^{\overline{\zeta}} \EE\left[ \ind{N_Y=\zeta} \right] \leq 1$. 
For the term $E_{2}$, we have
\begin{align*}
    E_{2} &= \sum_{\zeta = 0}^{\overline{\zeta}} \sqrt{ \sum_{u=1}^{4\T-2}\frac{ (\zeta + u) }{u}\EE\left[M^{\pi}_{\zeta+u}\right] } \\
    &=\sum_{\zeta = 0}^{\overline{\zeta}} \sqrt{\EE\left[ \sum_{u=1}^{4\T-2}\ind{N_Y = \zeta + u}\frac{ (\zeta + u) }{u} \right]} \\
    &\overset{\1}{\leq}\sqrt{\sum_{\zeta=0}^{\overline{\zeta}}1} \cdot \sqrt{\sum_{\zeta = 0}^{\overline{\zeta}} \EE\left[ \sum_{u=1}^{4\T-2}\ind{N_Y = \zeta + u}\frac{ (\zeta + u) }{u} \right]} 
    \end{align*}
where step $\1$ follows from the Cauchy--Schwarz inequality. Continuing, we have
\begin{align*}
    E_2 &\leq \sqrt{\overline{\zeta}+1} \cdot \sqrt{ \EE\left[ \sum_{u=1}^{4\T-2}(1/u)\sum_{\zeta = 0}^{\overline{\zeta}}\ind{N_Y = \zeta + u} (\zeta + u)  \right]} \\
    &\leq \sqrt{\overline{\zeta}+1} \cdot \sqrt{ \EE\left[ \sum_{u=1}^{4\T-2} \frac{(\overline{\zeta} + u)}{u}\right]} \leq \sqrt{(\overline{\zeta}+1)(\overline{\zeta}+4\T-2) \log(4\T)}.
\end{align*}
Putting together the pieces, we have
\begin{align*}
\sum_{\zeta = 0}^{\overline{\zeta}} \EE\left[\left|M_{\zeta}^{\pi}(X^n)-\Mhat_{\some,\zeta}\right|\right] \lesssim (\overline{\zeta}+1)\sqrt{\frac{\T\log(2\T)}{n}}+\sqrt{\frac{\T}{n}}\sqrt{(\overline{\zeta}+1)(\overline{\zeta}+4\T-2) \log(4\T)} + \frac{(\overline{\zeta}+1)^2 \T}{n}.
\end{align*}
In the statement of Theorem~\ref{theorem:competitive_regret_theorem}, we set the value $\overline{\zeta}+1 = n^{1/3}$. Since we have assumed that $n^{1/3} \gtrsim \tau \log (Cn)$, we can write $(\overline{\zeta}+4\T-2) \lesssim \overline{\zeta}+1$ and substituting above yields
\begin{align*}
    \sum_{\zeta = 0}^{\overline{\zeta}} \EE\left[\left|M_{\zeta}^{\pi}(X^n)-\Mhat_{\some,\zeta}\right|\right] \lesssim (\overline{\zeta}+1)\sqrt{\frac{\T\log(4\T)}{n}}+ \frac{(\overline{\zeta}+1)^2 \T}{n}.
\end{align*}

\paragraph{Bounding plug-in error:}
For bounding the error of the plug-in estimator one thing to note is that we use this estimator only for $\zeta \geq \overline{\zeta}+1 = n^{1/3}$. Applying the statement of Proposition \ref{lemma:empirical_TV_bound}, the following holds for each $\zeta \geq \overline{\zeta}+1$ provided that $\tau \geq \tmix\left(\frac{\epsilon}{n^2}\right)$: \begin{align*}
     \left| M^{\pi}_{\zeta}(X^n) - \Mhat_{\emp,\zeta} \right| &\lesssim   \frac{\sqrt{\T\zeta}\cdot\varphi_{\zeta}(X^n)}{n}\sqrt{\log\left(\frac{Cn}{\delta}\right)}.
\end{align*}
probability at least $1-\delta-3\epsilon$.

Taking the union bound of the above inequality for all $\zeta \geq \overline{\zeta}+1$ we have
\begin{align*}
     \left| M^{\pi}_{\zeta}(X^n) - \Mhat_{\emp,\zeta} \right| &\lesssim   \frac{\sqrt{\T\zeta}\cdot\varphi_{\zeta}(X^n)}{n}\sqrt{\log\left(\frac{Cn}{\delta}\right)}.
\end{align*}
probability at least $1-n\delta-3n \epsilon$. Substituting the above inequality into the error for the plug-in estimator we have
\begin{align*}
    \sum_{\zeta = \overline{\zeta}+1}^n \left| M^{\pi}_{\zeta}(X^n) - \Mhat_{\emp,\zeta} \right| &\lesssim \sum_{\zeta = \overline{\zeta}+1}^n  \frac{\sqrt{\T\zeta}\cdot\varphi_{\zeta}(X^n)}{n}\sqrt{\log\left(\frac{Cn}{\delta}\right)} \\
    &\overset{\1}{\leq}  \sqrt{\frac{\T}{\overline{\zeta}+1}} \sqrt{\log\left(\frac{Cn}{\delta}\right)},
\end{align*}
probability at least $1-n\delta-3n\epsilon$. Above, step $\1$ follows from Eq.~\eqref{eq:cauchy_application}. Substituting $\epsilon' = 3n\epsilon$ and $\delta' = n\delta$, we have
\begin{align*}
    \sum_{\zeta = \overline{\zeta}+1}^n \left| M^{\pi}_{\zeta}(X^n) - \Mhat_{\emp,\zeta} \right| &\lesssim   \sqrt{\frac{\T}{\overline{\zeta}+1}} \sqrt{\log\left(\frac{Cn^2}{\delta'}\right)},
\end{align*}
with probability at least $1-\delta'-\epsilon'$.
Choosing $\epsilon' = \frac{1}{n^2}$ and then integrating the tail bound we obtain 
\begin{align}
    \sum_{\zeta = \overline{\zeta}+1}^n \EE\left| M^{\pi}_{\zeta}(X^n) - \Mhat_{\emp,\zeta} \right| \lesssim \sqrt{\frac{\T}{\overline{\zeta}+1}}\sqrt{\log\left(C n^2\right)}.
\end{align}
Recall that Proposition \ref{lemma:empirical_TV_bound} requires $\T\geq \tmix(\epsilon/n^2)$. We set $\epsilon = \epsilon'/3n$ and finally substitute $\epsilon' = 1/n^2$. Thus we require the condition $\T \geq \tmix(1/n^5)$.

\paragraph{Combining the pieces:}  
Combining the two bounds, we have
\begin{align*}
    \EE \left[\TV(M^{\pi}(X^n), \Mhat(\tau)) \right] &\lesssim (\overline{\zeta}+1)\sqrt{\frac{\T\log(4\T)}{n}} + \frac{(\overline{\zeta}+1)^2 \T}{n}+\sqrt{\frac{\T}{\overline{\zeta}+1}}\sqrt{\log\left(Cn^2\right)}.
\end{align*}
Note that we have set $\overline{\zeta}+1 = n^{1/3}$ and also that $n \geq C' \tau^3 \log^3 (Cn)$ for a large enough constant $C'$. This allows us to simplify the above expression and obtain
\begin{align*}
    \EE \left[\TV(M^{\pi}(X^n), \Mhat(\tau)) \right] &\lesssim \frac{ \sqrt{\tau \log (Cn)}}{n^{1/6}},
\end{align*}
as claimed.
We conclude by proving claim~\eqref{eq:cauchy_application}.

\paragraph{Proof of claim~\eqref{eq:cauchy_application}:} First, note that $n = \sum_{\zeta=0}^{n}\zeta \varphi_{\zeta}$. Part (a) of the inequality then follows by noting that
\begin{align} \label{eq:first-ineq} 
    n \geq \sum_{\zeta=\overline{\zeta}+1}^{n}\zeta \varphi_{\zeta} \geq (\overline{\zeta}+1)\sum_{\zeta=\overline{\zeta}+1}^{n}\varphi_{\zeta}. 
\end{align}
To prove part (b), we use the Cauchy--Schwarz inequality to obtain
\begin{align*}
    \sum_{ \overline{\zeta}+1}^{n} \varphi_{\zeta} \sqrt{\zeta} = \sum_{\zeta = \overline{\zeta}+1}^{n} \sqrt{\varphi_{\zeta}} \sqrt{\zeta\varphi_{\zeta}} \leq \sqrt{\sum_{\zeta = \overline{\zeta}+1}^{n}\varphi_{\zeta}} \sqrt{\sum_{\zeta = \overline{\zeta}+1}^{n}\zeta \varphi_{\zeta}}. 
\end{align*}
Using Eq.~\eqref{eq:first-ineq} then yields
\begin{align*}
    \sum_{\zeta = \overline{\zeta}+1}^{n} \varphi_{\zeta} \sqrt{\zeta} \leq \sqrt{\frac{n}{\overline{\zeta}+1}} \cdot \sqrt{n}  = \frac{n}{\sqrt{\overline{\zeta}+1}},
\end{align*}
which completes the proof of the claim and therefore the proof of Theorem~\ref{theorem:competitive_regret_theorem}.
\qed

\subsection{Proof of Lemma \ref{lemma:alpha_mixing_EB}}\label{sec:lemma_alpha_mixing_EB}

We structure the proof as follows. Both parts of the lemma rely on a certain blocking argument, which we present first. Then, we prove each part of the lemma in turn. For the rest of this proof, suppose $\tau$ is some arbitrary integer, and assume for simplicity in exposition that $\tau$ divides $n$. In part (b), we will choose a particular value of $\tau$.

\paragraph{Blocking argument:} 

We start by dividing the entire sequence $U^n$ into $n_0 = n / \T$ intervals, each of length $\T$. 
%Assume $\T$ divides $n$, so that  $n = n_0 \cdot \tau$. 
Now we write the quantity of interest $\left| \EE[U_1] - \frac{1}{n}\sum_{i=1}^{n} U_i \right|$ as a sum over disjoint even and odd-numbered blocks of the sequence. Let us denote the set of indices inside each of these intervals as $S_k$ for $k \in [n_0]$. Note, by definition, that $|S_k| = \T$ for all $k \in [n_0]$.
We define the following portions of the original sequence $U^n$:
\begin{align*}
    U_{\odd} &\defn \bigoplus_{k \text{ odd}} \bigoplus_{j \in S_k} U_j, \quad \text{ and } \\ 
    U_{\even} &\defn \bigoplus_{k \text{ even}} \bigoplus_{j \in S_k} U_j.
\end{align*}
Using the triangle inequality, we have
\begin{align}
    \nonumber
    \left| \EE[U_1] - \frac{1}{n}\sum_{i=1}^{n} U_i \right| &= \left| \frac{\EE[U_1]}{2} - \frac{1}{n}\sum_{\text{k odd}} \sum_{j\in S_k} U_j + \frac{\EE[U_1]}{2} -\frac{1}{n}\sum_{\text{k even}} \sum_{j\in S_k} U_j \right| \\ \label{eq:intermediate_triangle}
    &\leq \left|\frac{\EE[U_1]}{2} - \frac{1}{n}\sum_{\text{k odd}} \sum_{j\in S_k} U_j\right| + \left|\frac{\EE[U_1]}{2} - \frac{1}{n}\sum_{\text{k even}} \sum_{j\in S_k} U_j\right|.
\end{align}
We now construct a new auxiliary stochastic process as follows. We choose the tuple $(Y_1, \ldots, Y_{\tau})$ independently of everything else, and according to the law of $(U_1, \ldots, U_{\tau})$. Similarly, choose the blocks $(Y_{\tau + 1}, \ldots, Y_{2 \tau}), \ldots, (Y_{n - \tau + 1}, \ldots, Y_{n})$ IID from the same distribution.
In words, at the start of each interval $S_k$, we restart the process from a state sampled from the stationary distribution $\pi$.
Let us denote this sequence as
\begin{align*}
Y^n \defn (Y_j)_{j \in [n]}.
\end{align*}
We again decompose the sequence $Y^n$ over odd and even blocks similar to $U^n$.
\begin{align*}
    Y_{\odd} &= \bigoplus_{k \text{ odd}} \bigoplus_{j \in S_k} Y_j, \text{ and } \\ Y_{\even} &= \bigoplus_{k \text{ even}} \bigoplus_{j \in S_k} Y_j.
\end{align*}
Note that the stochastic processes $U_{\odd}$ and $Y_{\odd}$ are each of length $n/2$ and correspond directly to the definitions of $W_j$ and $W_j'$ in the statement of Lemma~\ref{lemma:restart}.
Formally, we can apply Lemma \ref{lemma:restart}, with the function $f:\Xspace^{n/2} \to \{0,1\}$ defined as below for an arbitrary sequence $V^{n/2} := (V_1,\ldots,V_{n/2})$:
\[
f(V^{n/2}) \defn \ind{\left| \EE[V_1] - \frac{2}{n}\sum_{i=1}^{n/2} V_i \right| >t}.
\]
Applying the lemma yields
\begin{align*}
\notag 
    \EE\left[  f(U_{\odd}) \right]\leq \EE\left[  f(Y_{\odd}) \right] +  \frac{n_0}{2}\cm\rho^{\T}
\end{align*}
for any $t \in \mathbb{R}$.
Substituting the definition of $f$ in the above inequality yields
\begin{align}\label{eq:auxiliary-to-original}
    \mathbb{P}\left( \left|\frac{\EE[U_1]}{2} - \frac{1}{n}\sum_{\text{k odd}} \sum_{j\in S_k} U_j\right| >\frac{t}{2}\right) \leq \mathbb{P}\left( \left|\frac{\EE[Y_1]}{2} - \frac{1}{n}\sum_{\text{k odd}} \sum_{j\in S_k} Y_j\right| >\frac{t}{2}\right) + \frac{n_0}{2}\cm\rho^{\T}.
\end{align}
An identical statement can be shown for the even subsequence.

Substituting Eq.~\eqref{eq:auxiliary-to-original} for odd and even sequences into Eq. \eqref{eq:intermediate_triangle} and applying the union bound, we obtain
\begin{align}
\nonumber
    \mathbb{P}\left( \left|\EE[U_1] - \frac{1}{n}\sum_{j=1}^n U_j\right| >t\right) &\leq \mathbb{P}\left( \left|\frac{\EE[Y_1]}{2} - \frac{1}{n}\sum_{\text{k odd}} \sum_{j\in S_k} Y_j\right| >\frac{t}{2}\right) \\ \label{eq:decorrelated}
    &\qquad \qquad + \mathbb{P}\left( \left|\frac{\EE[Y_1]}{2} - \frac{1}{n}\sum_{\text{k even}} \sum_{j\in S_k} Y_j\right| >\frac{t}{2}\right) + n_0\cm\rho^{\T}.
\end{align}
We will now use Eq.~\eqref{eq:decorrelated} to prove each part of the lemma separately.

\paragraph{Proof of part (a):} For notational convenience, define the random variables $J_k \defn \frac{1}{\T} \sum_{j\in S_k}Y_j$ for $k \in [n_0]$. We collect some relevant properties of these random variables. First, since $Y_j \in [0,B]$, we have $J_k \in [0,B]$ for all $k \in [n_0]$. By construction of $Y^n$, the random variables $\{J_k\}_{k=1}^{n_0}$ are also independent and identically distributed. 
Moreover, because the process $Y^n$ is initialized at the stationary distribution $\pi$, we have $\EE[Y_1] = \EE[J_1]$.
Finally, recall that we defined $v^2 := \frac{1}{\T^2} \var\left(\sum_{i = 1}^\T U_i\right)$ in the statement of the lemma. It is easy to verify that
\[
v^2 := \frac{1}{\T^2} \var\left(\sum_{i = 1}^\T U_i\right)=\frac{1}{\T^2} \var\left(\sum_{i = 1}^\T Y_i\right) = \var(J_1) = \frac{2}{n_0} \sum_{\text{k odd}} \var(J_k).
\]
We now bound the two terms on the RHS of Eq.~\eqref{eq:decorrelated}. Let us present the argument for the first (i.e. odd) term without loss of generality. We have
\begin{align*}
\left|\EE[Y_1] - \frac{2}{n}\sum_{\text{k odd}} \sum_{j\in S_k} Y_j\right| =\left|\EE[J_1] - \frac{2}{n_0}\sum_{\text{k odd}}  J_k\right| \overset{\1}{\leq} \sqrt{\frac{4\T v^2 \log(1/\delta)}{n}} + \frac{4B \T \log(1/\delta)}{3n},
\end{align*}
where inequality $\1$ holds
with probability at least $1-2\delta$ by the Bernstein inequality. 
Repeating the argument for the even subsequence and combining the bounds with Eq.~\eqref{eq:decorrelated}, we obtain
\begin{align*}
    \mathbb{P}\left( \left|\EE[U_1] - \frac{1}{n}\sum_{j=1}^n U_j\right| > \sqrt{\frac{4\T v^2 \log(1/\delta)}{n}} + \frac{4B \T \log(1/\delta)}{3n}\right) \leq 4\delta + n_0 \cm \rho^{\T}.
\end{align*}
For any $\T \geq \tmix(\epsilon/n)$, we have $n_0 \cm \rho^{\T} \leq \epsilon$, so that
\begin{align*}
    \left|\EE[U_1] - \frac{1}{n}\sum_{j=1}^n U_j\right| \leq \sqrt{\frac{4\T v^2 \log(1/\delta)}{n}} + \frac{4B \T \log(1/\delta)}{3n},
\end{align*}
with probability at least $1-4\delta - \epsilon$.
This completes the proof of part (a).

\paragraph{Proof of part (b):} 
For the proof of this part only, we abuse notation slightly and let $\tau = \tmix(\epsilon/n)$ for convenience, still letting $n_0 = n/\T$. It is convenient to define the shorthand $L_k \defn \frac{1}{\tau B}\sum_{j\in S_k} Y_j = \frac{J_k}{B}$ for each $k \in [n_0]$. Since $Y_j \in [0,B]$, we have $L_k \in [0,1]$ for all $k \in [n_0]$. As before, the collection of random variables $\{L_k\}_{k = 1}^{n_0}$ is independent and identically distributed. 

Using the empirical Bernstein bound from \citet[Theorem 4]{maurer2009empirical}, we have
\begin{align}
    \label{eq:iid_EB}
    \frac{1}{B}\left|\EE[Y_1] - \frac{2}{n}\sum_{\text{k odd}} \sum_{j\in S_k} Y_j\right| = \left| \EE[L_1] - \frac{2}{n_0}\sum_{\text{k odd}}L_k \right| \overset{\1}{\leq} \sqrt{\frac{4\Var(L)\log(2/\delta)}{n_0}} + \frac{14 \log(2/\delta)}{3(n_0-2)},
\end{align}
where step $\1$ holds with probability at least $1-2\delta$. 
Recall that we defined the spread as
\begin{align*}
    \Var(L) &= \frac{1}{(n_0/2)(n_0/2-1)}\sum_{\substack{p,q \text{ odd} \\ p<q}} (L_p - L_q)^2,
\end{align*}
since the number of odd blocks is equal to $n_0/2$.
We can upper bound $\Var(L)$ as below:
\begin{align*}
    \Var(L) &= \frac{4}{\T^2 n_0(n_0-2)}\sum_{\substack{p,q \text{ odd} \\ p<q}}\left(\sum_{i\in S_p}\frac{Y_i}{B} - \sum_{j\in S_q}\frac{Y_j}{B}\right)^2 \\ 
    &\overset{\1}{\leq} \frac{4}{\T n_0(n_0-2)} \sum_{\substack{p,q \text{ odd} \\ p<q}} \left(\sum_{i \in S_p}\frac{Y_i}{B}+\sum_{j \in S_q}\frac{Y_j}{B}\right) \\ 
    &\leq \frac{4}{\T B (n_0-2)}\sum_{\text{k odd}}\sum_{j \in S_k} Y_j.
\end{align*}
where step $\1$ follows from Young's inequality and the fact that $Y_i/B \in [0,1]$.
Substituting expressions for $\Var(L)$, $L_k$ and $n_0$ into Eq. \eqref{eq:iid_EB}, we obtain
\begin{align}
    \label{eq:Y_EB}
    \left| \frac{\EE[Y_1]}{2} - \frac{1}{n} \sum_{\text{k odd}} \sum_{j\in S_k} Y_j\right| \leq  \frac{2 \sqrt{ \T B \log(2/\delta) \sum_{\text{k odd}}\sum_{j \in S_k} Y_j}}{n-2\T} + \frac{7 \T B \log(2/\delta)}{3(n-2\T)}
\end{align}
with probability at least $1-2\delta$. Executing the same argument for the even subsequence, combining the pieces with the inequality $\sqrt{\sum_{\text{k odd}} \sum_{j\in S_k} Y_j} + \sqrt{\sum_{\text{k even}} \sum_{j\in S_k} Y_j} \leq \sqrt{2 \sum_{k = 1}^{n_0} \sum_{j\in S_k} Y_j}$ and $2\sqrt{2}<3$, and noting that $\sum_{k = 1}^{n_0} \sum_{j \in S_k} Y_j = \sum_{j=1}^n Y_j$, we have
\begin{align}\label{eq:X_Y_EB}
    \left| \EE[U_1] - \frac{1}{n}\sum_{j=1}^{n} U_j \right| \leq \frac{3\sqrt{ \T B \log(2/\delta) \sum_{j=1}^n Y_j}}{n-2\T} + \frac{14 \T B \log(2/\delta)}{3(n-2\T)}
\end{align}
with probability at least $1-4\delta - n_0 \cm \rho^{\T}$. 

Next, we claim for $\tau = \tmix(\epsilon/n)$ that if $\sum_{j=1}^n U_j \geq 36 \T B \log(2/\delta)$ and $n \geq 24\T$, then the inequality
\begin{align}\label{eq:Y_U_relation}
    \sqrt{\sum_{j=1}^n Y_j} 
    \leq 11 \sqrt{\sum_{j=1}^n U_j} + 14\sqrt{\T B \log(2/\delta)},
\end{align}
holds with probability at least $1-6\delta - n_0 \cm \rho^{\T}$. We defer the proof of Eq. \eqref{eq:Y_U_relation} to the end of this section.

Substituting Eq. \eqref{eq:Y_U_relation} into Eq. \eqref{eq:X_Y_EB}, we obtain
\begin{align*}
    \left|\EE[U_1] - \frac{1}{n}\sum_{j=1}^n U_j\right| \leq  \frac{33\sqrt{ \T  B \log(2/\delta) \sum_{j=1}^n U
    _j}}{n-2\T} + \frac{ 140\T B \log(2/\delta)}{3(n-2\T)},
\end{align*}
with probability at least $1-10\delta - 2 n_0 \cm \rho^{\T}$. 
Now setting $\tau = \tmix(\epsilon/n)$, recall that in the statement of the lemma, we assumed the conditions $\sum_{j=1}^n U_j \geq 36B \tau \log(2/\delta)$ and $n\geq 24\T$.
From this, we obtain
\begin{align*}
    \left|\EE[U_1] - \frac{1}{n}\sum_{j=1}^n U_j\right| \leq  \frac{734\sqrt{ \T  B \log(2/\delta) \sum_{j=1}^{n} U_j}}{9n},
\end{align*}
with probability at least $1-10\delta - 2 n_0 \cm \rho^{\T}$.
For $\T = \tmix(\epsilon/n)$, we have $n_0 \cm \rho^{\T} \leq \epsilon$, and this completes the proof.

\noindent \underline{Proof of claim~\eqref{eq:Y_U_relation}:} 
Executing the empirical Bernstein argument for the auxiliary sequence\footnote{Note that in this case, we do not need to split over odd and even subsequences; just a direct application of the empirical Bernstein bound suffices, and we obtain a slightly better constant.} $Y^n$ we have
\begin{align*}
    \left| \EE[Y_1] - \frac{1}{n} \sum_{k = 1}^{n_0} \sum_{j\in S_k} Y_j\right| \leq  \frac{2 \sqrt{ \T B\log(2/\delta)  \sum_{k=1}^{n_0}\sum_{j \in S_k} Y_j}}{n-2\T} + \frac{7 \T B \log(2/\delta)}{3(n-2\T)},
\end{align*}
with probability at least $1-2\delta$.
Simplifying notation, we have
\begin{align}
    \label{eq:Y_n_EB}
    \left| \EE[Y_1] - \frac{1}{n} \sum_{j=1}^n Y_j\right| \leq  \frac{2 \sqrt{ \T B\log(2/\delta)  \sum_{j=1}^n Y_j}}{n-2\T} + \frac{7 \T B \log(2/\delta)}{3(n-2\T)},
\end{align}

Combing the Eq. \eqref{eq:Y_n_EB} and Eq. \eqref{eq:X_Y_EB} with the triangle inequality, we have
\begin{align}
    \label{eq:diff_EB}
    \left| \frac{1}{n} \sum_{j=1}^n Y_j - \frac{1}{n} \sum_{j=1}^n U_j \right| \leq \frac{5\sqrt{ \T B \log(2/\delta)  \sum_{j=1}^n Y_j}}{n-2\T} + \frac{7 \T B \log(2/\delta)}{n-2\T},
\end{align}
with probability at least $1-6\delta-n_0 \cm \rho^{\T}$. 

At the same time, using the identity
$\sqrt{a} \leq \sqrt{b} + \frac{|a - b|}{\sqrt{b}}$ for all $a,b \geq 0$, we obtain
\begin{align}
    \frac{\sqrt{\sum_{j=1}^n Y_j}}{n-2\T}
    &\leq \frac{\sqrt{\sum_{j=1}^n U_j}}{n-2\T} + \frac{1}{n-2\T} \frac{\left|\sum_{j=1}^n Y_j-\sum_{j=1}^n U_j\right|}{\sqrt{\sum_{j=1}^n U_j}}.  \label{eq:relate_XY_EB}
\end{align}
Substituting Eq. \eqref{eq:relate_XY_EB} into Eq. \eqref{eq:diff_EB}, we obtain
\begin{align*}
    \left| \frac{1}{n} \sum_{j=1}^n Y_j - \frac{1}{n} \sum_{j=1}^n U_j \right| &\leq \frac{ 5 \sqrt{ \T B \log(2/\delta)  \sum_{j=1}^n U_j}}{n-2\T} +\frac{7 \T B \log(2/\delta)}{n-2\T} \\ 
    &\qquad \qquad + \frac{5 \sqrt{\T B \log(2/\delta)}}{n-2\T} \frac{\left|\sum_{j=1}^n Y_j-\sum_{j=1}^n U_j\right|}{\sqrt{\sum_{j=1}^n U_j}},
\end{align*}
with probability at least $1-6\delta-n_0 \cm \rho^{\T}$. 

For $\tau = \tmix(\epsilon/n)$, we may now use the conditions $\sum_{j=1}^n U_j \geq 36 \T B \log(2/\delta)$ and $n \geq 24\T$ to obtain
\begin{align*}
    \left|  \sum_{j=1}^n Y_j -  \sum_{j=1}^n U_j \right| &\leq 60\sqrt{ \T B \log(2/\delta)  \sum_{j=1}^n U_j} +84\T B \log(2/\delta),
\end{align*}
with probability at least $1-6\delta-n_0 \cm \rho^{\T}$. Substituting the above display into Eq. \eqref{eq:relate_XY_EB}, we obtain
\begin{align*}
    \sqrt{\sum_{j=1}^n Y_j} 
    &\leq 11\sqrt{\sum_{j=1}^n U_j} + 14 \sqrt{\T B\log(2/\delta)},
\end{align*}
with probability at least $1-6\delta - n_0 \cm \rho^{\T}$. 
This completes the proof of the claim and therefore of the lemma.
\qed

\subsection{Proof of Proposition \ref{lemma:empirical_TV_bound}}
\label{sec:pf_empiirical_TV_bound}
 Recall that Proposition~\ref{lemma:empirical_TV_bound} bounds the $\ell_1$ risk between the plug-in estimator $\Mhat_{\emp,\zeta}$ and the true count probability mass $M^{\pi}_{\zeta}$ for sufficiently large values of count $\zeta$. Also recall the shorthand notation $\tau_0 = \tmix(\epsilon/n^2)$. We denote by $\Xspace_{\zeta}$ the set of all symbols appearing exactly $\zeta$ times in $X^n$ and by $\Xspace_{\geq \zeta}$ the set of all symbols appearing at least $\zeta$ times in $X^n$. Clearly, we have $\Xspace_{\zeta} \subseteq \Xspace_{\geq \zeta}$. Using this fact and the triangle inequality, we have
\begin{align}
    \nonumber
     \left| M^{\pi}_{\zeta} -  \Mhat_{\emp,\zeta}\right| &= \left| M^{\pi}_{\zeta} - \frac{\zeta \varphi_{\zeta}}{n} \right| \\ \nonumber
     &= \left|\sum_{x\in \Xspace}\ind{N_x(X^n) = \zeta}\left(\pi_x - \frac{\zeta}{n}\right)\right| \\ \nonumber
     &\leq \sum_{x\in \Xspace}\ind{N_x(X^n) = \zeta}\left|\pi_x - \frac{N_x(X^n)}{n}\right| \\ 
     \label{eq:intermediate_PI_zeta}
     &\leq \sum_{x\in \Xspace_{\geq \zeta}}\ind{N_x(X^n) = \zeta}\left|\pi_x - \frac{N_x(X^n)}{n}\right|.
\end{align}
At this juncture, one possible approach is to prove a high-probability bound on 
the deviation term $\left|\pi_x - \frac{N_x(X^n)}{n}\right|$ for each $x \in \Xspace$, and then take a union bound over the state space. However, this would induce a dependence on $|\Xspace|$ and the resulting bound would no longer be universal in $n$. Consequently, we reduce the sum over $\Xspace_{\geq \zeta}$ to the sum over a deterministic set of size at most $\mathcal{O}(n)$ using the following lemma.

\begin{lemma}\label{lemma:large_frequency}
Recall that $\tau_0 = \tmix(\epsilon/n^2)$. Then
\begin{align*}
    \mathbb{P}\left( \exists x \in \Xspace:  N_x(X^n) \geq 1+\sqrt{(4 + 8\T_0)\log(n/\delta)} + \frac{4\T_0 \log(n/\delta)}{3} \text{ and } \pi_x \leq 1/n \right) \leq \delta + \epsilon.
\end{align*}
\end{lemma}
\noindent Lemma~\ref{lemma:large_frequency} is proved in Section~\ref{sec:pf_lemma_large_frequency}.
Let us also define the deterministic set
$
    \Yspace \defn \Bigg\{x \in \Xspace : \pi_x \geq \frac{1}{n}\Bigg\}.
$
By Lemma~\ref{lemma:large_frequency}, we have that for any $\zeta \geq 1+\sqrt{(4 + 8\T_0)\log(n/\delta)} + \frac{4\T_0 \log(n/\delta)}{3}$, the inclusion
\begin{align} \label{eq:X-subset-Y}
    \Xspace_{\geq \zeta} \subseteq \Yspace  \text{ holds with probability at least } 1 - \delta - \epsilon.
\end{align}
Next, we observe that $N_x(X^n)$ can be written as follows:
\begin{align*}
    N_x(X^n) = \sum_{j=1}^n \ind{X_j =x},
\end{align*}
which takes the form of a sum of indicator random variables over the sequence $X^n$.
Since $\Yspace$ contains at most $n$ (nonrandom) elements, we may apply the union bound over these elements in conjunction with Lemma~\ref{lemma:alpha_mixing_EB}.
This implies that the following event occurs \emph{uniformly} over all $x \in \Yspace$ with probability at least $1 - 10\delta - 2\epsilon$ (note that we defined $\T_0 = \tmix(\epsilon/n^2)$):
\begin{align} \label{eq:deviation-on-Y}
\left|\pi_x - \frac{N_x(X^n)}{n}\right| \lesssim 
\begin{cases}
    \frac{\sqrt{\T_0 N_x(X^n)\log(2n/\delta)}}{n} \quad &\text{ if } N_x(X^n) \geq 36 \T_0  \log(2n/\delta) \\
    \sqrt{\frac{4\pi_x(1+2\T_0)\log(n/\delta)}{n}} + \frac{4\T_0\log(n/\delta)}{3n} &\text{ otherwise.}
\end{cases}
\end{align}
We make particular use of the first case in the proof --- the second case is stated for completeness. Working now on the intersection of the high-probability events in~\eqref{eq:X-subset-Y} and~\eqref{eq:deviation-on-Y}, which in turn occurs with probability at least $1 - 11 \delta - 3\epsilon$, and considering
\[
\zeta \geq \max \left\{ 36 \T_0  \log(2n/\delta), 1+\sqrt{(4 + 8\T_0)\log(n/\delta)} + \frac{4\T_0 \log(n/\delta)}{3} \right\},
\]
we have
\begin{align*}
    &\sum_{x\in \Xspace_{\geq \zeta}}\ind{N_x(X^n) = \zeta}\left|\pi_x - \frac{N_x(X^n)}{n}\right| \\
    &\quad= \sum_{x\in \Xspace_{\geq \zeta} \cap \Yspace} \ind{N_x(X^n) = \zeta}\left|\pi_x - \frac{N_x(X^n)}{n}\right| + \sum_{x\in \Xspace_{\geq \zeta} \cap \Yspace^c} \ind{N_x(X^n) = \zeta}\left|\pi_x - \frac{N_x(X^n)}{n}\right| \\
    &\quad\overset{\1}{=} \sum_{x\in \Xspace_{\geq \zeta} \cap \Yspace} \ind{N_x(X^n) = \zeta}\left|\pi_x - \frac{N_x(X^n)}{n}\right| \\ 
    &\quad\overset{\2}{\lesssim} \sum_{x\in \Xspace_{\geq \zeta}} \ind{N_x(X^n) =\zeta}\frac{\sqrt{\T_0 N_x(X^n)\log(2n/\delta)}}{n} \\
    &\quad\leq \sum_{x\in \Xspace} \ind{N_x(X^n) =\zeta}\frac{\sqrt{\T_0 N_x(X^n)\log(2n/\delta)}}{n},
\end{align*}
where step $\1$ follows by Eq.~\eqref{eq:X-subset-Y} and step $\2$ by Eq.~\eqref{eq:deviation-on-Y}.

Now, using the fact that
$\sum_{x\in \Xspace}\ind{N_x(X^n)=\zeta}\sqrt{N_x(X^n)} = \varphi_{\zeta}(X^n) \sqrt{\zeta}$,
we obtain
\begin{align}
    \label{eq:empirical_TV_bound_detailed}
    \left| M^{\pi}_{\zeta} -  \Mhat_{\emp,\zeta}\right| &\lesssim \frac{\sqrt{\T_0}}{n}\sqrt{\log\left(\frac{2n}{\delta}\right)} \varphi_{\zeta}(X^n) \sqrt{\zeta}
\end{align}
on the same event occurring with probability at least $1 -11\delta - 3 \epsilon$. Adjusting constants in the definitions of $\delta$ and $\epsilon$ and substituting the value of $\T_0$ yields the statement of the proposition.
\qed

\subsubsection{Proof of Lemma \ref{lemma:large_frequency}}\label{sec:pf_lemma_large_frequency}

 In order to prove this lemma, we start by carefully partitioning the state space. We number the elements in $\Xspace$ as $1, \ldots, |\Xspace|$ and assume without loss of generality that they are ordered such that $\pi_1 \leq \pi_2 \leq \cdots \leq \pi_{|\Xspace|}$. Let us define $j^* +1 \defn \min\{ i: \pi_i > 1/n\}$. Also define the cumulative probabilities $c_i = \sum_{m = 1}^i \pi_i$, and let $i_k := \min\{i: c_i \geq k/n\}$ denote a sequence of $n + 1$ indices with $i_0 \defn 0$.
Now, partition the first $j^*$ indices into disjoint sets
$T_1, \ldots, T_{L}$ given by
\begin{align}
T_{\ell} = \{i_{\ell - 1} + 1, \ldots, i_{\ell}\}, 
\end{align}
where the set is considered to be empty if $i_{\ell} = i_{\ell - 1}$, and $L$ is defined such that $i_L = j^*$.
We overload notation and denote by $\pi_{T_{\ell}}$ the total stationary mass of all symbols appearing in the partition $T_{\ell}$, i.e. $\pi_{T_{\ell}} := \sum_{x \in T_{\ell}} \pi_x$.
Note by the definition of $\{i_{\ell}\}_{\ell=1}^L$ and $\{T_{\ell}\}_{\ell=1}^L$ that each block $T_{\ell}$ contains elements whose stationary probabilities sum up to at most $1/n$. In other words, we have $\pi_{T_{\ell}} \leq 1/n$ for all $\ell \in [L]$. Thus, we have
$
\sum_{\ell=1}^{L} \pi_{T_{\ell}}\leq \frac{L}{n} \leq 1,
$
which implies that $L \leq n$.

Next, we define
\begin{align*}
    N_{T_{\ell}}(X^n) \defn \sum_{x \in T_{\ell}} N_x(X^n).
\end{align*}
We also set $\zeta_0 := 1+\sqrt{(4 + 8\T_0)\log(n/\delta)} + \frac{4\T_0 \log(n/\delta)}{3}$, where we defined $\T_0 := \tmix(\epsilon/n^2)$ in the statement of Lemma~\ref{lemma:large_frequency}. By construction of the partitions $\{T_{\ell}\}_{\ell=1}^L$, we obtain
\begin{align} 
\mathbb{P}\left( \exists x: N_x(X^n) \geq \zeta_0 \text{ and } \pi_x \leq 1/n\right) &= \mathbb{P}\left( \exists x \in \cup_{\ell = 1}^{L} T_\ell: N_x(X^n) \geq \zeta_0\right) \nonumber \\
&\leq \mathbb{P} \left(\exists \ell \in [L]: N_{T_{\ell}}(X^n) \geq \zeta_0\right) \nonumber \\
&\leq \sum_{\ell = 1}^{L} \mathbb{P}\left(N_{T_{\ell}}(X^n) \geq \zeta_0\right).\label{eq:partition-bound}
\end{align}
It remains to characterize each of the terms $\mathbb{P}\left(N_{T_{\ell}}(X^n) \geq \zeta_0\right)$, for which we will use the statement of Bernstein's inequality from Lemma~\ref{lemma:alpha_mixing_EB}(a).
Consider the random variable $$B_i = \ind{X_i \in T_{\ell}} - \mathbb{P}\left(X_i \in T_{\ell}\right).$$ We note that $|B_i| \leq 1$ and $\EE[B_i] = 0$. Applying the Bernstein inequality from Lemma \ref{lemma:alpha_mixing_EB}(a) --- in its tail bound form --- to the process $(B_1,\ldots,B_n)$ with $\T_0 = \tmix(\epsilon/n^2)$ yields
\begin{align}
    \label{eq:Bernstein_B}
    \mathbb{P}\left( \sum_{i=1}^n B_i \geq t\right) \leq 4 \exp\left(\frac{- t^2}{ 4 n \T_0 v^2 +   \frac{4 \T_0 }{3}t}\right) + \frac{\epsilon}{n}
\end{align}
for any $t > 0$.
Recall from the statement of the lemma that $v^2 \defn \frac{1}{\T_0^2}\var \left( \sum_{j=1}^{\T_0} B_j  \right) = \frac{1}{\T_0^2}\sum_{j=1}^{\T_0} \var \left(  B_j  \right)+ \frac{2}{\T_0^2}\sum_{j>i}^{\T_0}\cov(B_i,B_j)$.
We have
\begin{align*}
    \var(B_i) &= \mathbb{P}(X_i \in T_{\ell}) -\mathbb{P}(X_i \in T_{\ell})^2 \leq 1/n \text{ and } \\
    \sum_{j>i}^{\T_0}\cov(B_i,B_j) &\overset{\1}{\leq} \sum_{j>i}^{\T_0} \sqrt{\var(B_i)\var(B_j)} \leq \frac{\T^2_0}{n} ,
\end{align*}
where step $\1$ follows from the Cauchy-Schwarz inequality.
This directly yields
\[
 v^2 \leq \frac{1}{n \T_0} + \frac{2}{n} \implies 4n \T_0 v^2 \leq 4(1+2\T_0).
\]
Continuing from Eq. \eqref{eq:Bernstein_B} and substituting the definition of the random variables $\{B_i\}_{i=1}^n$, we have
\[
\mathbb{P} \left( \sum_{i=1}^n \ind{X_i \in T_\ell} \geq  t + \sum_{i=1}^n \mathbb{P}\left(X_i \in T_{\ell}\right) \right) \leq \exp\left(\frac{-t^2}{4 + 8\T_0 + \frac{4 \T_0 }{3}t }\right)+\frac{\epsilon}{n}.
\]
Next we note that, because the sets $\{T_{\ell}\}_{\ell=1}^L$ are disjoint, we have $\sum_{i=1}^n \mathbb{P}\{X_i \in T_{\ell}\} \leq 1$.
From this, we obtain
\[
\mathbb{P} \left( N_{T_{\ell}}(X^n) \geq  t + 1 \right) \leq \exp\left(\frac{-t^2}{4 + 8\T_0 + \frac{4 \T_0 }{3}t }\right)+\frac{\epsilon}{n}.
\]
Substituting $t = \sqrt{(4 + 8\T_0)\log(n/\delta)} + \frac{4\T_0 \log(n/\delta)}{3}$ in the above yields
\begin{align}\label{eq:tail-event-partition}
\mathbb{P} \left( N_{T_{\ell}}(X^n) \geq  1+\sqrt{(4 + 8\T_0)\log(n/\delta)} + \frac{4\T_0 \log(n/\delta)}{3} \right) = \mathbb{P}\left(N_{T_{\ell}}(X^n) \geq \zeta_0 \right) \leq \frac{\delta+\epsilon}{n}.
\end{align}
Combining Ineqs.~\eqref{eq:tail-event-partition} and~\eqref{eq:partition-bound}, we obtain
\begin{align*}
\mathbb{P}\left( \exists x: N_x(X^n) \geq \zeta_0 \text{ and } \pi_x \leq 1/n\right) \leq L \left(\frac{\delta + \epsilon}{n}\right) \overset{\1}{\leq} \delta + \epsilon,
\end{align*}
where step $\1$ follows because we established that $L \leq n$.
This completes the proof of the lemma. \qed

\subsection{Proof of Proposition \ref{lemma:wingit_TV_bound}}\label{sec:pf_wingit_TV_bound}

In this section, we bound the error of the $\Wingit$ estimator for small frequencies. We first derive bounds on the mean squared error (\MSE) for each $\zeta$, and then use these to obtain bounds on the $\ell_1$ error.
Throughout this proof, note that $\T$ should be thought of as an arbitrary integer that divides $n$, with $n_0 = n/(2\T)$.

As in~\citet{pananjady2024just}, we can relate the $\MSE$ of the $\Wingit$ estimator to a ``skipped" version as below:
\begin{align}
    \label{eq:intermediate_MSE_bound}
    \MSE\left(M_{\zeta}^{\pi}(X^n), \Mhat_{\some, \zeta}(\tau)\right) \leq \frac{1}{2 \tau}\sum_{\ell=0}^{2 \tau-1} \EE \left( M_{\zeta}^{\pi}(X^n) - \Mhat_{\some, \zeta}(\tau;\ell)\right)^2,
\end{align}
where
\begin{align} \label{eq:estimator-skipped-smallcount}
\Mhat_{\some, \zeta}(\tau; \ell) := \frac{1}{n_0} \sum_{j = 1}^{n_0} \Mhat_{\tau, \zeta}^{(2\tau j - \ell)} \text{ for each } \ell = 0, \ldots, 2\tau - 1.
\end{align}
$\Mhat_{\some, \zeta}(\tau; \ell)$ represents the skipped estimator used in our analysis.

We now focus on bounding the term $\EE \left( M_{\zeta}^{\pi}(X^n) - \Mhat_{\some, \zeta}(\tau;0)\right)^2$. An $\MSE$ bound proven for $\ell=0$ can be identically transferred to bound the $\MSE$ of $\Mhat_{\some, \zeta}(\tau;\ell)$ for any $\ell=0,\ldots,2\tau-1$ \citep[Proposition 5]{pananjady2024just}. Thus, without loss of generality, we address the case $\ell = 0$ and analyze the estimator $\Mhat_{\some}(\T; 0)$. Consequently, we use the shorthand $\Dsetone_j := \Dsetone_{j, 0}$ and $\Isetone_{j} := \Isetone_{j, 0}$. For more details and a pictorial representation of the skipped estimators and the corresponding ``dependent" and ``independent" sets, refer to \citet[Section 7.1]{pananjady2024just}.

Adding and subtracting the term $\EE_{ \substack{Y \sim \pi \\ Y \indpt X^n }} \ind{N_{Y}(X_{\Isetone_j}) = \zeta}$ to the quantity of interest given by \mbox{$\frac{1}{2} | \Mhat_{\some,\zeta}(\tau;0) - M^{\pi}_{\zeta}(X^n) |^2$}, we obtain
\begin{align*}
    &\frac{1}{2} \cdot \left|  \frac{1}{n_0} \sum_{j = 1}^{n_0} \ind{ N_{X_{2\tau j}}(X_{\Isetone_j}) =\zeta} - \EE_{ \substack{Y \sim \pi \\ Y \indpt X^n }} \ind{N_{Y}(X^n) = \zeta} \right|^2 \notag \\
&\qquad \qquad \qquad \qquad\leq \underbrace{\left|  \frac{1}{n_0} \sum_{j = 1}^{n_0} \left( \EE_{ \substack{Y \sim \pi \\ Y \indpt X^n }} \ind{N_{Y}(X_{\Isetone_j}) =\zeta} - \EE_{ \substack{Y \sim \pi \\ Y \indpt X^n }} \ind{N_{Y}(X^n) = \zeta} \right) \right|^2}_{T_1} \notag \\
&\qquad \qquad \qquad \qquad \qquad \qquad + \underbrace{\left|  \frac{1}{n_0} \sum_{j = 1}^{n_0} \left( \ind{N_{X_{2\tau j}}(X_{\Isetone_j}) = \zeta} - \EE_{ \substack{Y \sim \pi \\ Y \indpt X^n }} \ind{N_{Y}(X_{\Isetone_j}) = \zeta} \right) \right|^2}_{T_2}.
\end{align*}
The rest of the proof proceeds by bounding $\EE[T_1]$ and $\EE[T_2]$.

\paragraph{Bounding $\EE[T_1]$:} 
In order to bound $T_1$ we define the random variable $P_j = \ind{ N_{Y}(X_{\Isetone_j}) =\zeta} -\ind{N_{Y}(X^n) = \zeta}$. We have
\begin{align*}
    T_1 = \frac{1}{n_0^2} \left| \sum_{j=1}^{n_0} \EE_{\substack{Y }} P_j  \right|^2.
\end{align*}
We obtain $P_j \leq \ind{Y \in X_{\Dsetone_j}}\ind{N_{Y}(X_{\Isetone_j })=\zeta} \leq 1$ in a manner similar to the argument in~\citep[Lemma 12]{pananjady2024just}. We briefly outline the argument here for completeness.
Since $P_j \in \{-1,0,1\}$ and $\ind{Y \in X_{\Dsetone_j}} \cdot \ind{N_{Y}(X_{\Isetone_j })=\zeta} \in \{0,1\}$, the inequality is only non-trivial for the case where $P_j =1$.
In this case, we have $ N_{Y}(X_{\Isetone_j}) =\zeta \text{ and } N_{Y}(X^n) \neq \zeta$, which implies that $Y \in X_{\Dsetone_j}$. Thus, we have $\ind{Y \in X_{\Dsetone_j}}\ind{N_{Y}(X_{\Isetone_j })=\zeta} =1$ and so the above inequality holds.

Since the blocks $\{\Dsetone_j\}_{j=1}^{n_0}$ are non-overlapping, we have 
\[
\bigsqcup_{j' \in [n_0] \setminus j} \Dsetone_{j'} \subset \Isetone_j. 
\]
Now, suppose that for some $j$ we have $\ind{Y \in X_{\Dsetone_j}} \cdot \ind{N_Y(X_{\Isetone_j}) = \zeta}=1$.
This means that $Y$ occurs at least once in $\Dsetone_j$, but its number of occurrences outside of $\Dsetone_j$ is equal to $\zeta$.
Then, we have
\[
\sum_{j' \in [n_0] \setminus j} \ind{Y \in X_{\Dsetone_{j'}}} \cdot \ind{N_Y(X_{\Isetone_{j'}}) = \zeta} \leq \sum_{j' \in [n_0] \setminus j} \ind{Y \in X_{\Dsetone_{j'}}} \leq N_Y(X_{\Isetone_j}) = \zeta.
\]
Putting these pieces together yields $\sum_{j = 1}^{n_0} P_j \leq \zeta + 1$ point-wise. 
Ultimately, this yields
\begin{align}\label{eq:wingit-t1-bound}
    T_1 = \frac{1}{n_0^2} \left(\EE_{\substack{Y}} \sum_{j=1}^{n_0} P_j \right)^2 \leq \left(\frac{\zeta+1}{n_0}\right)^2,
\end{align}
and so $\EE[T_1] \leq \left(\frac{\tau(\zeta+1)}{n}\right)^2$.
This ultimately turns out to be a lower-order term and so this worst-case analysis suffices.

\paragraph{Bounding $\EE[T_2]$:}
The term $T_2$, which captures a certain conditional variance, will be the dominant term in the analysis. First, we have
\begin{align}
    T_2 = \frac{1}{n_0^2} \sum_{j,k=1}^{n_0} Z_j Z_k,
\end{align}
where we define
\[
Z_j := \ind{N_{X_{2j \tau}}(X_{\Isetone_j}) = \zeta} - \EE_{\substack{Y \sim \pi \\ Y \indpt X^n}} \ind{N_{Y}(X_{\Isetone_j}) = \zeta} \text{ for all } j \in [n_0].
\]
We thus have
\begin{align*}
    T_2 = \frac{1}{n_0^2} \sum_{j=1}^{n_0} Z_j^{2} + \frac{1}{n_0^2} \sum_{j=1}^{n_0}\sum_{\substack{k=1 \\ k \neq j}}^{n_0} Z_j Z_k.
\end{align*}
We first bound the cross terms in $T_2$. We define the following random variables for each $j,k \in [n_0]$ with $j \neq k$ and conditioned on a random variable $Y \sim \pi$:
\begin{align}
    \label{eq:Q}
    Q_{j,k}  \defn \ind{N_Y(X_{\Isetone_j \cap \Isetone_k}) = \zeta} -\ind{N_Y(X_{\Isetone_k}) = \zeta}.
\end{align}
Observe that $Q_{j,k} \in [-1,1]$ point-wise and so $|Q_{j,k}| \in [0,1]$ point-wise. The following lemma relates the terms $Z_j Z_k$ and $|Q_{j,k}| + |Q_{k,j}|$ in expectation.

\begin{lemma}\label{lemma:cross_terms_abs}
For each $j \neq k$, we have
\begin{align*}
    \EE[Z_j Z_k] \leq 4 \EE[|Q_{j, k}|] + 4 \EE[|Q_{k, j}|] + 12 \cm\rho^{\tau},
\end{align*}
where $Q_{j,k}$ is defined in Eq.~\eqref{eq:Q}.
\end{lemma}
We take Lemma $\ref{lemma:cross_terms_abs}$ as given for the moment and prove it in Section~\ref{sec:pf-lemma-cross-terms-smallcount}.
Further, auxiliary Lemma~\ref{lemma:abs_Q_bound} yields the following point-wise bound on $|Q_{j,k}|$ for any $Y \in \Xspace$:
\begin{align*}
    |Q_{j,k}| \leq \ind{Y \in X_{\Dsetone_j}} \cdot \left(\ind{N_Y(X_{\Isetone_k}) = \zeta}+ \ind{N_Y(X_{\Isetone_j \cap \Isetone_k}) = \zeta}\right).
\end{align*}
Define the shorthand
\begin{align}\label{eq:Cy_T1T2}
    R_Y \defn  \sum_{j=1}^{n_0} \sum_{\substack{k=1 \\ k\neq j}}^{n_0} \ind{Y \in X_{\Dsetone_j}} \ind{N_Y(X_{\Isetone_k}) = \zeta}+\sum_{j=1}^{n_0} \sum_{\substack{k=1 \\ k\neq j}}^{n_0} \ind{Y \in X_{\Dsetone_j}} \cdot \ind{N_Y(X_{\Isetone_j \cap \Isetone_k}) = \zeta}.
\end{align}
Combining Lemma~\ref{lemma:cross_terms_abs} and Eq.~\eqref{eq:Cy_T1T2}, we obtain the following bound for the cross terms in $T_2$:
\begin{align*}
    &\frac{1}{n_0^2} \sum_{j=1}^{n_0}\sum_{\substack{k=1 \\ k \neq j}}^{n_0} \EE[Z_j Z_k] \overset{(i)}{\leq} \frac{8}{n_0^2} \EE[R_Y] + 12 \cm\rho^{\T}.
\end{align*}

It remains to provide a satisfactory bound for $\EE[R_Y]$, which is the technical crux of our analysis.
The following lemma provides an adaptive, point-wise bound on $R_Y$.
\begin{lemma}\label{lemma:exact_small_count_markovian}
Consider the quantity $R_Y$ defined in Eq.~\eqref{eq:Cy_T1T2}. 
For all $Y \in \Xspace$, we have
\begin{align*}
    R_Y &\leq  \zeta n_0 \cdot \ind{N_Y=\zeta} \log(2\T) + 12 n_0 \sum_{u = 1}^{4\tau - 2} \ind{N_Y = \zeta + u}\frac{ (\zeta + u) }{u}.
\end{align*}
\end{lemma}
We take Lemma \ref{lemma:exact_small_count_markovian} as given for the moment and prove it in Section~\ref{sec:pf-lemma-Cy-smallcount}.

For the diagonal terms in $T_2$, we have
\begin{align}
    \frac{1}{n_0^2} \sum_{j=1}^{n_0} \EE[Z_j^{2}] &=   \frac{1}{n_0^2} \sum_{j = 1}^{n_0}\EE \left[ \left| \ind{N_{X_{2\T j}}(X_{\Isetone_j}) = \zeta} - \EE_{ \substack{Y \sim \pi \\ Y \indpt X^n }} \ind{N_{Y}(X_{\Isetone_j}) = \zeta}  \right|^2\right]. \label{eq:T_2_1_bound}
    \end{align}
Using the property of indicator random variables taking values in $[0,1]$, we obtain
    \begin{align*}
        \frac{1}{n_0^2} \sum_{j=1}^{n_0} \EE[Z_j^{2}] \leq \frac{1}{n_0^2} \sum_{j=1}^{n_0}\EE \left[  \ind{N_{X_{2\T j}}(X_{\Isetone_j}) = \zeta} + \EE_{ \substack{Y \sim \pi \\ Y \indpt X^n }} \ind{N_{Y}(X_{\Isetone_j}) = \zeta} \right] 
    \end{align*}
Next, we apply Lemma~\ref{lemma:bridge} to obtain
    \begin{align*}
        \frac{1}{n_0^2} \sum_{j=1}^{n_0} \EE[Z_j^{2}] &\leq \frac{2}{n_0^2} \sum_{j=1}^{n_0}\EE \left[ \ind{N_{Y}(X_{\Isetone_j}) = \zeta} \right] + \frac{1}{n_0^2} \sum_{j=1}^{n_0} 3 \cm \rho^\T \\ 
        &= \frac{2}{n_0^2} \sum_{j=1}^{n_0}\EE \left[ \ind{N_{Y}(X_{\Isetone_j}) = \zeta} \right] + \frac{3 \cm \rho^\T}{n_0}.
    \end{align*}
The following lemma provides a point-wise, adaptive bound on the term $\sum_{j=1}^{n_0}  \ind{N_{Y}(X_{\Isetone_j}) = \zeta}$.
    \begin{lemma}
        \label{lemma:diagonal_T1}
    We have
        \begin{align*}
            \sum_{j=1}^{n_0}  \ind{N_{Y}(X_{\Isetone_j}) = \zeta}  \leq \sum_{u=1}^{2\T-1} \ind{N_Y = \zeta+u}\cdot \frac{\zeta+u}{u}+ n_0 \cdot \ind{N_Y=\zeta}.
        \end{align*}
    \end{lemma}
We take Lemma \ref{lemma:diagonal_T1} as given for the moment and prove it in Section~\ref{sec:pf-lemma-diagonal-T1}.
Applying the lemma yields
    \begin{align*}
        \frac{1}{n_0^2} \sum_{j=1}^{n_0} \EE[Z_j^{2}] &\leq \frac{2}{n_0^2}\sum_{u=1}^{2\T-1}\EE[\ind{N_Y=\zeta+u}]\frac{\zeta+u}{u} + \frac{2}{n_0} \EE[\ind{N_Y=\zeta}] + \frac{3 \cm \rho^\T}{n_0}.
    \end{align*}
Combining the cross and diagonal terms for $T_2$, we have
\begin{align}
    T_2 &\leq \frac{2}{n_0^2}\sum_{u=1}^{2\T-1}\EE[\ind{N_Y=\zeta+u}]\frac{\zeta+u}{u} + \frac{2}{n_0} \EE[\ind{N_Y=\zeta}] + \frac{3 \cm \rho^\T}{n_0}+ \frac{8}{n_0^2}\EE[R_Y] + 12 \cm \rho^{\T} \nonumber \\ 
    &\overset{\1}{\leq} \frac{2}{n_0^2}\sum_{u=1}^{2\T-1}\EE[M^{\pi}_{\zeta+u}]\frac{\zeta+u}{u} + \frac{2}{n_0} \EE[M^{\pi}_{\zeta}] + \frac{8}{n_0}\zeta \cdot \EE [M^{\pi}_{\zeta}] \cdot \log(2\T) \nonumber \\
    &\qquad + \frac{96}{n_0}\sum_{u = 1}^{4\tau - 2} \EE[M^{\pi}_{\zeta+u}]\frac{ (\zeta + u) }{u} +15 \cm \rho^{\T} \nonumber \\ 
    &\lesssim \frac{1}{n_0} \EE[M^{\pi}_{\zeta}] + \frac{1}{n_0}\EE [M^{\pi}_{\zeta}] \cdot(\zeta) \cdot  \log(2\T)  + \frac{1}{n_0}\sum_{u = 1}^{4\tau - 2} \EE[M^{\pi}_{\zeta+u}]\frac{ (\zeta + u) }{u} + \cm \rho^{\T},\label{eq:wingit-t2-bound}
\end{align}
where step $\1$ uses the relation
\begin{align*}
\EE\left[ \sum_{u=1}^{4\T-2}\ind{N_Y = \zeta + u} \right] &= \EE \left[\sum_{x\in \Xspace}\pi_x \sum_{u=1}^{4\T-2}\ind{N_x = \zeta + u} \right] \\
&= \EE \left[\sum_{u=1}^{4\T-2} \sum_{x\in \Xspace}\pi_x \ind{N_x = \zeta + u} \right] \\
&= \sum_{u=1}^{4\T-2} \EE \left[M^{\pi}_{\zeta+u} \right].
\end{align*}
Combining the bounds for $T_1$ and $T_2$ (Ineqs.~\eqref{eq:wingit-t1-bound} and~\eqref{eq:wingit-t2-bound} respectively), we have 
\begin{align}
     \EE \left[ | \Mhat_{\some,\zeta}(\tau;0) - M^{\pi}_{\zeta}(X^n) |^2 \right] &\lesssim \left(\frac{(\zeta+1)\T}{n}\right)^2+ \frac{1}{n_0} \EE[M^{\pi}_{\zeta}] + \frac{1}{n_0}\EE [M^{\pi}_{\zeta}] \cdot\zeta \cdot  \log(2\T) \nonumber \\ 
     &\qquad \qquad + \frac{1}{n_0}\sum_{u = 1}^{4\tau - 2} \EE[M^{\pi}_{\zeta+u}]\frac{ (\zeta + u) }{u} + \cm \rho^{\T}. \label{eq:combined_t1_t2}
\end{align}
We now relate the $\ell_1$-risk of the $\Wingit$ estimator to its $\MSE$ for each value of $\zeta$.
We have
\begin{align}
\nonumber
&\EE |M_{\zeta}^{\pi}(X^n) - \Mhat_{\some, \zeta}(\tau)| \\ \nonumber 
    &\overset{\1}{\leq} \sqrt{ \EE |M_{\zeta}^{\pi}(X^n) - \Mhat_{\some, \zeta}(\tau)|^2} \notag \\ \nonumber
    &\overset{\2}{\lesssim} \sqrt{\left(\frac{(\zeta+1)\T}{n}\right)^2+ \frac{1}{n_0} \EE[M^{\pi}_{\zeta}] + \frac{1}{n_0}\EE [M^{\pi}_{\zeta}] \cdot\zeta \cdot  \log(2\T) + \frac{1}{n_0}\sum_{u = 1}^{4\tau - 2} \EE[M^{\pi}_{\zeta+u}]\frac{ (\zeta + u) }{u} + \cm \rho^{\T}} \\ \label{eq:intermediate_TV_error}
    &\lesssim \frac{(\zeta+1) \T}{n} + \sqrt{\frac{1}{n_0} \EE[M^{\pi}_{\zeta}]}+ \sqrt{\frac{\EE [M^{\pi}_{\zeta}] \cdot \zeta \cdot  \log(2\T)}{n_0}} + \sqrt{\frac{1}{n_0}\sum_{u = 1}^{4\tau - 2} \EE[M^{\pi}_{\zeta+u}]\frac{ (\zeta + u) }{u}} + \sqrt{\cm \rho^{\T}}.
\end{align}
where step $\1$ follows from Jensen's inequality and step $\2$ follows from Eq.~\eqref{eq:combined_t1_t2}.
Selecting $\tau \geq \tmix(n^{-2})$ yields the statement of Proposition \ref{lemma:wingit_TV_bound}.
\qed

It remains to prove Lemmas~\ref{lemma:cross_terms_abs}, ~\ref{lemma:exact_small_count_markovian} and~\ref{lemma:diagonal_T1}, which we do next.

\subsubsection{Proof of Lemma~\ref{lemma:cross_terms_abs}}\label{sec:pf-lemma-cross-terms-smallcount}
The proof of this lemma is analogous to the proof of Lemma 1 in~\cite{pananjady2024just}.
Define, for convenience,
\begin{align*}
    \overline{Q}_{j,k} := \EE_{\substack{Y}}[Q_{j,k}] = \EE_{\substack{Y}} \left[\ind{N_Y(X_{\Isetone_j \cap \Isetone_k}) = \zeta} - \ind{N_Y(X_{\Isetone_k}) = \zeta}\right].
\end{align*}
Note that by auxiliary Lemma~\ref{lemma:Q_bound}, we have 
\[
\overline{Q}_{j,k} \leq \EE_Y \left[\ind{Y \in X_{\Isetone_k \setminus \Isetone_j}} \cdot \ind{N_Y(X_{\Isetone_j \cap \Isetone_k}) = \zeta}\right] \leq 1.
\]
We then have the decomposition
\small
\begin{align*}
    & Z_j Z_k  \\
    &= \left( \ind{ N_{X_{2\T j}}(X_{\Isetone_{j}}) = \zeta } - \EE_Y [ \ind{ N_Y(X_{\Isetone_{j} \cap \Isetone_k}) = \zeta } ] + \overline{Q}_{k, j} \right) \cdot \\
    &\qquad\qquad\qquad\qquad\left( \ind{ N_{X_{2\T k}}(X_{\Isetone_{k}}) = \zeta } - \EE_{Y'} [ \ind{ N_{Y'} (X_{\Isetone_{j} \cap \Isetone_k}) = \zeta } ] + \overline{Q}_{j, k} \right) \\
    &\leq \left( \ind{ N_{X_{2\T j}}(X_{\Isetone_{j}}) = \zeta } - \EE_Y [ \ind{ N_Y(X_{\Isetone_{j} \cap \Isetone_k}) = \zeta } ]  \right) \cdot 
    \left( \ind{ N_{X_{2\T k}}(X_{\Isetone_{k}}) = \zeta } - \EE_{Y'} [ \ind{ N_{Y'} (X_{\Isetone_{j} \cap \Isetone_k}) = \zeta } ]  \right)  \\
    &\qquad \qquad \qquad \qquad \qquad  + \overline{Q}_{j, k} + \overline{Q}_{k, j} + \overline{Q}_{j, k} \cdot \overline{Q}_{k, j} \\
    &\leq \underbrace{ \left( \ind{ N_{X_{2\T j}}(X_{\Isetone_{j}}) = \zeta } - \EE_Y [ \ind{ N_Y(X_{\Isetone_{j} \cap \Isetone_k}) = \zeta } ]  \right) \cdot 
    \left( \ind{ N_{X_{2\T k}}(X_{\Isetone_{k}}) = \zeta } - \EE_{Y'} [ \ind{ N_{Y'} (X_{\Isetone_{j} \cap \Isetone_k}) = \zeta } ]  \right) }_{U_{j, k}} \\
    &\qquad \qquad \qquad \qquad \qquad  + |\overline{Q}_{j, k}| + |\overline{Q}_{k, j}| + |\overline{Q}_{j, k}| \cdot |\overline{Q}_{k, j}| \\
    &\overset{\1}{\leq} U'_{j, k} + \frac{3}{2} ( |\overline{Q}_{j, k}| + |\overline{Q}_{k, j}| ),\
\end{align*}
\normalsize
where $Y' \sim \pi$ is an independent copy of $Y$.
Above, step $\1$ follows due to the following algebraic inequalities: Since each $|\overline{Q}_{j,k}| \in [0,1]$, we have $|\overline{Q}_{j, k}| \cdot |\overline{Q}_{k, j}| \leq \sqrt{|\overline{Q}_{j, k}| \cdot |\overline{Q}_{k, j}|} \leq \frac{1}{2} ( |\overline{Q}_{j, k}| + |\overline{Q}_{k, j}|)$.

Taking the expectation with respect to the sequence $X^n$, we have
\begin{align*}
    \EE_{X^n}[Z_j Z_k] \leq \EE_{X^n}[U_{j, k}] + \frac{3}{2} \EE_{X^n}[( |\overline{Q}_{j, k}| + |\overline{Q}_{k, j}| )].
\end{align*}
Now we establish that $\EE_{X^n}[U_{j, k}] \leq \frac{5}{2}\EE_{X^n}[|\overline{Q}_{j,k}|] + \frac{5}{2}\EE_{X^n}[|\overline{Q}_{k, j}|] + 12\cm\rho^{\T}$. We have the further decomposition
\small
\begin{align}
&\EE_{X^n}[U_{j, k}] \notag \\
&= \underbrace{\EE [\ind{N_{X_{2\T j}}(X_{\Isetone_{j}}) = \zeta} \cdot \ind{N_{X_{2\T k}} (X_{\Isetone_{k}}) = \zeta}]}_{U_1}
    - \underbrace{\EE_{X^n} \left[\ind{N_{X_{2\T j}}(X_{\Isetone_{j}}) = \zeta} \cdot \EE_{Y'} [\ind{N_{Y'}(X_{\Isetone_j \cap \Isetone_{k}}) = \zeta}]\right]}_{U_2} \notag \\
    &\; - \underbrace{\EE_{X^n}\left[\ind{N_{X_{2\T k}}(X_{\Isetone_{k}}) = \zeta} \cdot \EE_{Y} [\ind{N_Y(X_{\Isetone_{j} \cap \Isetone_k}) = \zeta}]\right]}_{U_3} \notag \\
    &+ \underbrace{\EE_{X^n} \left[ \EE_{Y} [\ind{N_Y(X_{\Isetone_{j} \cap \Isetone_k }) = \zeta}] \cdot \EE_{Y'} [\ind{N_{Y'}(X_{\Isetone_j \cap \Isetone_{k}}) = \zeta}] \right]}_{U_4}. \label{eq:four-terms-prime}
\end{align}
\normalsize
We now bound each of the above terms in turn.
First, we bound $U_1$ as
\begin{align}
U_1 &= \EE [\ind{N_{X_{2\T j}}(X_{\Isetone_{j}}) = \zeta} \cdot \ind{N_{X_{2\T k}} (X_{\Isetone_{k}}) = \zeta}] \notag \\
&\overset{\1}{\leq} \EE [\ind{N_{Y}(X_{\Isetone_{j}}) = \zeta} \cdot \ind{N_{Y'} (X_{\Isetone_{k}}) = \zeta}] + 6\cm \rho^{\tau} \notag \\ 
&\overset{}{=} \EE \left[\EE_Y\left[\ind{N_{Y}(X_{\Isetone_j \cap \Isetone_k}) = \zeta }-Q_{k,j}\right] \cdot \EE_{Y'}\left[\ind{N_{Y'}(X_{\Isetone_j \cap \Isetone_k}) = \zeta }-Q_{j,k}\right]\right] + 6\cm \rho^{\tau} \notag \\
&\overset{\2}{\leq} \EE \left[\left(\EE_Y\left[\ind{N_{Y}(X_{\Isetone_j \cap \Isetone_k}) = \zeta }\right]-\overline{Q}_{k,j}\right) \cdot \left(\EE_{Y'}\left[\ind{N_{Y'}(X_{\Isetone_j \cap \Isetone_k}) = \zeta }\right]-\overline{Q}_{j,k}\right)\right] + 6\cm \rho^{\tau} \notag \\
&\overset{}{=} \EE \left[ \EE_{Y'} \left[\ind{N_{Y'}(X_{\Isetone_j \cap \Isetone_k}) = \zeta } \right] \cdot \EE_Y \left[\ind{N_Y(X_{\Isetone_j \cap \Isetone_k}) = \zeta }\right] \right]+ \EE\left[ \overline{Q}_{j,k}\cdot \overline{Q}_{j,k} \right] \notag \\ 
&\qquad -\EE \left[\overline{Q}_{j,k} \EE_{Y} [\ind{N_{Y}(X_{\Isetone_j \cap \Isetone_k}) = \zeta } ] \right]-\EE\left[\overline{Q}_{k,j} \EE_{Y'} [\ind{N_{Y'}(X_{\Isetone_j \cap \Isetone_k}) = \zeta } ] \right]+ 6\cm\rho^{\tau}, \label{eq:U1prime-bound}
\end{align}
where step $\1$ uses Lemma~\ref{lemma:two-bridges} from the appendix (applied with $i_1 = 2\T\min\{j,k\}, i_2 = 2\T\max\{j,k\}$, and noting that $i_2 - i_1 \geq 2\T$ as $j \neq k$), and step $\2$ follows because $Y$ and $Y'$ are independent of everything else.

Proceeding to the next term, 
note that $U_2$ may be viewed as the expectation over $X^n$ of 
\[
f(X_{2j\T}; X_{\Isetone_j}) := \ind{N_{X_{2\T j}}(X_{\Isetone_{j}}) = \zeta} \cdot \EE_Y [\ind{N_Y(X_{\Isetone_{k} \cap \Isetone_{j}}) = \zeta}], 
\]
which is bounded in the range $[0,1]$.
We may now apply Lemma~\ref{lemma:bridge} from the appendix (for the choice $i = 2\T j$) to obtain
$
|\EE [f(X_{2\T j}; X_{\Isetone_j})] - \EE [f(Y'; X_{\Isetone_j})] | \leq 3\cm\rho^{\T}$.
Thus, we have
\begin{align} \label{eq:U2prime-bound}
U_2 \notag &\geq \EE_{X^n} \left[\EE_{Y'}[\ind{N_{Y'}(X_{\Isetone_{j}}) = \zeta}] \cdot \EE_Y [\ind{N_Y(X_{\Isetone_{j} \cap \Isetone_k}) = \zeta}]\right] - 3\cm\rho^{\T} \notag \\
&\geq \EE_{X^n} \left[\EE_{Y'}[\ind{N_{Y'}(X_{\Isetone_{j} \cap \Isetone_k}) = \zeta}] \cdot \EE_Y [\ind{N_Y(X_{\Isetone_{j} \cap \Isetone_k}) = \zeta}]\right] - \EE_{X^n}[\overline{Q}_{k,j}] - 3\cm\rho^{\T}.
\end{align}
By an identical argument to the above, we have
\begin{align} \label{eq:U3prime-bound}
U_3 \geq \EE_{X^n} \left[\EE_{Y'}[\ind{N_{Y'}(X_{\Isetone_{j} \cap \Isetone_k}) = \zeta}] \cdot \EE_Y [\ind{N_Y(X_{\Isetone_{j} \cap \Isetone_k}) = \zeta}]\right] - \EE_{X^n}[\overline{Q}_{j,k}] - 3\cm\rho^{\T}.
\end{align}
Putting Ineqs.~\eqref{eq:U1prime-bound},~\eqref{eq:U2prime-bound} and~\eqref{eq:U3prime-bound} together with the definition of $U'_4$ and performing the requisite cancellations, we have
\begin{align*}
\EE[U'_{j, k}] &\leq (U'_1 - U'_2)+ (- U'_3 + U'_4) \notag \\ &\leq \EE_{X^n}\left[ \overline{Q}_{j,k}\cdot \overline{Q}_{j,k} \right]-\EE_{X^n}\left[\overline{Q}_{j,k} \EE_{Y} [\ind{N_{Y}(X_{\Isetone_j \cap \Isetone_k}) = \zeta } ] \right]-\EE_{X^n}\left[\overline{Q}_{k,j} \EE_{Y'} [\ind{N_{Y'}(X_{\Isetone_j \cap \Isetone_k}) = \zeta } ] \right] \\
&\qquad +\EE_{X^n}[\overline{Q}_{j,k}] + \EE_{X^n}[\overline{Q}_{k,j}] + 12 \cm \rho^{\T}  \\
&\overset{\1}{\leq} \EE_{X^n}\left[ |\overline{Q}_{j,k}|\cdot |\overline{Q}_{j,k}| \right]+\EE_{X^n}\left[|\overline{Q}_{j,k}| \right]+\EE_{X^n}\left[|\overline{Q}_{k,j}| \right] +\EE_{X^n}[|\overline{Q}_{j,k}|]+ \EE_{X^n}[|\overline{Q}_{k,j}|] + 12 \cm \rho^{\T} \\
&\overset{\2}{\leq} \frac{5}{2}\EE_{X^n}[|\overline{Q}_{j,k}|] + \frac{5}{2}\EE_{X^n}[|\overline{Q}_{j,k}|] + 12 \cm \rho^{\T},
\end{align*}
where step $\1$ follows by taking the absolute value and using the triangle inequality and step $\2$ follows due to the following algebraic inequalities: Since each $|\overline{Q}_{j,k}| \in [0,1]$, we have $|\overline{Q}_{j, k}| \cdot |\overline{Q}_{k, j}| \leq \sqrt{|\overline{Q}_{j, k}| \cdot |\overline{Q}_{k, j}|} \leq \frac{1}{2} ( |\overline{Q}_{j, k}| + |\overline{Q}_{k, j}|)$.
Substituting the bound for $\EE[U_{j, k}]$ completes the proof of Lemma \ref{lemma:cross_terms_abs}.
\qed

\subsubsection{Proof of Lemma \ref{lemma:exact_small_count_markovian}}\label{sec:pf-lemma-Cy-smallcount}

From Eq.~\eqref{eq:Cy_T1T2}, we have
$$R_Y =  \underbrace{\sum_{j=1}^{n_0} \sum_{\substack{k=1 \\ k\neq j}}^{n_0} \ind{Y \in X_{\Dsetone_j}} \cdot \ind{N_Y(X_{\Isetone_k}) = \zeta}}_{A_1}+\underbrace{\sum_{j=1}^{n_0} \sum_{\substack{k=1 \\ k\neq j}}^{n_0} \ind{Y \in X_{\Dsetone_j}} \cdot \ind{N_Y(X_{\Isetone_j \cap \Isetone_k}) = \zeta}}_{A_2}.$$
First, we show an upper bound on $A_1$. 
It is convenient to introduce variables $u$ and $v$ that count the number of occurrences of $Y$ in $\Dsetone_j$ and, $\Dsetone_k$ respectively.
Then, we have
\begin{align*}
    A_1 &= \sum_{j=1}^{n_0} \sum_{\substack{k=1 \\ k\neq j}}^{n_0} \ind{Y \in X_{\Dsetone_j}} \cdot \ind{N_Y(X_{\Isetone_k}) = \zeta} \\
    &\overset{}{=} \sum_{j=1}^{n_0} \sum_{\substack{k=1 \\ k\neq j}}^{n_0} \sum_{u=1}^{2\T -1} \ind{N_Y(X_{\Dsetone_j}) = u} \cdot \ind{N_Y(X_{\Isetone_j \cap \Isetone_k})=\zeta-u} \\
    &=\sum_{j=1}^{n_0} \sum_{\substack{k=1 \\ k\neq j}}^{n_0} \sum_{u=1}^{2\T -1} \sum_{v=0}^{2\T -1}\ind{N_Y(X_{\Dsetone_j}) = u} \cdot \ind{N_Y(X_{\Dsetone_k}) = v} \cdot \ind{N_Y(X_{\Isetone_j \cap \Isetone_k})=\zeta-u}.
\end{align*}
Next, we note that $N_Y =N_Y(X_{\Dsetone_j})+ N_Y(X_{\Dsetone_k})+ N_Y(X_{\Isetone_j \cap \Isetone_k})$. Using this along with a union bound over indicator random variables, we obtain
\begin{align*}
A_1
&\leq \sum_{j=1}^{n_0} \sum_{\substack{k=1 \\ k\neq j}}^{n_0} \sum_{u=1}^{2\T -1} \sum_{v=0}^{2\T -1}\ind{N_Y = \zeta+v} \cdot \ind{N_Y(X_{\Dsetone_j }\cup X_{\Dsetone_k }) = u+v} \cdot \ind{N_Y(X_{\Isetone_j \cap \Isetone_k})=\zeta-u}.
\end{align*}
In order to obtain the final bound, we need to separate the indicator random variable over $X_{\Dsetone_j} \cup X_{\Dsetone_k}$ into indicators over $X_{\Dsetone_j }$ and $X_{\Dsetone_k}$ for each value of $u, v \in [2\T-1]$.
Again applying a union bound over indicator random variables, we have
\begin{align} \nonumber
        &\sum_{j=1}^{n_0} \sum_{\substack{k=1 \\ k \neq j}}^{n_0} \ind{N_Y(X_{\Dsetone_j} \cup X_{\Dsetone_k}) = u+v} \cdot \ind{N_Y = \zeta +v} \\ \nonumber
        &\leq \sum_{j=1}^{n_0}  \sum_{\substack{k=1 \\ k \neq j}}^{n_0} \left(\ind{N_Y(X_{\Dsetone_j}) \geq (u+v)/2} + \ind{N_Y(X_{\Dsetone_k}) \geq (u+v)/2} \right) \cdot \ind{N_Y = \zeta + v} \\ \label{eq:intermediate_union_bound}
        &= 2 \sum_{j = 1}^{n_0} \sum_{\substack{k=1 \\ k \neq j}}^{n_0}\ind{N_Y(X_{\Dsetone_j}) \geq (u+v)/2} \cdot \ind{N_Y = \zeta + v}.
\end{align}
Substituting Eq.~\eqref{eq:intermediate_union_bound} in the bound for $A_1$, we obtain
\begin{align*}
A_1&\leq 2\sum_{j=1}^{n_0} \sum_{\substack{k=1 \\ k\neq j}}^{n_0} \sum_{u=1}^{2\T -1} \sum_{v=0}^{2\T -1}\ind{N_Y = \zeta+v} \cdot \ind{N_Y(X_{\Dsetone_j }) \geq (u+v)/2} \cdot \ind{N_Y(X_{\Isetone_j \cap \Isetone_k})=\zeta-u}. 
\end{align*} 
In order to proceed further, we need to handle the cases for $v=0$ and $v>0$ slightly differently. 
In particular, we have
\begin{align*}
A_1 &\leq \underbrace{2\sum_{j=1}^{n_0} \sum_{\substack{k=1 \\ k\neq j}}^{n_0} \sum_{u=1}^{2\T -1} \sum_{v=1}^{2\T -1}\ind{N_Y = \zeta+v} \cdot \ind{N_Y(X_{\Dsetone_j }) \geq (u+v)/2} \cdot \ind{N_Y(X_{\Isetone_j \cap \Isetone_k})=\zeta-u}}_{A_1^{(1)}} \\
&\qquad\qquad \qquad+  \underbrace{2\sum_{j=1}^{n_0} \sum_{\substack{k=1 \\ k\neq j}}^{n_0} \sum_{u=1}^{2\T -1} \ind{N_Y = \zeta} \cdot \ind{N_Y(X_{\Dsetone_j }) \geq u/2} \cdot \ind{N_Y(X_{\Isetone_j \cap \Isetone_k})=\zeta-u}}_{A_1^{(2)}}.
\end{align*}
We bound the terms $A_1^{(1)}$ and $A_1^{(2)}$ separately.
Beginning with the term $A_1^{(1)}$, we have
\begin{align*}
A_1^{(1)} &= 2\sum_{j=1}^{n_0} \sum_{\substack{k=1 \\ k\neq j}}^{n_0} \sum_{u=1}^{2\T -1} \sum_{v=1}^{2\T -1}\ind{N_Y = \zeta+v} \cdot \ind{N_Y(X_{\Dsetone_j }) \geq (u+v)/2} \cdot \ind{N_Y(X_{\Isetone_j \cap \Isetone_k})=\zeta-u} \\ 
&\overset{\1}{\leq} 2\sum_{j=1}^{n_0} \sum_{\substack{k=1 \\ k\neq j}}^{n_0} \sum_{u=1}^{2\T -1} \sum_{v=1}^{2\T -1}\ind{N_Y = \zeta+v} \cdot \ind{N_Y(X_{\Dsetone_j }) \geq v/2} \cdot \ind{N_Y(X_{\Isetone_j \cap \Isetone_k})=\zeta-u} \\
&\overset{\2}{=} 2\sum_{j=1}^{n_0} \sum_{\substack{k=1 \\ k\neq j}}^{n_0} \sum_{v=1}^{2\T -1}\ind{N_Y = \zeta+v} \cdot \ind{N_Y(X_{\Dsetone_j }) \geq v/2} \cdot \sum_{u=1}^{2\T -1} \ind{N_Y(X_{\Isetone_j \cap \Isetone_k})=\zeta-u} 
\end{align*}
Above, step $\1$ follows because we have $\ind{N_Y(X_{\Dsetone_j }) \geq (u+v)/2} \leq \ind{N_Y(X_{\Dsetone_j }) \geq v/2}$, and step $\2$ follows by again separating the summation over the cases $u = 0$ and $u > 0$. Using the fact that $\sum_{u=1}^{2\T-1}\ind{N_Y(X_{\Isetone_j \cap \Isetone_k})=\zeta-u} \leq 1$, we then obtain
\begin{align*}
A_1^{(1)} &\overset{}{\leq} 2\sum_{j=1}^{n_0} \sum_{\substack{k=1 \\ k\neq j}}^{n_0} \sum_{v=1}^{2\T -1}\ind{N_Y = \zeta+v} \cdot \ind{N_Y(X_{\Dsetone_j }) \geq v/2}\\
&\overset{\1}{\leq} 4n_0 \sum_{v=1}^{2\T-1} \ind{N_Y = \zeta+v} \cdot \frac{\zeta+v}{v}.
\end{align*}
Above, step $\1$ follows because
\begin{align}\label{eq:intermediate_union_sum}
    2 \sum_{j = 1}^{n_0} \sum_{\substack{k=1 \\ k \neq j}}^{n_0}\ind{N_Y(X_{\Dsetone_j}) \geq v/2} \cdot \ind{N_Y = \zeta + v} \leq 4 n_0 \cdot \frac{(\zeta + v) \cdot \ind{N_Y = \zeta + v}}{v}.
\end{align}
The inequality follows because under the condition $\ind{N_Y = \zeta + v}$ and for a fixed $j \in [n_0]$, at most $\frac{N_Y}{v/2} = \frac{\zeta+v}{v/2}$ blocks can have at least $v/2$ occurrences of $Y$ for any $Y \in \Xspace$. Thus, the inner summation over $k$ is bounded by $\frac{(\zeta + v) \ind{N_Y = \zeta + v}}{v/2}$. Completing the outer summation over $j$ yields the desired statement. 
    
Next, we will bound the term $A_1^{(2)}$ which turns out to follow a similar but slightly simpler argument. Using the fact that $\ind{N_Y(X_{\Isetone_j \cap \Isetone_k})=\zeta-u} \leq 1$, we obtain
\begin{align*}
    A_1^{(2)} &\leq  2\sum_{u=1}^{2\T -1}\sum_{j=1}^{n_0} \sum_{\substack{k=1 \\ k\neq j}}^{n_0}  \ind{N_Y = \zeta} \cdot \ind{N_Y(X_{\Dsetone_j }) \geq u/2} \\
    &\overset{\1}{\leq} 4n_0 \cdot \sum_{u=1}^{2\T-1} \ind{N_Y = \zeta} \cdot \frac{\zeta}{u} \leq 4n_0 \cdot \zeta \cdot \ind{N_Y=\zeta} \cdot \log(2\T).
\end{align*}
Above, step $\1$ follows from Eq.~\eqref{eq:intermediate_union_sum} and then using the fact that $\sum_{u=1}^{2\T-1} 1/u \leq \log(2\T)$.
Combining the bounds for $A_1^{(1)}$ and $A_1^{(2)}$ yields
\begin{align*}
    A_1 \leq 4n_0 \cdot \zeta \cdot \ind{N_Y=\zeta} \cdot \log(2\T) + 4n_0 \sum_{v=1}^{2\T-1} \ind{N_Y = \zeta+v} \cdot \frac{\zeta+v}{v}.
\end{align*}
Next, we will show the upper bound on $A_2$ through a similar argument.
Fixing two indices $j < k \in [n_0]$ and using the union bound over indicators, we have
\begin{align*}
&\ind{Y \in X_{\Dsetone_j}} \cdot \ind{N_Y(X_{\Isetone_j \cap \Isetone_k}) = \zeta } + \ind{Y \in X_{\Dsetone_k}} \cdot \ind{N_Y(X_{\Isetone_j \cap \Isetone_k}) = \zeta } \\
&\qquad \overset{\1}{\leq} 2 \cdot \ind{Y \in X_{\Dsetone_j} \cup X_{\Dsetone_k}} \cdot \ind{N_Y(X_{\Isetone_j \cap \Isetone_k}) = \zeta } \\
&\qquad \overset{\2}{=} 2 \sum_{u = 1}^{2(2\tau - 1)} \ind{N_Y(X_{\Dsetone_j} \cup X_{\Dsetone_k}) = u} \cdot \ind{N_Y(X_{\Isetone_j \cap \Isetone_k}) = \zeta } \\
&\qquad = 2 \sum_{u = 1}^{4\tau - 2} \ind{N_Y = \zeta + u} \cdot \ind{N_Y(X_{\Dsetone_j} \cup X_{\Dsetone_k}) = u},
\end{align*}
where step $\1$ follows from the union bound and step $\2$ follows by introducing the summation over $u$ which counts the number of occurrences of $Y$ in $X_{\Dsetone_j} \cup X_{\Dsetone_k}$.
Consequently, from the above statement, we have
\begin{align*}
A_2 &\leq 2 \sum_{j=1}^{n_0} \sum_{\substack{k=1 \\ k\neq j}}^{n_0} \sum_{u = 1}^{4\tau - 2} \ind{N_Y = \zeta + u} \cdot \ind{N_Y(X_{\Dsetone_j} \cup X_{\Dsetone_k}) = u} \\
&= 2\sum_{u = 1}^{4\tau - 2} \ind{N_Y = \zeta + u} \cdot \sum_{j=1}^{n_0} \sum_{\substack{k=1 \\ k\neq j}}^{n_0} \ind{N_Y(X_{\Dsetone_j} \cup X_{\Dsetone_k}) = u} \\
&\stackrel{\1}{\leq} 8n_0 \sum_{u = 1}^{4\tau - 2} \ind{N_Y = \zeta + u}  \cdot \frac{(\zeta + u) }{u}.
\end{align*}
Above, step $\1$ holds using the steps similar to the ones used to obtain Eq.~\eqref{eq:intermediate_union_bound} and Ineq~\eqref{eq:intermediate_union_sum}. 
This completes the proof of the upper bound on $A_2$. 

Combining the bounds for $A_1$ and $A_2$ completes the proof of the lemma. 
\qed

\subsubsection{Proof of Lemma \ref{lemma:diagonal_T1}}\label{sec:pf-lemma-diagonal-T1}
This lemma is proved in a manner similar to the proof of Lemma~\ref{lemma:exact_small_count_markovian}, but the argument is much simpler.
We introduce the variable $u$ which counts the number of occurrences of $Y$ in $X_{\Dsetone_j}$.
Then, we have
\begin{align*}
    \sum_{j=1}^{n_0}  \ind{N_{Y}(X_{\Isetone_j}) = \zeta} &= \sum_{j=1}^{n_0} \sum_{u=0}^{2\T-1} \ind{N_Y(X_{\Dsetone_j})=u} \cdot \ind{N_{Y}(X_{\Isetone_j}) = \zeta} \\
    &=\sum_{j=1}^{n_0} \sum_{u=0}^{2\T-1} \ind{N_Y=\zeta+u} \cdot \ind{N_{Y}(X_{\Dsetone_j}) = u}
\end{align*}
As in the proof of Lemma~\ref{lemma:exact_small_count_markovian}, we handle the cases for $u=0$ and $u>0$ slightly differently. 
We have
\begin{align*}
    \sum_{j=1}^{n_0}  \ind{N_{Y}(X_{\Isetone_j}) = \zeta} &= \sum_{u=1}^{2\T-1} \ind{N_Y=\zeta+u} \cdot \sum_{j=1}^{n_0} \ind{N_{Y}(X_{\Dsetone_j}) = u} +   \ind{N_Y=\zeta} \cdot \sum_{j=1}^{n_0}  \ind{N_{Y}(X_{\Dsetone_j}) = 0} \\
    &\overset{\1}{\leq} \sum_{u=1}^{2\T-1}\ind{N_Y=\zeta+u} \cdot \frac{\zeta+u}{u} + n_0 \cdot \ind{N_{Y} = \zeta},
\end{align*}
where step $\1$ follows because under the condition $N_Y=\zeta+u$, at most $\frac{N_Y}{u} = \frac{\zeta+u}{u}$ blocks can have exactly $u$ occurrences of $Y$. This completes the proof of the lemma.
\qed

\section{Discussion}
We have provided a flexible estimator and analysis of the vector of count probabilities $M^{\pi}(X^n)$ (and therefore, stationary distribution estimation) of any stationary exponentially $\alpha$-mixing stochastic process.
An explicit construction for the IID case~\citep{acharya2013optimal} reveals that our estimation error rate
is sharp in its dependence on $n$; we conjecture that the dependence on $\tmix$ is tight for our estimator, but showing this remains open. While we have obtained high-probability bounds on the error of the plug-in estimator, our bounds on the \Wingit are in expectation --- obtaining corresponding high-probability bounds would be an interesting follow-up problem. 
Finally, our approach heavily uses the fact that TV decomposes as an $\ell_1$ norm of frequency-by-frequency errors; obtaining corresponding analysis for the KL-divergence is an intriguing direction for future work. 

More broadly, we believe that the de-correlation devices introduced through our design and analysis of the \Wingit and plug-in estimators for sequences of random variables may find applications in related problems, such as stochastic optimization~\citep{mou2024optimal, li2023accelerated}, uncertainty quantification with dependent data~\citep{xu2023conformal, agrawal2024markov}, and property testing problems for stochastic processes~\citep{beran1992statistical, kalai2024calibrated}.

\subsection*{Acknowledgments}
This work was supported in part by the National Science Foundation through grants CCF-2107455, DMS-2210734, CCF-2239151 and IIS-2212182, and by research awards from Adobe, Amazon, and Mathworks.

\bibliographystyle{abbrvnat}
\bibliography{references-arXiv}

\appendix

\section{Auxiliary technical lemmas}

In this appendix, we collect some auxiliary technical lemmas required in our main proofs.

\subsection{Lemmas related to the TV distance}

The following lemma relates the TV distance of two natural estimators/distributions to the TV distance between their count probability masses.
\begin{lemma}
\label{lemma:TV_count_masses}
    Let us denote by $R_1$ and $R_2$ any two natural distributions over $\Xspace$. For a given sequence $X^n$, denote their vector of count probability masses as $M_1$ and $M_2$ respectively. Then, we have
    \begin{align*}
        \TV(R_1,R_2) = \TV(M_1,M_2).
    \end{align*}
\end{lemma}
\begin{proof}
    We have
    \begin{align*}
        \TV(R_1,R_2) &= \frac{1}{2}\sum_{x \in \Xspace} \left|R_1(x) -R_2(x)  \right| \\ 
        &\overset{\1}{=} \frac{1}{2} \sum_{x \in \Xspace} \sum_{u=0}^n\ind{N_x=u}\frac{\left| M_{u,1} - M_{u,2} \right|}{\varphi_u} \\ 
        &= \frac{1}{2}  \sum_{u=0}^n\sum_{x \in \Xspace}\ind{N_x=u}\frac{\left| M_{u,1} - M_{u,2} \right|}{\varphi_u} \\ 
        &= \frac{1}{2} \sum_{u=0}^n\left| M_{u,1} - M_{u,2} \right| = \TV(M_1,M_2),
    \end{align*}
where step $\1$ uses fact that both $R_1$ and $R_2$ are natural distributions.
This completes the proof of the lemma.
\end{proof}

The following lemma relates the total variation (TV) distance between the normalized and unnormalized versions of the vector of count probabilities for our estimator from \(M^{\pi}\). This result allows us to simplify our analysis by exclusively working with the unnormalized estimator in all our proofs.
\begin{lemma}
    \label{lemma:normalized}
    Let $\Mhat^{\unnorm}$ denote a vector of nonnegative entries indexed by $\zeta = \{0, 1, \ldots, n\}$, and define $\nu  \defn \sum_{\zeta = 0}^n \Mhat^{\unnorm}_{\zeta} > 0$. Define the normalized estimator $\Mhat = \Mhat^{\unnorm} / \nu$ as a vector on the simplex $\Delta(\{0, 1, \ldots, n \})$. Let $M^{\pi}$ be the vector of count probability masses of the stationary distribution $\pi$. Then the following holds:
    \begin{align*}
    \TV(\Mhat,M^{\pi}) \leq  \| \Mhat^{\unnorm} - 
 M^{\pi} \|_1.
\end{align*}
\end{lemma}

\begin{proof}
Using the definition $\nu \defn \sum_{\zeta = 0}^n \Mhat^{\unnorm}_{\zeta}$, we have
\begin{align*}
    \TV(\Mhat,M^{\pi}) &=  \frac{1}{2} \sum_{\zeta = 0}^{n} \left| \Mhat_{\zeta} - M^{\pi}_{\zeta} \right| \\ 
    &= \frac{1}{2} \sum_{\zeta = 0}^{n} \left| \frac{\Mhat_{\zeta}^{\unnorm}}{\nu}-\Mhat_{\zeta}^{\unnorm}+\Mhat_{\zeta}^{\unnorm} - M^{\pi}_{\zeta} \right| \\ 
    &\leq \frac{1}{2} \sum_{\zeta = 0}^{n} \Mhat_{\zeta}^{\unnorm}\left| \frac{1}{\nu} - 1 \right| + \frac{1}{2} \sum_{\zeta = 0}^{n} \left|\Mhat_{\zeta}^{\unnorm} - M^{\pi}_{\zeta} \right| \\ \label{eq:relate_norm}
    &= \frac{|\nu-1|}{2} + \frac{1}{2}\| \Mhat^{\unnorm} - M^{\pi} \|_1.
\end{align*}
It remains to bound the term $|\nu - 1|$, which we do below:
\begin{align*}
    |\nu-1| = \left|\sum_{\zeta=0}^n(\Mhat_{\zeta}^{\unnorm}-M^{\pi}_{\zeta})\right| \leq \sum_{\zeta=0}^n |\Mhat_{\zeta}^{\unnorm}-M^{\pi}_{\zeta}| = \| \Mhat^{\unnorm} - M^{\pi} \|_1.
\end{align*}
Putting together the pieces completes the proof.
\end{proof}

\subsection{Lemmas concerning indicator random variables}

In this section, we state and prove two simple lemmas concerning indicator random variables.
\begin{lemma}
    \label{lemma:Q_bound}
    For any symbol $x \in \Xspace$, let $Q_{j,k}$ be defined as in Eq.~\eqref{eq:Q}. Then, we have
    \begin{align*}
    Q_{j,k}  &\leq  \ind{x \in X_{\Isetone_k \setminus \Isetone_j}} \cdot \ind{N_x(X_{\Isetone_j \cap \Isetone_k}) = \zeta} \\ 
    &= \ind{x \in X_{\Dsetone_j}} \cdot \ind{N_x(X_{\Isetone_j \cap \Isetone_k}) = \zeta}.
\end{align*}
\end{lemma}
\begin{proof}
    The only case when this inequality can be violated is when $Q_{j,k}=1$ and $\ind{x \in X_{\Dsetone_j}} \cdot \ind{N_x(X_{\Isetone_j \cap \Isetone_k}) = \zeta} = 0$. In every other case the inequality is satisfied since $Q_{j,k} \in \{-1,0,1\}$ and $\ind{x \in X_{\Dsetone_j}} \cdot \ind{N_x(X_{\Isetone_j \cap \Isetone_k}) = \zeta} \in \{0,1\}$. Let us thus analyze the case when the equality is violated and show that this cannot happen.
    Note that $Q_{j,k}=1$ implies
    \begin{align*}
        N_x(X_{\Isetone_j \cap \Isetone_k}) = \zeta \text{ and } N_x(X_{\Isetone_k}) \neq \zeta.
    \end{align*}
    We can write the set $X_{\Isetone_k}$ as the union of the disjoint sets $X_{\Isetone_j \cap \Isetone_k}$ and $X_{\Dsetone_j}$. Thus, we have
    \begin{align*}
        N_x(X_{\Isetone_k}) = N_x(X_{\Isetone_j \cap \Isetone_k}) + N_x(X_{\Dsetone_j }) = \zeta+N_x(X_{\Dsetone_j }).
    \end{align*}
    Since $N_x(X_{\Isetone_j \cap \Isetone_k}) = \zeta$ and $N_x(X_{\Isetone_k}) \neq \zeta$, we can be sure that the symbol $x \in X_{\Dsetone_j }$. This implies that $\ind{x \in X_{\Dsetone_j}} \cdot \ind{N_Y(X_{\Isetone_j \cap \Isetone_k}) = \zeta} =1$. Hence, the inequality given in the statement of the lemma is always satisfied.
\end{proof}

\begin{lemma}
    \label{lemma:abs_Q_bound}
    For any fixed symbol $x \in \Xspace$, let $Q_{j,k}$ be defined as in Eq.~\eqref{eq:Q}. Then, we have
    \begin{align*}
    |Q_{j,k}| &\leq \ind{x \in X_{\Dsetone_j}} \cdot \left(\ind{N_x(X_{\Isetone_k}) = \zeta}+ \ind{N_x(X_{\Isetone_j \cap \Isetone_k}) = \zeta}\right).
\end{align*}
\end{lemma}
\begin{proof}
    From Lemma \ref{lemma:Q_bound}, we have
    \begin{align*}
        Q_{j,k} \leq \ind{x \in X_{\Dsetone_j}} \cdot \ind{N_x(X_{\Isetone_j \cap \Isetone_k}) = \zeta}.
    \end{align*}
    Thus, $Q_{j,k} \geq 0$ implies that $|Q_{j,k}| = Q_{j,k}$.
    In this case, Lemma~\ref{lemma:Q_bound} yields
    \begin{align*}
        |Q_{j,k}| = Q_{j,k} &\leq  \ind{x \in X_{\Dsetone_j}} \cdot \ind{N_x(X_{\Isetone_j \cap \Isetone_k}) = \zeta} \\ \nonumber
        &\leq \ind{x \in X_{\Dsetone_j}} \cdot \left(\ind{N_x(X_{\Isetone_k}) = \zeta}+ \ind{N_x(X_{\Isetone_j \cap \Isetone_k}) = \zeta}\right).
    \end{align*}
    The only case left to analyze is when $Q_{j,k} = -1$. This implies that
    \begin{align*}
        N_x(X_{\Isetone_j \cap \Isetone_k}) \neq \zeta \text{ and } N_x(X_{\Isetone_k}) = \zeta.
    \end{align*}
    This clearly implies that the symbol $x \in X_{\Dsetone_j}$. In this case, we have $|Q_{j,k}| = 1$ and $\ind{x \in X_{\Dsetone_j}} \cdot \left(\ind{N_x(X_{\Isetone_k}) = \zeta}+ \ind{N_x(X_{\Isetone_j \cap \Isetone_k}) = \zeta}\right)=1$. This completes the proof of the lemma.
\end{proof}

\subsection{Surrogate Process Lemmas}\label{sec:surrogate_process_lemmas}
We now state and prove a series of lemmas that relate the original stochastic process, or sub-process, to surrogate stochastic processes in which certain parts of the process are replaced by independent copies of random variables drawn from the stationary distribution.

\begin{lemma}\label{lemma:bridge}
For each $i \in [n]$, define the stochastic processes
\begin{align*}
Z_i &= (X_1, X_2, \ldots, X_{i - \tau}, X_i, X_{i + \tau}, X_{i + \tau + 1}, \ldots, X_n), \\
Z'_i &= (X_1, X_2, \ldots, X_{i - \tau}, X'_i, X_{i + \tau}, X_{i + \tau + 1}, \ldots, X_n),
\end{align*}
where $X'_i \sim \pi$ is drawn independently of everything else. Suppose the stationary stochastic process $X_t$ is exponentially $\alpha$-mixing with parameters $\mu$ and $\rho$ (see Eq.~\eqref{eq:mixing_condition}). Then, we have $\TV(Z_i, Z'_i) \leq 3\cm\cdot \rho^{\tau}$.   

Consequently, for any function $f$ with range $[0, 1]$, we have
\[
|\EE [f(Z_{i}) - f(Z'_{i})]| \leq 3\cm \cdot \rho^{\tau}.
\]
\end{lemma}
\begin{proof}

We prove the bound on total variation, noting that the consequence for bounded functions
follows as an immediate corollary.
We have
\begin{align*}
    &\TV(Z_i,Z_i') \\ 
    &= \sup_{S_1,\ldots,S_n \subseteq \Xspace} \big| \mathbb{P}(X_1\in S_1,\ldots,X_{i-\tau} \in S_{i-\tau},X_i \in S_i,X_{i+\tau} \in S_{i+\tau},\ldots,X_n \in S_n ) \\ 
    &\qquad \qquad \qquad - \mathbb{P}(X_1 \in S_1,\ldots,X_{i-\tau}\in S_{i-\tau},X_i'\in S_i,X_{i+\tau} \in S_{i+\tau},\ldots,X_n \in S_n ) \big| \\ 
    &\overset{\1}{=}\sup_{S_1,\ldots,S_n \subseteq \Xspace} \big| \mathbb{P}(X_1 \in S_1,\ldots,X_{i-\tau} \in S_{i-\tau},X_i \in S_i,X_{i+\tau} \in S_{i+\tau},\ldots,X_n \in S_n ) \\ 
    & \qquad \qquad \qquad -\mathbb{P}(X_1 \in S_1,\ldots,X_{i-\tau} \in S_{i-\tau})\mathbb{P}(X_i \in S_i,X_{i+\tau} \in S_{i+\tau},\ldots,X_n \in S_n ) \\ 
    & \qquad \qquad \qquad +\mathbb{P}(X_1 \in S_1,\ldots,X_{i-\tau} \in S_{i-\tau})\mathbb{P}(X_i \in S_i,X_{i+\tau} \in S_{i+\tau},\ldots,X_n \in S_n ) \\ 
    &\qquad \qquad \qquad - \mathbb{P}(X_i \in S_i)\mathbb{P}(X_1 \in S_1,\ldots,X_{i-\tau} \in S_{i-\tau},X_{i+\tau} \in S_{i+\tau},\ldots,X_n=S_n )\big| \\ 
    &\overset{\2}{\leq} \alpha(\tau)+ \sup_{S_1,\ldots,S_n \subseteq \Xspace} \big|\mathbb{P}(X_1 \in S_1,\ldots,X_{i-\tau} \in S_{i-\tau})\mathbb{P}(X_i \in S_i,X_{i+\tau} \in S_{i+\tau},\ldots,X_n \in S_n ) \\ 
    & \qquad \qquad \qquad \qquad \qquad -\mathbb{P}(X_1 \in S_1,\ldots,X_{i-\tau} \in S_{i-\tau})\mathbb{P}(X_i \in S_i)\mathbb{P}(X_{i+\tau} \in S_{i+\tau},\ldots,X_n \in S_n ) \\ 
    &\qquad \qquad \qquad \qquad \qquad +\mathbb{P}(X_1 \in S_1,\ldots,X_{i-\tau} \in S_{i-\tau})\mathbb{P}(X_i \in S_i)\mathbb{P}(X_{i+\tau} \in S_{i+\tau},\ldots,X_n \in S_n ) \\ 
    &\qquad \qquad \qquad \qquad \qquad - \mathbb{P}(X_i \in S_i)\mathbb{P}(X_1 \in S_1,\ldots,X_{i-\tau} \in S_{i-\tau},X_{i+\tau} \in S_{i+\tau},\ldots,X_n \in S_n )\big| \\ 
    &\overset{\3}{\leq} 2\alpha(\tau)+\sup_{S_1,\ldots,S_n \subseteq \Xspace} \mathbb{P}(X_i \in S_i)\big|\mathbb{P}(X_1 \in S_1,\ldots,X_{i-\tau} \in S_{i-\tau})\mathbb{P}(X_{i+\tau} \in S_{i+\tau},\ldots,X_n \in S_n ) \\ 
    &\qquad \qquad \qquad \qquad \qquad \qquad \qquad \qquad - \mathbb{P}(X_1 \in S_1,\ldots,X_{i-\tau} \in S_{i-\tau},X_{i+\tau} \in S_{i+\tau},\ldots,X_n \in S_n )\big| \\ 
    &\leq 2\alpha(\tau) + \alpha(2\tau) \leq 3\cm \cdot \rho^{\tau},
\end{align*}
where $\1$ follows from independence and adding and subtracting the same term, and $\2$ and $\3$ follow from triangle inequality and the definition of the $\alpha$-mixing coefficient.
This completes the proof of the lemma.
\end{proof}

\begin{lemma} \label{lemma:two-bridges}
For each $i_1 < i_2 \in [n]$ with $i_2 - i_1 \geq 2\tau$, define the stochastic sub-processes
\begin{align*}
Z_{i_1, i_2} &= (X_1, X_2, \ldots, X_{i_1 - \tau}, X_{i_1}, X_{i_1 + \tau}, \ldots, X_{i_2 - \tau}, X_{i_2}, X_{i_2 + \tau}, \ldots, X_n), \\
Z'_{i_1, i_2} &= (X_1, X_2, \ldots, X_{i_1 - \tau}, X'_{i_1}, X_{i_1 + \tau}, \ldots, X_{i_2 - \tau}, X'_{i_2}, X_{i_2 + \tau}, \ldots, X_n),
\end{align*}
where $X'_{i_1}, X'_{i_2} \sim \pi$ are drawn independently of each other and of everything else. Suppose the stationary stochastic process $X_t$ is exponentially $\alpha$-mixing with parameters $\mu$ and $\rho$ (see Eq.~\eqref{eq:mixing_condition}).
Then, we have \mbox{$\TV(Z_{i_1, i_2}, Z'_{i_1, i_2}) \leq 6\cm\cdot \rho^{\tau}$.}

Consequently, for any function $f$ with range $[0, 1]$, we have
\[
|\EE [f(Z_{i_1, i_2}) - f(Z'_{i_1, i_2})]| \leq 6\cm\cdot \rho^{\tau}.
\]
\end{lemma}
\begin{proof}
We prove the bound on total variation, noting that the consequence for bounded functions
follows as an immediate corollary.
    The steps to bound the total variation are similar to the proof of Lemma \ref{lemma:bridge}. Define the following intermediate process:
    \begin{align*}
        \widetilde{Z} = (X_1, X_2, \ldots, X_{i_1 - \tau}, X_{i_1}', X_{i_1 + \tau}, \ldots, X_{i_2 - \tau}, X_{i_2}, X_{i_2 + \tau}, \ldots, X_n).
    \end{align*}
    Applying the triangle inequality, we have
    \begin{align*}
        \TV(Z_{i_1, i_2},Z'_{i_1, i_2}) \leq \TV(Z_{i_1, i_2},\widetilde{Z}) + \TV(\widetilde{Z},Z'_{i_1, i_2}) \overset{\1}{\leq} 6\cm\cdot \rho^{\tau},
    \end{align*}
    where step $\1$ follows from Lemma \ref{lemma:bridge}.
    This completes the proof of the lemma.
\end{proof}

\begin{lemma} \label{lemma:restart}
Let $n_0 = n/\tau$.
For each $j \in [n_0]$, define the stochastic process $(X'_{k})_{k \in \Dset_{j\T}}$ as a $|\Dset_{j\T}|$-length sequence and initial state sampled from the distribution $\pi$ and independently of everything else. Thus, the tuple $(X'_1, \ldots, X'_{\tau})$ is independent of everything else, and according to the law of $(X_1, \ldots, X_{\tau})$. Similarly, choose the blocks $(X'_{\tau + 1}, \ldots, X'_{2 \tau}), \cdots, (X'_{n - \tau + 1}, \ldots, X'_{n})$ i.i.d. from the same distribution. We also assume the stationary stochastic process $X_t$ is exponentially $\alpha$-mixing with parameters $\mu$ and $\rho$ (see Eq.~\eqref{eq:mixing_condition}).
Now, define the stochastic processes
\begin{align*}
W_j &= \bigoplus_{\ell \in S_j} X_{\Dsetone_{\ell}} \quad \text{ and } \\
W'_j &= \bigoplus_{\ell \in S_j} X'_{\Dsetone_{\ell}}
\end{align*}
where we define $S_j = \{ \ell \in [n_0]: |\ell - j| \mod 2 \equiv 0\}$. 
Then, we have $\TV(W_j, W'_j) \leq n_0 \cdot \cm \cdot \rho^{\T}$.

Consequently, for any function $f$ having range $[0, 1]$, we have
\[
|\EE [f(W_j) - f(W'_j)]| \leq n_0 \cdot \cm \cdot \rho^{\T}.
\]
\end{lemma}

\begin{proof}
We prove the bound on total variation, noting that the consequence for bounded functions follows as an immediate corollary.
    Denote $s = |S_j|$ as shorthand and index the elements of $S_j$ as $\ell_1,\ldots,\ell_s$.
    Fix $j$ and, for each $k = 0,1, \ldots s$, define the auxiliary stochastic process 
    \[
    \widetilde{W}_j^{(k)} = \Big( \bigoplus_{k' = 1}^{k} X_{\Dsetone_{\ell_{k'}}} \Big) \bigoplus \Big( \bigoplus_{k' = k + 1}^{s} X'_{\Dsetone_{\ell_{k'}}} \Big). 
    \]
    noting that $\widetilde{W}_j^{(s)} = W_j$ and $\widetilde{W}_j^{(0)} = W'_j$.
    Then by the triangle inequality, we have
    \begin{align*}
        \TV(W_j,W'_j) \leq \sum_{k=0}^{s - 1} \TV(\widetilde{W}_j^{(k)},\widetilde{W}_j^{(k+1)}).
    \end{align*}
    We now claim that 
    \begin{align} \label{eq:clm-equiv-wtilde}
    \TV(\widetilde{W}_j^{(k)},\widetilde{W}_j^{(k+1)}) \leq \cm \cdot \rho^{\tau} \text{ for all } k = 0, 1, \ldots, m-1. 
    \end{align}
    Since $s = |S_j| \leq n_0$, this claim immediately yields the desired result.
    To prove claim~\eqref{eq:clm-equiv-wtilde}, we note that 
    \begin{align*}
        &
        \TV(\widetilde{W}_j^{(k)},\widetilde{W}_j^{(k+1)})  \\ 
        &= \sup_{S_1,\ldots,S_s \subseteq \Xspace} | \mathbb{P}(X_{\Dsetone_{\ell_1}}\in S_1,\ldots,X_{\Dsetone_{\ell_k}}\in S_k,X'_{\Dsetone_{\ell_{k+1}}}\in S_{k+1},\ldots,X'_{\Dsetone_{\ell_{s}}}\in S_s) \\ 
        &\qquad \qquad \qquad - \mathbb{P}(X_{\Dsetone_{\ell_1}}\in S_1,\ldots,X_{\Dsetone_{\ell_k}}\in S_k,X_{\Dsetone_{\ell_{k+1}}}\in S_{k+1},X'_{\Dsetone_{\ell_{k+2}}}\in S_{k+2},\ldots,X'_{\Dsetone_{\ell_{s}}}\in S_s)| \\ 
        &\overset{\1}{=} \sup_{S_1,\ldots,S_s \subseteq \Xspace} | \mathbb{P}(X_{\Dsetone_{\ell_1}}\in S_1,\ldots,X_{\Dsetone_{\ell_k}}\in S_{k})\mathbb{P}(X'_{\Dsetone_{\ell_{k+1}}}\in S_{k+1})\mathbb{P}(X'_{\Dsetone_{\ell_{k+2}}}\in S_{k+2},\ldots,X'_{\Dsetone_{\ell_{s}}}\in S_s) \\ 
        &\qquad \qquad \qquad - \mathbb{P}(X_{\Dsetone_{\ell_1}}\in S_1,\ldots,X_{\Dsetone_{\ell_k}}\in S_k,X_{\Dsetone_{\ell_{k+1}}}\in S_{k+1})\mathbb{P}(X'_{\Dsetone_{\ell_{k+2}}}\in S_{k+2},\ldots,X'_{\Dsetone_{\ell_{s}}}\in S_s)| \\ 
        &\leq \sup_{S_1,\ldots,S_s \subseteq \Xspace} | \mathbb{P}(X_{\Dsetone_{\ell_1}}\in S_1,\ldots,X_{\Dsetone_{\ell_k}}\in S_{k})\mathbb{P}(X_{\Dsetone_{\ell_{k+1}}}\in S_{k+1}) 
        %\\ 
        %&\qquad \qquad \qquad 
        - \mathbb{P}(X_{\Dsetone_{\ell_1}}\in S_1,\ldots,X_{\Dsetone_{\ell_k}}\in S_k,X_{\Dsetone_{\ell_{k+1}}}\in S_{k+1})| \\ 
        &\overset{\2}{\leq} \cm \cdot \rho^{\T},
    \end{align*}
    where step $\1$ follows from independence and step $\2$ follows from the $\alpha$-mixing assumption.
    This completes the proof of the lemma.
\end{proof}

\end{document}